\documentclass[11pt]{article}
\usepackage{dilab_arxiv}
\usepackage{amsmath,amsfonts,bm}

\def\eqref#1{equation~\ref{#1}}
\def\1{\bm{1}}

\DeclareMathAlphabet{\mathsfit}{\encodingdefault}{\sfdefault}{m}{sl}
\SetMathAlphabet{\mathsfit}{bold}{\encodingdefault}{\sfdefault}{bx}{n}

\def\gA{{\mathcal{A}}}

\def\gG{{\mathcal{G}}}

\def\gS{{\mathcal{S}}}

\newcommand{\KL}{D_{\mathrm{KL}}}
\newcommand{\TV}{D_{\rm TV}}

\usepackage{enumitem}
\usepackage{flafter}
\usepackage{wrapfig}
\usepackage{multicol}
\usepackage{subcaption}
\usepackage{placeins}
\usepackage{bookmark}
\usepackage{hyperref}
\usepackage{url}
\usepackage{amssymb}
\usepackage{mathtools}
\usepackage{amsthm}
\usepackage{tcolorbox}
\usepackage{array}
\usepackage{algorithm}
\usepackage{algorithmic}
\usepackage{multirow}
\usepackage{colortbl}
\usepackage{booktabs}
\usepackage[capitalize,noabbrev]{cleveref}
\usepackage[page,header]{appendix}
\apptocmd{\thebibliography}{\raggedright}{}{}

\newtheorem{theorem}{Theorem}[section]
\newtheorem{proposition}[theorem]{Proposition}

\setlist[itemize]{leftmargin=*,itemsep=3pt,topsep=5pt}

\hypersetup{
  pdftitle={FlexLoop: Depth-Elastic Looped Policies for Adaptive Test-Time Computation in Deep RL},
  pdfauthor={Xun Wang, Ruishuo Chen, Yu Chen, Zhuoran Li, and Longbo Huang},
  pdfsubject={FlexLoop: Depth-Elastic Looped Policies for Adaptive Test-Time Computation in Deep RL},
  pdfkeywords={reinforcement learning, looped architecture, test-time adaptation},
  bookmarksnumbered=true
}

\title{FlexLoop: Depth-Elastic Looped Policies for Adaptive Test-Time Computation in Deep RL}
\runningtitle{Depth-Elastic Looped Policies for Adaptive Test-Time Computation in Deep RL}
\date{arXiv preprint, September 2026}
\paperlogo{\includegraphics[height=1.5cm]{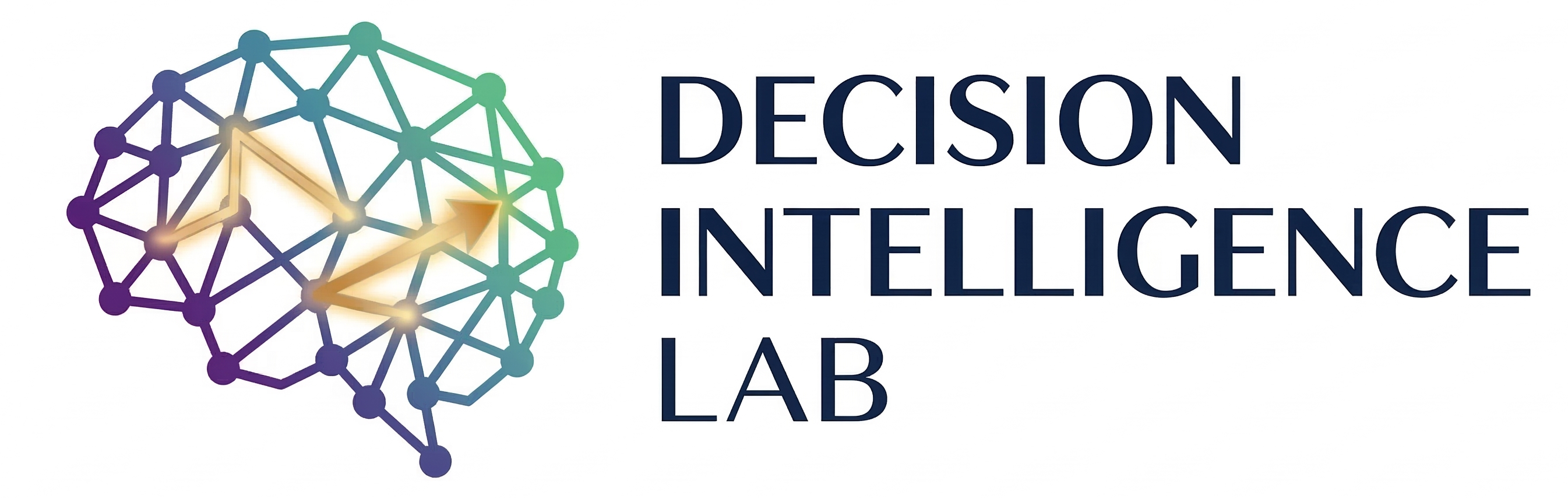}}
\author{
  Xun Wang$^{1}$, Ruishuo Chen$^{1}$, Yu Chen$^{1}$, Zhuoran Li$^{1}$, and
  Longbo Huang$^{1\,\text{\faEnvelope}}$
  \\[0.3em]\normalfont
  $^1$Institute for Interdisciplinary Information Sciences, Tsinghua University
  \\
  \text{\faEnvelope}\ Correspondence: \href{mailto:longbohuang@tsinghua.edu.cn}{\texttt{longbohuang@tsinghua.edu.cn}}
}

\begin{document}
\maketitle
\thispagestyle{fancy}

\begin{abstract}
Looped architectures scale computation by reusing the same parameters across recurrent steps, and recent work shows that they substantially improve deep reinforcement learning policies on long-horizon tasks. Since recurrent depth directly controls computation, one may expect looped policies to naturally support elastic inference across recurrent depths. Surprisingly, we find that pretrained looped policies exhibit severe recurrent-depth specialization: reliable decisions are concentrated near the full trained depth, tying deployment computation to this depth even when less computation may suffice. Achieving depth elasticity, i.e., reliable decisions across recurrent depths with adaptive computation at deployment, therefore remains a key challenge. To address this, we propose FlexLoop, a novel post-training framework that converts pretrained fixed-depth looped policies into depth-elastic policies. FlexLoop keeps training on the original RL objective to preserve full-depth capability while performing adjacent-depth policy distillation to progressively transfer decision quality from deeper to shallower recurrent steps. The resulting policy supports reliable inference across recurrent depths and enables state-wise adaptive inference through recurrent-depth consistency. 
Experiments on $30$ online and offline long-horizon goal-conditioned environments show that FlexLoop preserves full-depth performance while making shallower depths effective. Keeping competitive performance, FlexLoop reduces average recurrent depth by up to $\bf{43}$\textbf{\%} and achieves up to $\bf{1.34\times}$ wall-clock speedup in a stress test.
\end{abstract}

\section{Introduction}
\label{sec:intro}
\begin{wrapfigure}{r}{0.39\textwidth}
  \vspace{-12pt}
  \centering
  \includegraphics[width=0.4\textwidth]{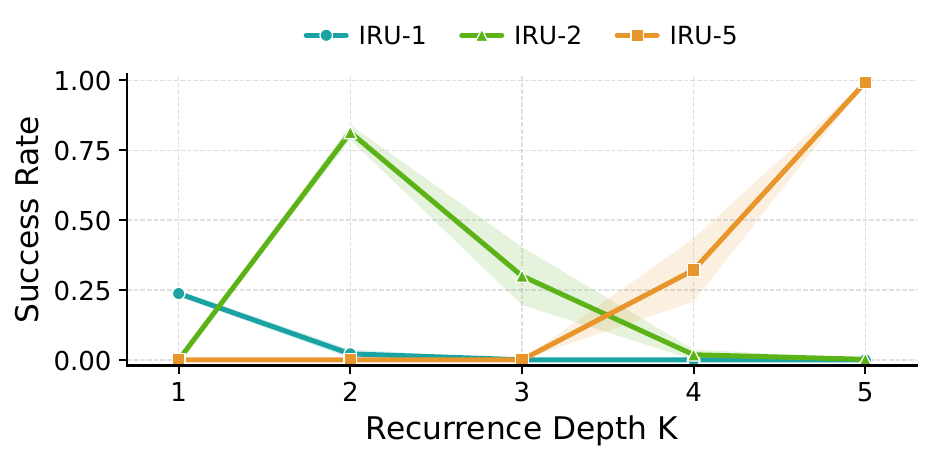}
  \caption{Fixed-depth training induces recurrent-depth specialization. Each IRU-$K$ is trained at the $K$-th recurrent depth. Reliable decisions are concentrated near the full trained depth. 
  % \ruishuo{The figure also shows failure in the other direction: IRU-1/IRU-2 collapse when run deeper than their training depth, but the caption only mentions shallower depths. Suggest: ``Performance peaks at the training depth and degrades sharply at both shallower and deeper depths.''}
  }
  \label{fig:moti}
  \vspace{-12pt}
\end{wrapfigure}
Scaling model capacity has been a central driver of progress in modern machine learning, with larger models enabling increasingly capable behavior across language modeling \citep{kaplan2020scalinglawsneurallanguage, wei2022emergent}, robotics \citep{zitkovich2023rt, kim2024openvla}, and deep reinforcement learning (DRL) \citep{wang2025thousand}. Beyond increasing parameter count, a complementary line of work has explored looped architectures \citep{dehghani2019universal, yang2024looped,goyal2026elt}, which repeatedly apply a shared computation block to increase effective computational depth while keeping the number of trainable parameters fixed. By scaling computation rather than model size, such architectures provide a parameter-efficient way to expand computational capacity.

Recent work on looped recurrent policies, such as IRU \citep{ghugare2026on}, has shown that recurrent computation provides an effective scaling dimension for DRL, with additional recurrent steps substantially improving planning and decision making on long-horizon tasks without increasing policy parameter count. This naturally suggests recurrent depth as a deployment-time computation budget, analogous to recent looped language models that support budget-aware or adaptive inference \citep{jeddi2026loopformer, popescu2026adaptive, schwethelm2026depthadaptive}.
Ideally, such policies shall exhibit \textit{depth elasticity}, namely, the ability of a single policy to maintain reliable decision quality across multiple recurrent depths and flexibly adjust its inference computation at
deployment.
% allowing a single policy to maintain decision quality across recurrent depths by allocating less computation to simple states and more computation to challenging decisions.
% \rsedit{Ideally, such policies should exhibit \textit{depth elasticity}: a single policy maintains decision quality across recurrent depths, so that it can spend less computation on simple states and more on challenging ones.}

Surprisingly, as shown in Figure~\ref{fig:moti}, standard RL training does not yield such elasticity. Instead, looped policies trained with a fixed recurrent depth exhibit severe recurrent-depth specialization, with reliable decision-making concentrated at the full training depth and performance deteriorating sharply at shallower depths.
This specialization ties deployment to the full recurrent depth. Although the full-depth policy itself is strong, always executing the maximum recurrent depth may use more computation than necessary for some decisions and limit adaptation to deployment resource constraints.
% Although the full-depth policy already provides strong decision-making capability, always executing the maximum recurrent depth may allocate more computation than necessary for some decisions and limit adaptation to deployment resource constraints.
% \rsedit{The full-depth policy is already strong, but always running it at full depth may spend more computation than needed on easy decisions and cannot adapt to resource constraints at deployment.}
This motivates a key question:

% Thus, although recurrent depth is adjustable by architecture, it remains tightly coupled to the training depth once trained, limiting deployment-time computation--performance trade-offs. This motivates a key question:
% \ruishuo{Readers of Figure~\ref{fig:moti} will also notice that extrapolation fails (IRU-1/IRU-2 collapse beyond their training depth), yet the story only discusses shallower depths and FlexLoop only targets $k<K$; I did not find this addressed later in the paper. Suggest acknowledging it and scoping explicitly, e.g., ``Running beyond the training depth does not help either, so the trained depth effectively fixes the usable computation. Since a strong full-depth policy is the natural starting point for deployment, we focus on making shallower depths reliable, so that computation can be reduced without retraining.'' Extrapolation ($k>K$) could then go to limitations/future work.} 

\textbf{\textit{Can we preserve the full-depth capability of a looped RL
policy while making its shallower recurrent depths reliable for decision making?}}

Addressing this question is nontrivial. Empirically, we find that looped architectures have sufficient capacity to support useful policies across multiple recurrent depths, yet naively training these depths jointly suppresses full-depth performance, while relying on distillation alone can degrade generalization to unseen tasks.
% \ruishuo{Empirically, we find that looped architectures have sufficient capacity to support useful policies across multiple recurrent depths, yet naively training these depths jointly suppresses full-depth policy performance, while directly distilling the full-depth policy to shallower depths fits the training tasks but fails to generalize to unseen ones, even at full depth.}

To resolve this dilemma, we propose FlexLoop, a novel post-training framework that converts a pretrained fixed-depth looped policy, which is reliable primarily at its training depth, into an elastic-depth policy that remains effective across recurrent depths.
% \ruishuo{To resolve this dilemma, we propose FlexLoop, a novel post-training framework that converts a pretrained IRU policy that is reliable only at its training depth into an elastic-depth policy that remains reliable across recurrent depths.} 
Starting from the pretrained policy, FlexLoop keeps training on the original RL objective to preserve full-depth performance and out-of-distribution generalization, while introducing adjacent-depth policy distillation to progressively propagate decision quality toward shallower recurrent steps. The resulting consistency across successive depths further provides a natural signal for state-wise adaptive inference, allowing computation to terminate once the policy has sufficiently stabilized.

Across 30 online and offline long-horizon goal-conditioned environments with
discrete and continuous action spaces, FlexLoop preserves full-depth performance while making shallower recurrent depths reliable, improving the deployment-time computation--performance trade-off.
With state-wise adaptive inference, FlexLoop automatically allocates recurrent computation according to state difficulty. It reduces the average recurrent depth by up to 43\% and achieves up to $1.34\times$ wall-clock speedup in a large-scale stress test, while maintaining competitive or better task performance.

In summary, our main contributions are:
\begin{itemize}[itemsep=2pt, topsep=2pt]
    \item We identify recurrent-depth specialization in existing looped RL policies: despite their capacity to operate across recurrent depths, fixed-depth training fails to induce depth elasticity, concentrating reliable decisions at the training depth and limiting adaptation to deployment computation constraints.
    % \ruishuo{We identify recurrent-depth specialization in existing looped RL policies: although recurrent depth is architecturally adjustable, policies trained at a fixed depth act reliably only at that depth and become unreliable at other depths, limiting deployment-time computation--performance trade-offs.}

    % \item We propose FlexLoop, a novel post-training framework that induces depth elasticity in pretrained looped RL policies by continuing the original RL objective to preserve full-depth capability and performing adjacent-depth policy distillation to transfer decision quality from deeper to shallower recurrent depths, enabling state-wise adaptive-depth inference at deployment without retraining.
    \item We propose FlexLoop, a novel post-training framework that induces depth elasticity in pretrained looped RL policies. It keeps training on the original RL objective to preserve full-depth capability and performs adjacent-depth policy distillation to transfer decision quality from deeper to shallower recurrent depths, enabling state-wise adaptive inference at deployment without retraining.

    \item Across 30 online and offline long-horizon goal-conditioned environments, FlexLoop preserves full-depth performance while enabling reliable decisions at shallower recurrent depths. State-wise adaptive inference reduces average recurrent computation by up to 43\% and achieves up to $1.34\times$ wall-clock speedup, while maintaining competitive or better task performance.
\end{itemize}
\section{Related Work}%add more recent works.
\label{sec:related}
\paragraph{Computation Scaling in DRL.}
Prior work has incorporated additional computation into DRL through planning and iterative value estimation
\citep{tamar2016value, silver2017predictron, farquhar2018treeqn}, while recent studies show that increasing network depth can substantially improve long-horizon decision making \citep{wang2025thousand}. More recently, \citet{ghugare2026on} showed that recurrent computation can scale RL capability without increasing policy parameter count. We build on this direction by investigating whether recurrent policies can remain effective across inference depths. We find that standard fixed-depth training instead leads to recurrent-depth specialization, highlighting the need for
elastic recurrent policies.

\paragraph{Looped Architectures and Elastic Depth.}
Looped architectures, which reuse the same parameters across multiple computation steps, have been studied for iterative reasoning and parameter-efficient computation \citep{dehghani2019universal, yang2024looped,zhu2026scalinglatentreasoninglooped,du2026loop}. Recent work extends looped models to variable inference budgets. For example, LoopFormer trains language models across different recurrent depths \citep{jeddi2026loopformer}, while ELT uses intra-loop self-distillation to transfer behavior from high-compute to intermediate depths in visual generation \citep{goyal2026elt}. Others study adaptive stopping and efficient depth-adaptive inference in looped language models \citep{popescu2026adaptive, schwethelm2026depthadaptive}. However, directly extending such depth elasticity to RL is challenging. Unlike settings with explicit targets, RL provides no ground-truth intermediate decisions, as policy quality is ultimately determined by long-horizon return rather than local output agreement. Achieving reliable depth elasticity for looped RL policies therefore remains underexplored.

\section{Preliminary}
\label{sec:preliminary}
\paragraph{Goal-Conditioned RL.}
We consider a goal-conditioned Markov decision process $\mathcal{M}$ $=$ $(\mathcal{S},\mathcal{A},\mathcal{G},P,r,$ $p_0,\gamma)$, where $\mathcal{S}$, $\mathcal{A}$, and $\mathcal{G}$ denote the state, action, and goal spaces, respectively; $P(\cdot\mid s,a)$ is the transition kernel; $r(s,a,g)$ is the goal-conditioned reward; $p_0(s,g)$ is a joint distribution over initial states and goals; and $\gamma\in[0,1)$ is the discount factor. At the beginning of each episode, $(s_0,g)\sim p_0$, and a goal-conditioned policy $\pi(a\mid s,g)$ induces trajectories according to $a_t\sim\pi(\cdot\mid s_t,g)$ and $s_{t+1}\sim P(\cdot\mid s_t,a_t)$. The objective is to maximize
$$J(\pi)=\mathbb{E}\left[\sum_{t=0}^{H-1}\gamma^t r(s_t,a_t,g)\right].$$
We consider both online and offline settings. In the latter, the agent learns from a fixed dataset \(\mathcal{D}\) without further environment interaction.
\paragraph{Looped Policies and IRU.}
Looped architecture \citep{jeddi2026loopformer,zhu2026scalinglatentreasoninglooped,du2026loop} increase computational depth by repeatedly applying a
shared computation block, without introducing depth-specific parameters.
Let $x$ denote the network input and $h_k$ the representation after $k$
recurrent steps. We write a looped computation abstractly as
\[
    h_{k+1}=F_\theta(h_k,x),
\]where $F_\theta$ is the computation block whose parameters
$\theta$ are shared across recurrent depth.

Recent work \citep{ghugare2026on} introduced the Interpolation Recurrent
Unit (IRU) as a minimal instantiation of this architecture and demonstrated
that scaling its recurrent depth can substantially improve long-horizon
decision making in DRL. We therefore focus primarily on IRU-based policies.
Given recurrent state $c_k$, IRU computes
\[
\begin{aligned}
    f_k &=
    \sigma\!\left(
        \mathrm{FORGET}_\theta([x,\tanh(c_k)])
    \right), \\
    I_k &=
    \tanh\!\left(
        \mathrm{INPUT}_\theta([x,\tanh(c_k)])
    \right), \\
    c_{k+1}
    &=
    f_k\odot c_k+(1-f_k)\odot I_k .
\end{aligned}
\]
At each depth $k$, an output head produces $z_k=G_\theta(c_k)$, which may
correspond to action values $Q_k(s,g,\cdot)$, a policy
$\pi_k(\cdot\mid s,g)$, or a value estimate $V_k(s,g)$ depending on the
RL algorithm. We denote the full recurrent depth by $K$. While standard IRU policies are
optimized through $z_K$, we aim to make $z_k$ reliable for $k<K$ while
preserving full-depth capability.

\section{Method: FlexLoop}
\label{sec:method}
\begin{figure}[!t]
    \centering
    \includegraphics[width=\linewidth]{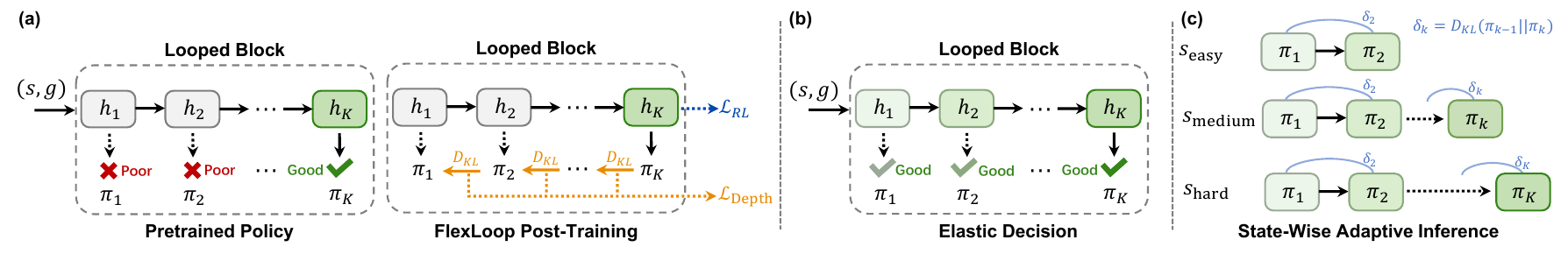}
    \caption{Framework of FlexLoop.\textbf{(a)} Starting from a depth-specialized looped policy, FlexLoop preserves full-depth capability with the original RL objective while propagating decision quality to shallower depths via adjacent-depth policy distillation. \textbf{(b)} The resulting policy supports reliable decisions across recurrent depths. \textbf{(c)} Adjacent-depth policy discrepancy serves as a stopping signal for state-wise adaptive inference, allocating more computation to harder states.}
    \label{fig:algo}
\end{figure}
As shown in Figure \ref{fig:algo}, we propose \textbf{FlexLoop}, a novel post-training framework for converting a depth-specialized looped RL policy into an elastic one. Starting from a pretrained loop policy, FlexLoop keeps training on the original RL training while introducing adjacent-depth policy self-distillation. The RL objective preserves the task capability acquired at full depth, while depth-wise distillation progressively makes shallower recurrent depths reliable for decision making, and induces a natural policy-consistency signal for state-wise adaptive-depth inference.

% We now introduce \textbf{FlexLoop}, a post-training framework for converting
% a depth-specialized looped RL policy into an elastic one. FlexLoop starts
% from a pretrained fixed-depth teacher and trains a new recurrent student
% whose decisions remain reliable across multiple recurrent depths.
% The central idea is to separate two roles during post-training:
% teacher-guided reinforcement learning preserves high-compute task capability,
% while depth-wise distillation transfers decision quality from deeper to
% shallower recurrent steps.

\subsection{Elastic Post-Training}
\label{sec:flexloop_general}

FlexLoop starts from a pretrained looped policy with full recurrent
depth $K$, inheriting all pretrained model parameters. During post-training,
it continues optimization with the original algorithm-specific RL objective
$\mathcal{L}_{\mathrm{RL}}$. We denote the policy induced after $k$
recurrent steps by $\pi_k(\cdot\mid s,g)$.

\paragraph{Full-Depth RL Objective.}
FlexLoop inherits all parameters from the pretrained depth-specialized model
and continues to optimize its full-depth value functions and policy under
the same algorithm-specific RL objective $\mathcal{L}_{\mathrm{RL}}$ as in
pretraining. This retains the capability acquired at depth $K$ and
preserves the out-of-distribution generalization of the pretrained policy
while shallower recurrent depths are elasticized.

\paragraph{Progressive Depth-Wise Policy Distillation.}
To make shallower recurrent depths directly executable, FlexLoop
progressively transfers decision quality from deeper to shallower recurrent
policies. Let $\pi_k(\cdot\mid s,g)$ denote the policy induced after $k$
recurrent steps. We optimize
\begin{equation}
    \mathcal{L}_{\mathrm{depth}}
    =
    \frac{1}{K-1}
    \sum_{k=1}^{K-1}
    D\!\left(
        \operatorname{sg}[\pi_{k+1}],
        \pi_k
    \right),
    \label{eq:depth_distill}
\end{equation}
where $D$ is an algorithm-specific policy discrepancy and $\operatorname{sg}[\cdot]$ denotes stop-gradient. This progressive distillation propagates full-depth decision quality toward shallower recurrent steps, while explicitly controlling policy variation across adjacent depths. The resulting discrepancy between successive policies provides a natural signal for state-wise adaptive inference, allowing computation to stop once the policy has sufficiently stabilized. %ablation for proof?

The full FlexLoop objective is
\begin{equation}
    \mathcal{L}_{\mathrm{FlexLoop}}
    =
    \mathcal{L}_{\mathrm{RL}}
    +
    \lambda_{\mathrm{depth}}
    \mathcal{L}_{\mathrm{depth}}
    % +
    % \lambda_{\mathrm{aux}}
    % \mathcal{L}_{\mathrm{aux}},
    \label{eq:flexloop}
\end{equation}
% where $\mathcal{L}_{\mathrm{aux}}$ denotes optional algorithm-specific
% auxiliary objectives. The first term preserves full-depth task capability,
% while the second explicitly controls how decisions evolve across recurrent
% depth.

\subsection{Instantiating FlexLoop across RL Algorithms}
\label{sec:flexloop_instantiations}

Following the setting of \citet{ghugare2026on}, we instantiate FlexLoop on three representative DRL algorithms for goal-conditioned long-horizon tasks, spanning online and offline settings as well as discrete and continuous action spaces.

\paragraph{Goal-Conditioned DQN (GC-DQN).}
We first instantiate FlexLoop with a goal-conditioned variant of DQN
\citep{mnih2015human}.
% Let $Q^T$ denote the frozen pretrained teacher.
% Following the teacher-guided objective, we replace the student's
% bootstrapping value with the teacher estimate,
% \begin{equation}
%     y^T
%     =
%     r+\gamma\max_{a'}Q^T(s',g,a'),
% \end{equation}
% and optimize the student's full-depth value by
% \begin{equation}
%     \mathcal{L}_{\mathrm{TG}}^{\mathrm{DQN}}
%     =
%     \mathbb{E}_{\mathcal D}
%     \left[
%         \left(
%             Q_K^\theta(s,g,a)
%             -
%             \operatorname{sg}[y^T]
%         \right)^2
%     \right]
% \end{equation}
Since a discrete-action policy is induced directly by its action values,
we convert each depth-$k$ value output into a soft policy
\begin{equation}\label{eq:softpolicy}
    \pi_k^Q(a\mid s,g)
    =
    \frac{
        \exp\!\left(Q_k(s,g,a)/\tau\right)
    }{
        \sum_{a'}\exp\!\left(Q_k(s,g,a')/\tau\right)
    },
\end{equation}
where $\tau$ is a temperature parameter.
We then apply progressive policy distillation between adjacent depths,
\begin{equation}
    \mathcal{L}_{\mathrm{depth}}^{\mathrm{DQN}}
    =
    \frac{1}{K-1}
    \sum_{k=1}^{K-1}
    \mathbb{E}_{(s,g)\sim\mathcal D}
    \left[
        D_{\mathrm{KL}}
        \left(
            \operatorname{sg}[\pi_{k+1}^Q]
            \,\Vert\,
            \pi_k^Q
        \right)
    \right],
\end{equation}
where $\mathcal{D}$ denotes the sampled transition batch from the replay buffer.

\paragraph{Goal-Conditioned PPO (GC-PPO).}
We next instantiate FlexLoop with a goal-conditioned variant of PPO
\citep{schulman2017proximal}.
% For teacher-guided critic learning, we directly regress the student's
% full-depth value estimate $V_K^\theta$ toward the frozen teacher critic $V^T$:
% \begin{equation}
%     \mathcal{L}_{\mathrm{TG}}^{V}
%     =
%     \mathbb{E}_{(s,g)\sim\mathcal{B}}
%     \left[
%         \left(
%             V_K^\theta(s,g)
%             -
%             \operatorname{sg}[V^T(s,g)]
%         \right)^2
%     \right],
% \end{equation}
% where $\mathcal{B}$ denotes the on-policy rollout batch.
% The actor remains optimized with the standard PPO objective, with GAE
% computed from the student's own value estimates.
We transfer decision quality across recurrent depth by distilling adjacent policies:
\begin{equation}
    \mathcal{L}_{\mathrm{depth}}^{\mathrm{PPO}}
    =
    \frac{1}{K-1}
    \sum_{k=1}^{K-1}
    \mathbb{E}_{(s,g)\sim\mathcal{B}}
    \left[
        \KL
        \left(
            \operatorname{sg}[\pi_{k+1}(\cdot\mid s,g)]
            \,\Vert\,
            \pi_k(\cdot\mid s,g)
        \right)
    \right],
    \label{eq:ppo_policy_chain}
\end{equation}
where $\mathcal{B}$ denotes the on-policy rollout batch.

% We additionally distill the critic across adjacent recurrent depths, which we empirically find further improves performance:
% \begin{equation}
%     \mathcal{L}_{\mathrm{aux}}^{\mathrm{PPO}}
%     =
%     \frac{1}{K-1}
%     \sum_{k=1}^{K-1}
%     \mathbb{E}_{(s,g)\sim\mathcal{B}}
%     \left[
%         \left(
%             V_k^\theta(s,g)
%             -
%             \operatorname{sg}[V_{k+1}^\theta(s,g)]
%         \right)^2
%     \right].
% \end{equation}
% Teacher guidance anchors the full-depth critic, while student-based PPO and adjacent-depth distillation preserve on-policy learning and promote depth-consistent actor--critic behavior.

\paragraph{Goal-Conditioned IQL (GC-IQL).}
We finally instantiate FlexLoop with a goal-conditioned variant of IQL
\citep{kostrikov2022offline}, where learning is restricted to a fixed
offline dataset $\mathcal{D}$.
% Following our teacher-guided formulation, we replace the value estimates
% used to construct the IQL targets with those of the frozen teacher.
% Specifically, the full-depth value function is trained by expectile
% regression toward the teacher critic,
% \begin{equation}
%     \mathcal{L}_{\mathrm{TG}}^{V}
%     =
%     \mathbb{E}_{\mathcal D}
%     \left[
%         L^{\tau}
%         \left(
%             Q^{T}(s,g,a)-V_K^\theta(s,g)
%         \right)
%     \right],
% \end{equation}
% where $Q^{T}=\min(Q_1^{T},Q_2^{T})$ and
% $L^{\tau}$ denotes the expectile loss with the parameter $\tau$.
% The full-depth critic is similarly trained with the teacher value bootstrap,
% \begin{equation}
%     \mathcal{L}_{\mathrm{TG}}^{Q}
%     =
%     \mathbb{E}_{\mathcal D}
%     \left[
%         \left(
%             Q_K^\theta(s,g,a)
%             -
%             \operatorname{sg}[y^{T}]
%         \right)^2
%     \right],
%     \qquad
%     y^{T}=r+\gamma V^{T}(s',g).
% \end{equation}
We distill the Gaussian policy across adjacent recurrent steps,
\begin{equation}
    \mathcal{L}_{\mathrm{depth}}^{\mathrm{IQL}}
    =
    \frac{1}{K-1}
    \sum_{k=1}^{K-1}
    \mathbb{E}_{(s,g)\sim\mathcal D}
    \left[
        \widetilde{D}_{\mathrm{KL}}
        \left(
            \operatorname{sg}[\pi_{k+1}(\cdot\mid s,g)]
            \,\Vert\,
            \pi_k(\cdot\mid s,g)
        \right)
    \right],
\end{equation}
where $\widetilde{D}_{\mathrm{KL}}$ denotes the normalized policy KL used
for continuous-action distillation, avoiding the scale-inflation degeneracy
of the raw KL objective.

% Unlike GC-PPO, critic-chain distillation provides no consistent benefit in GC-IQL in our experiments, so we do not include an additional critic-side auxiliary loss.
\subsection{State-Wise Adaptive Inference}
\label{sec:adaptive}

Once a policy produces meaningful decisions at multiple recurrent depths,
the remaining question is when additional computation is useful.
We first observe that good endpoint performance alone does not, in general, guarantee good shallower decisions.

\begin{proposition}[No Shallow-Depth Guarantee]
\label{prop:no_intermediate}
For any MDP $\mathcal M$ with bounded returns, recurrent depth $K\geq2$,
and $\epsilon>0$, there exists a shared-weight recurrent policy $\Pi$ such
that
\[
    J(\Pi_K) \geq \sup_{\pi}J(\pi)-\epsilon,
    \qquad
    J(\Pi_k) \leq \inf_{\pi}J(\pi)+\epsilon,
    \quad \forall 1\le k<K,
\]
where $\sup_\pi$ and $\inf_\pi$ are taken over all stationary goal-conditioned Markov policies.
\end{proposition}

Proposition~\ref{prop:no_intermediate} shows that strong full-depth
performance alone does not imply useful intermediate decisions. The
recurrent states may instead serve purely as internal computation.
The proof is given in Appendix~\ref{app:proof_no_intermediate}.

Through $\mathcal{L}_{\mathrm{depth}}$, FlexLoop explicitly controls how the policy evolves across recurrent depth.
Define the adjacent-depth discrepancy
\begin{equation}
    \epsilon_j
    =
    \sup_{s,g}
    D_{\mathrm{TV}}
    \left(
        \pi_j(\cdot\mid s,g),
        \pi_{j+1}(\cdot\mid s,g)
    \right),
\end{equation}
then we have the following proposition.
\begin{proposition}[Shallow-Depth Performance Bound]
\label{prop:depth_bound}
Assume $\forall(s,a,g)\in\gS\times\gA\times\gG, |r(s,a,g)|\leq R_{\max}$ and $\gamma<1$. Define $\delta_{\max}=\max_{1\leq j<K}
    \sup_{(s,g)}\KL\left(\pi_{j+1}(\cdot\mid s,g)\,\Vert\,\pi_j(\cdot\mid s,g)\right)$. Then, for any $1\leq k<K$,
\[
    \left|J(\pi_k)-J(\pi_K)\right|
    \leq
    \frac{2R_{\max}}{(1-\gamma)^2}
    (K-k)
    \sqrt{\frac{\delta_{\max}}{2}}.
\]
\end{proposition}
Proposition~\ref{prop:depth_bound} shows that controlling the largest adjacent-depth policy discrepancy directly bounds the performance gap between any shallow policy and the full-depth policy. Thus, when FlexLoop keeps neighboring recurrent policies close, shallow-depth performance cannot drift arbitrarily far from the full-depth endpoint. The proof is given in Appendix \ref{app:proof_shallow}.

\paragraph{State-Wise Adaptive Inference.}

\begin{figure}[!ht]
  \centering
  \begin{subfigure}[t]{.51\linewidth}
      \centering\includegraphics[width=\linewidth]{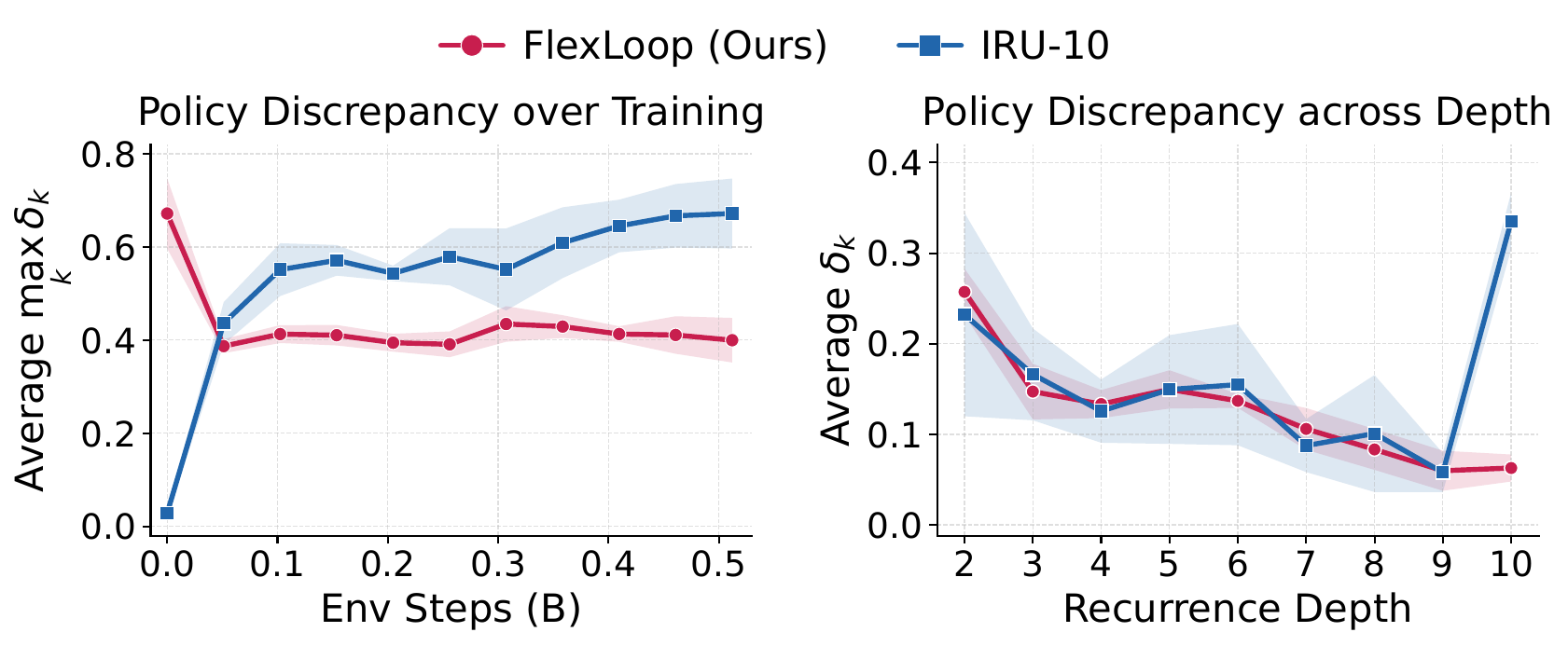}
    \caption{Reduction of adjacent-depth policy discrepancy.}
    \label{fig:delta}
  \end{subfigure}
  \begin{subfigure}[t]{.48\linewidth}
      \centering
      \includegraphics[width=\linewidth]{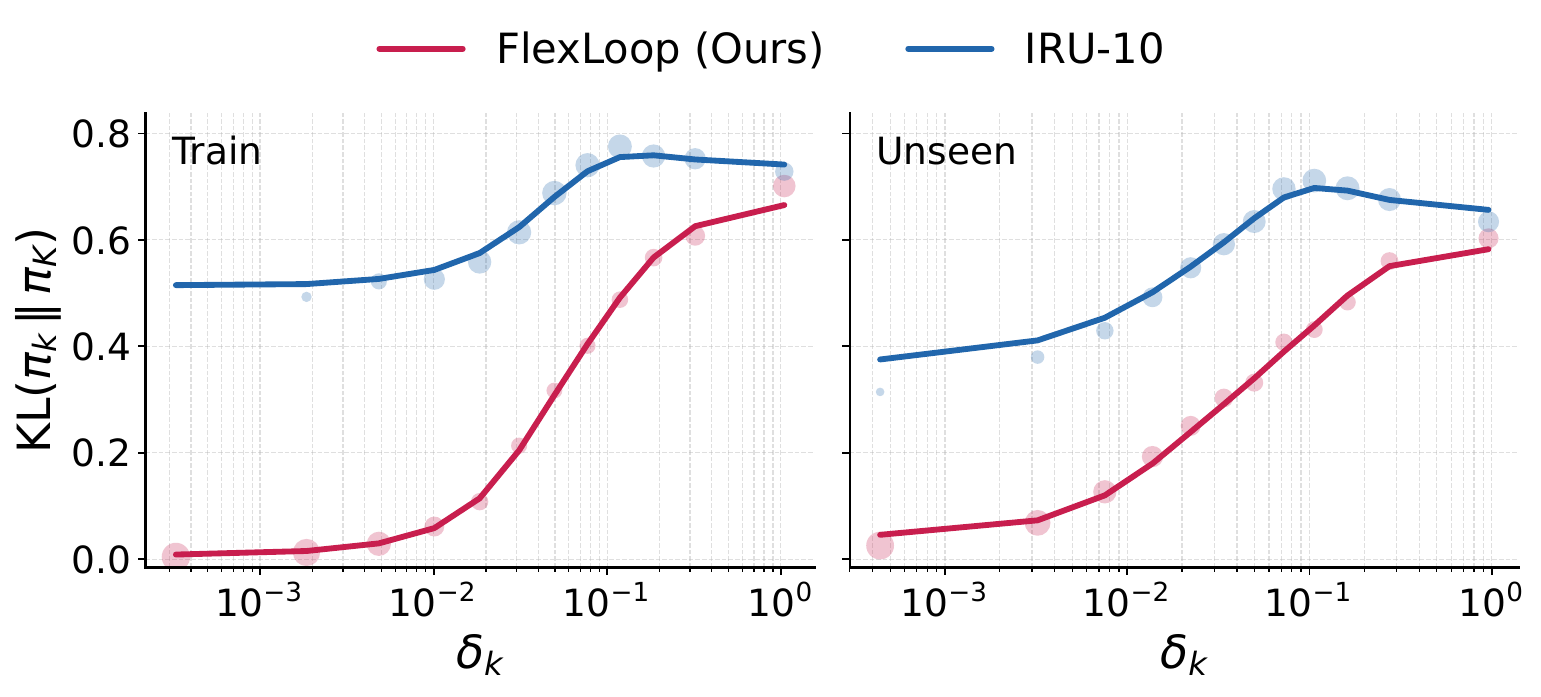}
    \caption{Local-to-endpoint policy consistency.}
    \label{fig:kl}
  \end{subfigure}
    \caption{(a) FlexLoop reduces adjacent-depth policy discrepancy during post-training and across depths. (b) Smaller \(\delta_k\) corresponds to smaller $\KL(\pi_k\Vert\pi_K)$ on both training and unseen tasks. Marker size denotes the sample fraction in each pooled quantile bin.} \label{fig:ada_evidence}
\end{figure}

Motivated by Proposition~\ref{prop:depth_bound}, we monitor the depth-wise change in the policy output,
\begin{equation}
\delta_k(s,g)
=
D\left(
\pi_{k-1}(\cdot\mid s,g),
\pi_{k}(\cdot\mid s,g)
\right),\quad k > 1
\label{eq}
\end{equation}
where $D$ is a policy-space divergence.

Figure~\ref{fig:delta} shows that FlexLoop reduces adjacent-depth policy divergence during training and across recurrent depths, while Figure~\ref{fig:kl} shows that smaller local discrepancies are consistently associated with smaller divergence to the full-depth policy on both training and unseen tasks. This suggests that $\delta_k$ is informative of endpoint proximity, motivating its use as a practical stopping signal for state-wise adaptive inference.
Accordingly, given a threshold $\varepsilon$, minimum depth $k_{\min}$, inference terminates at
\begin{equation}\label{eq:adaptive}
k^\star(s,g)
=
\min\left(
\left\{
k > k_{\min}:
\delta_k(s,g)<\varepsilon
\right\}
\cup \{K\}
\right).
\end{equation}
Varying $\varepsilon$ yields different computation--performance trade-offs from a single trained policy.

\section{Experiments}
\label{sec:experiments}
\begin{figure}[!t]
  \centering
  \begin{subfigure}[t]{\linewidth}
    \centering
    \includegraphics[width=0.32\linewidth]{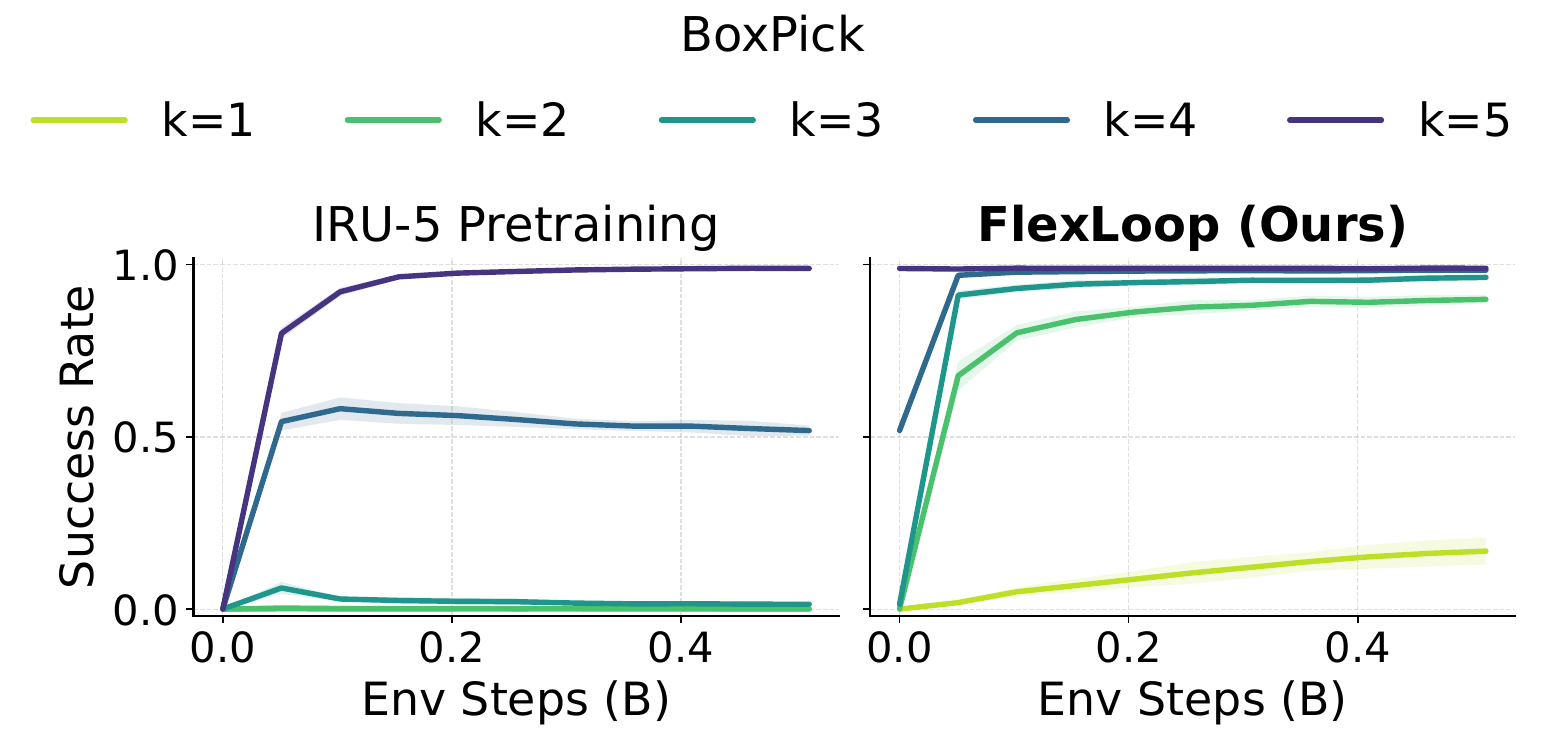}
  \hfill
    \includegraphics[width=0.32\linewidth]{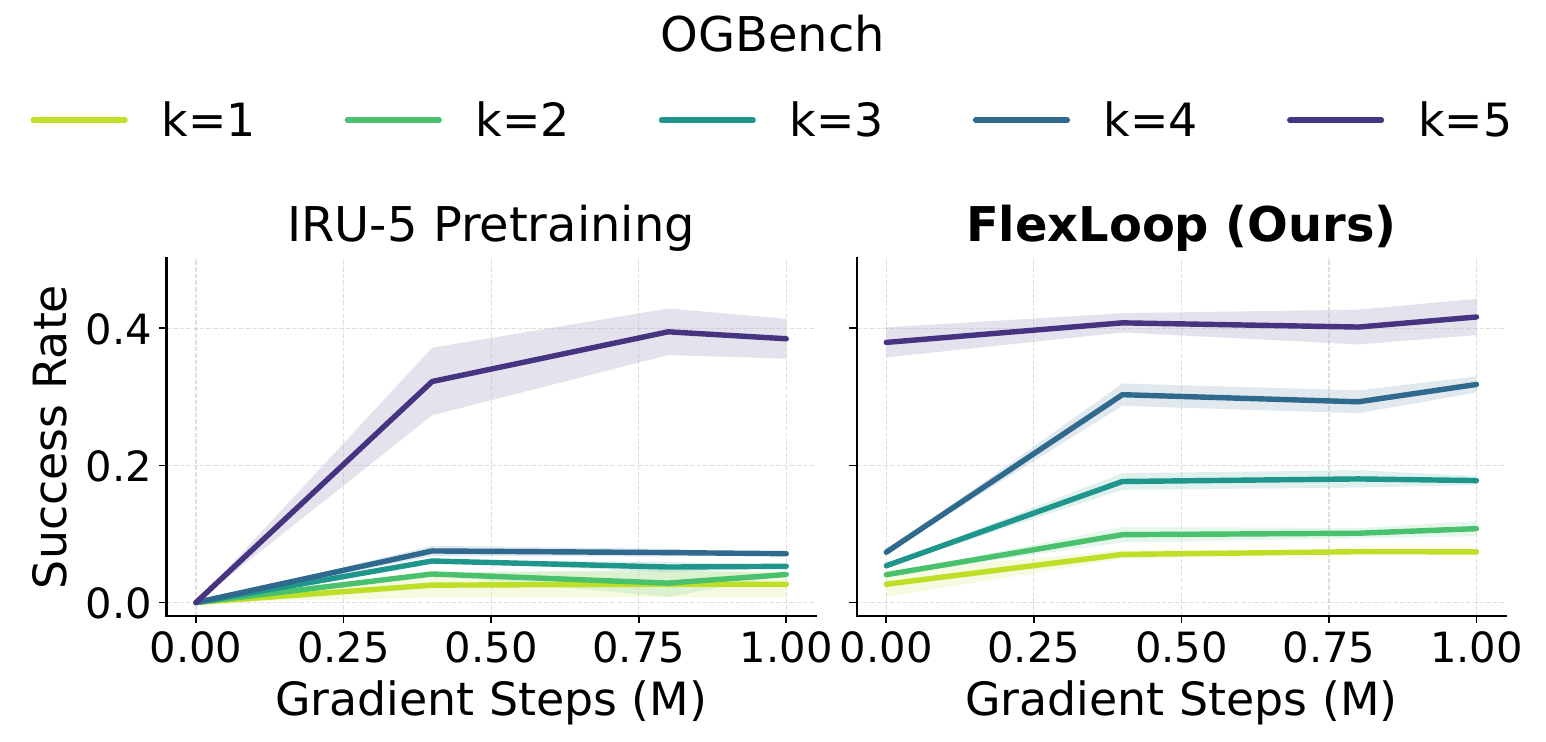}
  \hfill
    \includegraphics[width=0.33\linewidth]{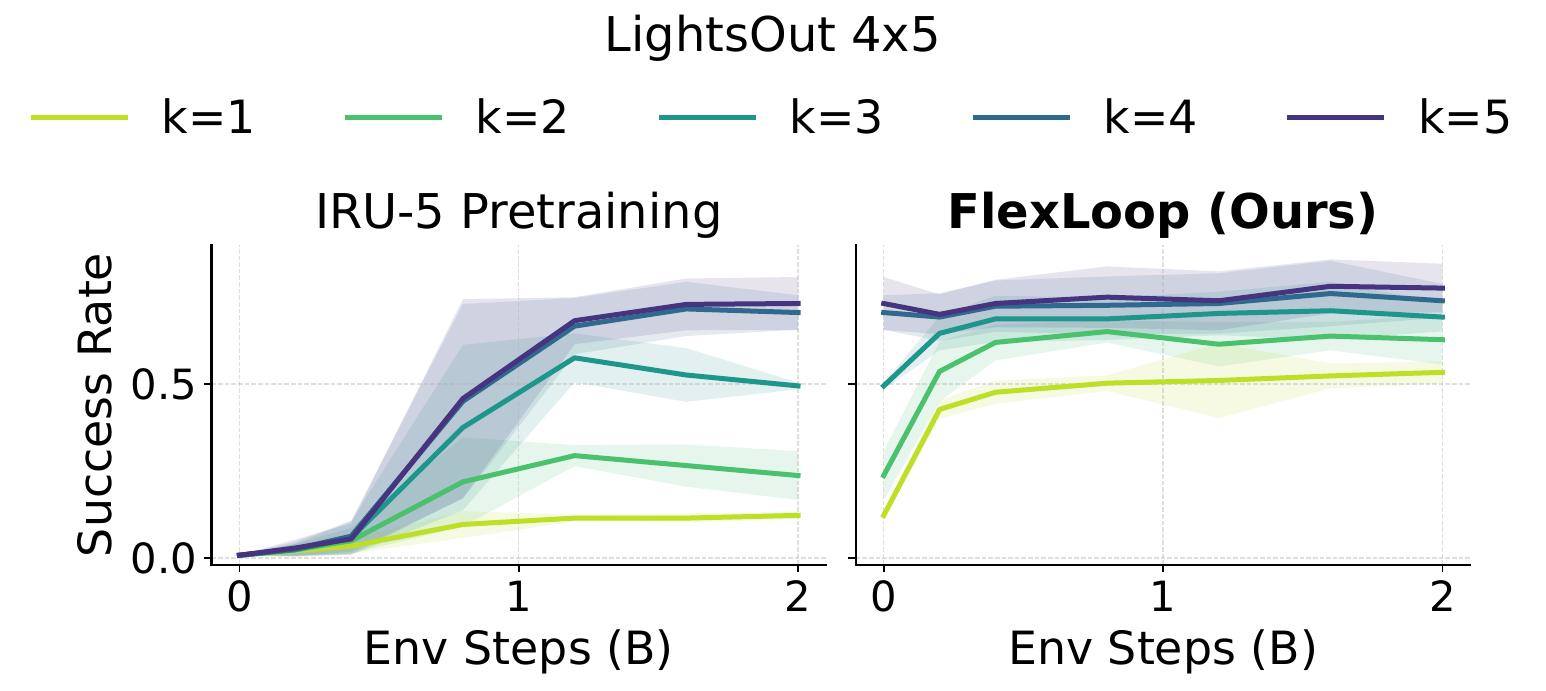}
  \caption{Depth-wise decision performance on training tasks.}\label{fig:res_main}
  \end{subfigure}
  \\
  \vspace{6pt}
  \begin{subfigure}[t]{\linewidth}
    \centering
    \includegraphics[width=0.32\linewidth]{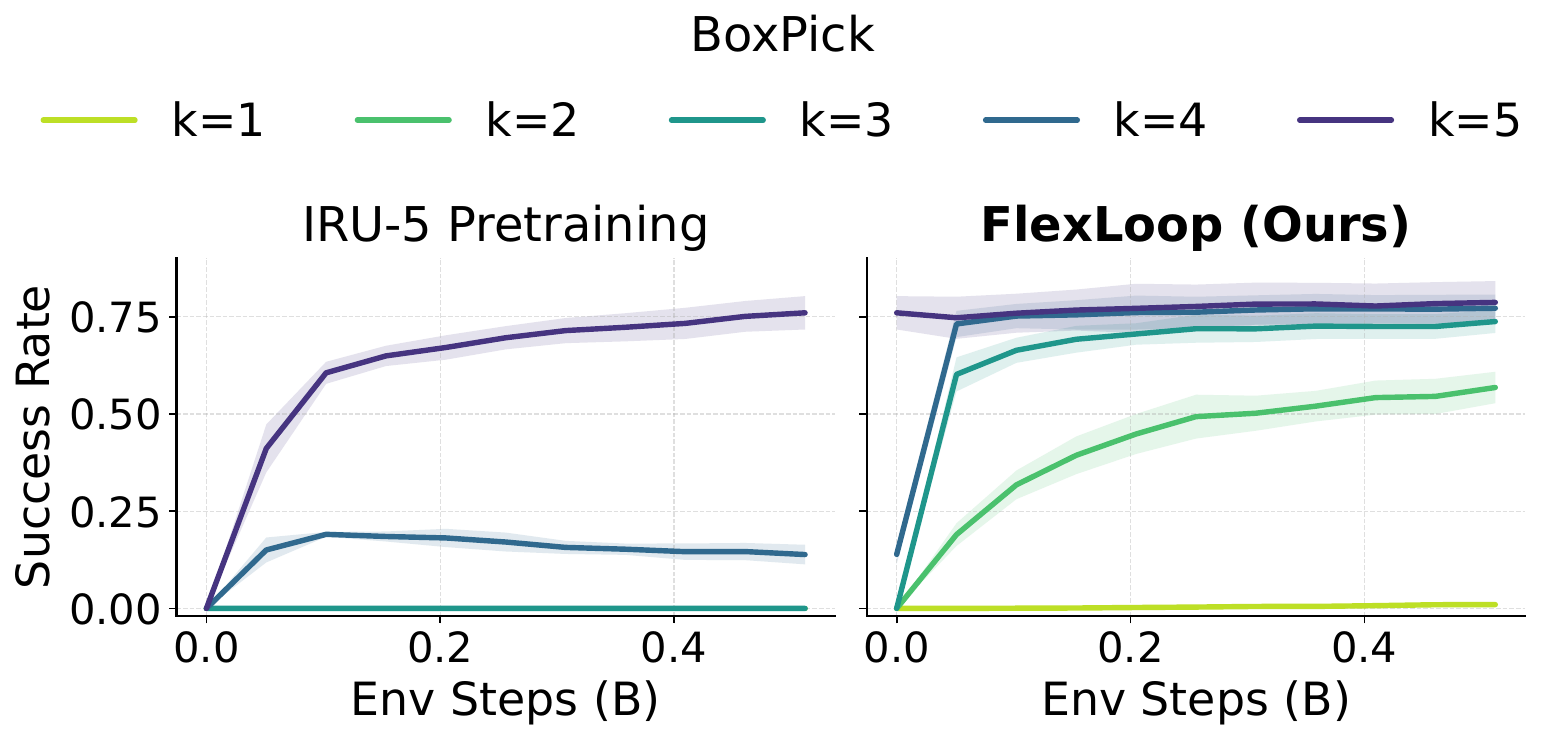}
  \hspace{2cm}
    \includegraphics[width=0.32\linewidth]{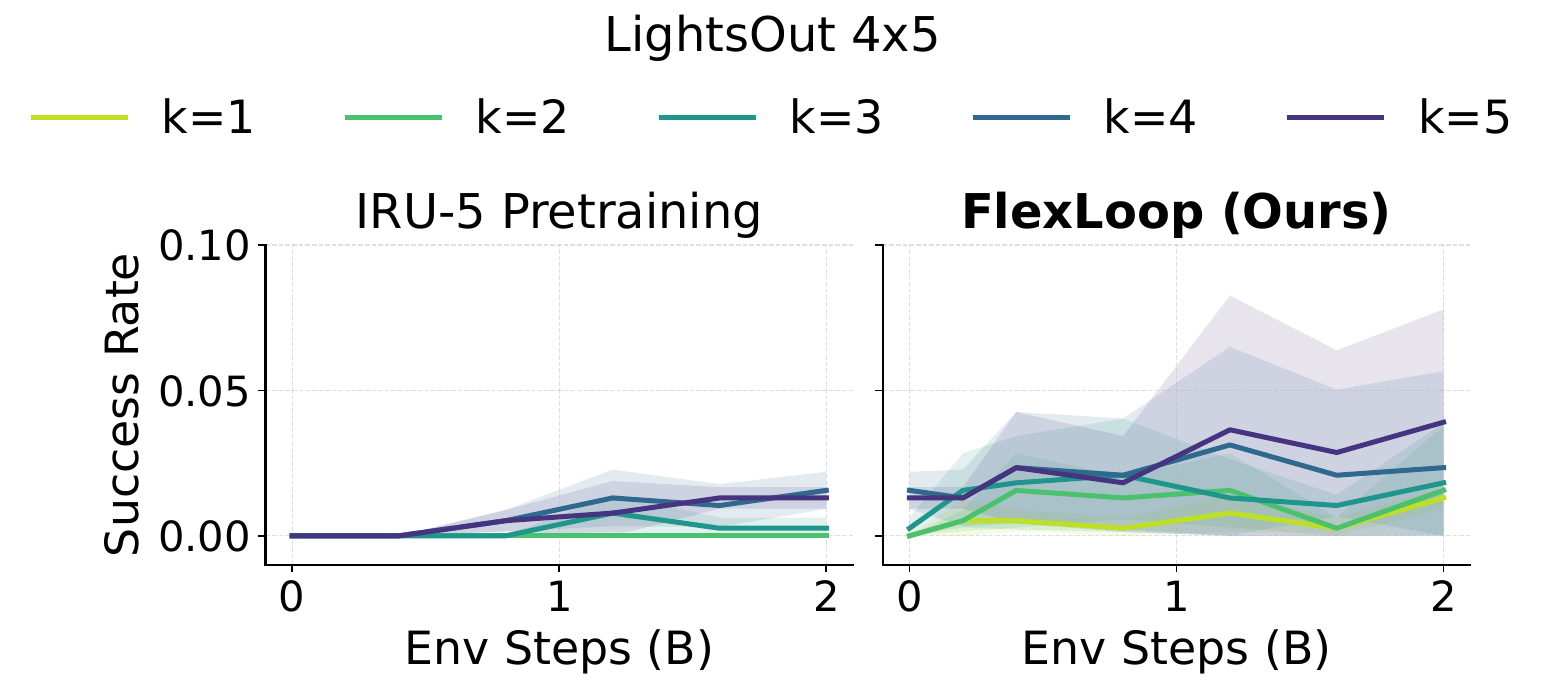}
\caption{Depth-wise decision performance on unseen tasks.}\label{fig:res_main_ood}
  \end{subfigure}
  
  \caption{
    Depth-wise decision performance of pretrained IRU-$5$ and FlexLoop on training and unseen tasks. Solid lines denote mean success rates and shading indicates one standard deviation.}
  \label{fig:res_kepoch}
\end{figure}

In this section, we evaluate FlexLoop across diverse online and offline goal-conditioned RL settings to answer three questions:
(i) can it preserve full-depth capability while enabling reliable shallow-depth decisions?
(ii) how does it trade off performance and inference compute? and
(iii) how do training design choices and recurrent architectures affect depth elasticity?
\subsection{Experimental Setup}
\paragraph{Benchmarks.}
Following \citet{ghugare2026on}, we evaluate FlexLoop on $30$ long-horizon goal-conditioned environments spanning online and offline RL, discrete and continuous action spaces, as well as in-distribution learning and out-of-distribution generalization. Specifically, we consider four environments from the BoxPick stitching benchmark \citep{bortkiewicz2025temporaldifferencelearninggold}, where an agent moves boxes on 2-D maps to designated goal locations; one LightsOut environment \citep{anderson1998turning}, formulated as a $4\times5$ binary switch puzzle; and 25 environments from OGBench \citep{ICLR2025_ecd92623}, covering three manipulation scenarios and two navigation scenarios. For BoxPick and LightsOut, we additionally evaluate generalization to unseen instances. Detailed environment descriptions are provided in Appendix~\ref{app:env_details}. We use goal-conditioned DQN for BoxPick, PPO for LightsOut, and goal-conditioned IQL for OGBench as the underlying RL algorithms.
\paragraph{Baselines and Implementation.}
We use IRU as the common looped-policy backbone and denote a pretrained depth-$K$ IRU policy as IRU-$K$. Following the results reported by \citet{ghugare2026on}, we choose $K=5$ as the default setting, as it provides strong performance while maintaining a practical computational cost. FlexLoop is initialized from the corresponding IRU-$5$ checkpoint and post-trained for a fixed number of optimization steps to acquire depth elasticity. Our implementation builds on the codebase of \citet{ghugare2026on}, retaining the same environment configurations, RL algorithms, and evaluation protocols unless otherwise stated. We report results over three random seeds. Additional implementation details and hyperparameters are provided in Appendix~\ref{app:implementation}.

\subsection{Analysis of Learned Depth Elasticity}

We evaluate FlexLoop on $30$ long-horizon goal-conditioned environments using pretrained IRU-$5$ policies. Figure~\ref{fig:res_kepoch} tracks the evolution of the recurrent-depth performance during IRU pretraining and FlexLoop post-training, showing the emergence of reliable shallow-depth policies and their generalization to unseen tasks in the BoxPick and LightsOut environments. Full environment-wise curves are provided in Appendix~\ref{app:res_kepoch}. We further compare the depth-wise performance curves of the final pretrained and post-trained policies in Appendix~\ref{app:elasticity}, providing a direct characterization of their depth elasticity. These results reveal several key observations:

% We evaluate FlexLoop on 30 long-horizon goal-conditioned environments using pretrained IRU-$5$ policies. Figure~\ref{fig:res_kepoch} tracks the evolution of recurrent-depth performance during IRU pretraining and FlexLoop post-training, showing the emergence of reliable shallow-depth policies and their generalization to unseen tasks in the BoxPick and LightsOut environments. Full environment-wise curves are provided in Appendix \ref{app:res_kepoch}

\textbf{FlexLoop improves shallow-depth decisions while preserving full-depth capability.}
Across almost all tasks in Figure~\ref{fig:res_main}, pretrained IRU policies concentrate reliable decision performance at the full depth, whereas shallower recurrent depths remain ineffective. After post-training, FlexLoop maintains or further improves the full-depth performance while rapidly increasing the performance of shallow recurrent depths. This shows that FlexLoop successfully converts a depth-specialized policy into an elastic policy without sacrificing the capability acquired by the pretrained model.

\textbf{FlexLoop generalizes depth elasticity beyond the training distribution.}
We evaluate the post-trained policies on unseen tasks to examine whether the learned depth elasticity generalizes beyond the training environments. As shown in Figure~\ref{fig:res_main_ood}, FlexLoop improves shallow-depth performance on unseen tasks while preserving the full-depth capability of the pretrained policy. This suggests that the improved shallow policies retain the general decision capability of the pretrained policy rather than simply overfitting to the training environments.

\textbf{Shallower recurrent depths are more challenging to elasticize.}
In Figure~\ref{fig:res_kepoch}, the improvement speed during post-training decreases monotonically with shallower recurrent depths. A possible explanation is that a reliable elastic policy needs to simultaneously support effective standalone decision making at shallow depths, maintain consistency across recurrent depths, and preserve robustness under distribution shifts. With fewer recurrent steps available, shallow policies have less computation capacity to refine intermediate representations and satisfy these requirements, making the shallowest depths harder to elasticize.

\textbf{The underlying RL formulation influences how FlexLoop acquires depth elasticity.}
FlexLoop achieves the largest improvements on the GC-DQN-based BoxPick benchmark, while gains are smaller on PPO and offline IQL tasks. We hypothesize that this difference arises from the coupling between value learning and policy execution. In value-based methods, improved value estimation directly determines action selection across recurrent depths. In contrast, methods with separately parameterized policies, such as PPO and IQL, rely on indirect policy updates from value learning, making depth-wise adaptation more challenging.

% full-depth preservation
% shallow-depth performance
% online OOD
% aggregate + main table

\subsection{State-wise Adaptive Inference}
FlexLoop's learned depth elasticity enables adaptive recurrent computation during deployment. We first evaluate whether adaptive inference can reduce computation while preserving performance, and then analyze the computation-performance trade-offs induced by different stopping thresholds.

\paragraph{Adaptive Inference Performance.}
\begin{wraptable}{r}{0.52\textwidth}
% \vspace{-6pt}
\centering
\caption{State-wise adaptive inference results across environments. $\pm$ indicates one standard deviation.}
\label{tab:adaptive_results}
\small
\resizebox{\linewidth}{!}{
\small
\begin{tabular}{lcccc}
\toprule
\multirow[c]{2}{*}{\textbf{Environment}} & \multicolumn{2}{c}{\textbf{FlexLoop}} & \textbf{IRU-5}  \\
\cmidrule(lr){2-3}\cmidrule(lr){4-4}
 & Avg. Depth & Succ. & Succ. \\
\midrule
OGBench & 4.19 $\pm$ 0.08 & \textbf{0.410} $\pm$ 0.020 & 0.385 $\pm$ 0.029 \\
BoxPick (Train) & 4.50 $\pm$ 0.02 & \textbf{0.989} $\pm$ 0.002 & \textbf{0.989} $\pm$ 0.003 \\
BoxPick (Unseen) & 4.46 $\pm$ 0.02 & \textbf{0.778} $\pm$ 0.055 & 0.760 $\pm$ 0.043 \\
LightsOut-4x5 (Train) & 2.84 $\pm$ 0.08 & \textbf{0.747} $\pm$ 0.081 & 0.732 $\pm$ 0.076 \\
LightsOut-4x5 (Unseen) & 2.93 $\pm$ 0.21 & \textbf{0.053} $\pm$ 0.031 & 0.013 $\pm$ 0.004 \\
\bottomrule
\end{tabular}
}
\vspace{-6pt}
\end{wraptable}

By tuning the stopping threshold $\varepsilon$ in \cref{eq:adaptive}, we select adaptive inference policies for each environment. Table~\ref{tab:adaptive_results} summarizes the resulting performance and average recurrent depth. FlexLoop maintains the performance of the full-depth policy while reducing the average recurrent depth by up to 43\%, demonstrating that learned depth elasticity enables effective inference computation reduction without sacrificing performance. Full environment-wise results are provided in Appendix \ref{app:adaptive}.
% \begin{table}[t]
% \centering
% \caption{State-wise adaptive inference results across environments. $\pm$ indicates one standard deviation.}
% \label{tab:adaptive_results}
% \small
% \begin{tabular}{lcccc}
% \toprule
% \multirow[c]{2}{*}{\textbf{Environment}} & \multicolumn{2}{c}{\textbf{FlexLoop}} & \textbf{IRU-5}  \\
% \cmidrule(lr){2-3}\cmidrule(lr){4-4}
%  & Avg. Depth & Succ. & Succ. \\
% \midrule
% OGBench & 4.19 $\pm$ 0.08 & \textbf{0.410} $\pm$ 0.020 & 0.385 $\pm$ 0.029 \\
% BoxPick (Train) & 4.50 $\pm$ 0.02 & \textbf{0.989} $\pm$ 0.002 & \textbf{0.989} $\pm$ 0.003 \\
% BoxPick (Unseen) & 4.46 $\pm$ 0.02 & \textbf{0.778} $\pm$ 0.055 & 0.760 $\pm$ 0.043 \\
% LightsOut-4x5 (Train) & 2.84 $\pm$ 0.08 & \textbf{0.747} $\pm$ 0.081 & 0.732 $\pm$ 0.076 \\
% LightsOut-4x5 (Unseen) & 2.93 $\pm$ 0.21 & \textbf{0.053} $\pm$ 0.031 & 0.013 $\pm$ 0.004 \\
% \bottomrule
% \end{tabular}
% \end{table}

\paragraph{Computation-Performance Trade-off.}
We further analyze how the stopping threshold controls the trade-off between computation and performance. By varying $\varepsilon$, we obtain the Pareto curves ( See Figure~\ref{fig:res_pareto} in Appendix \ref{app:pareto}). In general, smaller thresholds allocate deeper recurrent computation and tend to achieve higher performance, while larger thresholds reduce the average inference depth with controllable performance trade-offs. The resulting Pareto curves show that FlexLoop often matches 
\begin{wraptable}{r}{0.52\textwidth}
\vspace{-6pt}
\centering
\caption{Wall-clock comparison on BoxPick-Exact-4. $\pm$ indicates one standard deviation. Bold numbers indicate the shortest wall-clock latency.}\label{tab:wc_results}
\small
\resizebox{\linewidth}{!}{
\begin{tabular}{llccc}
\toprule
 & \textbf{Model} & \textbf{Success} & \textbf{Avg. Depth} & \textbf{Wall-clock (ms)} \\
\midrule
\multirow{4}{*}{Train} & IRU-5 & $0.995 \pm 0.001$ & $5$ & $170.1 \pm 17.2$ \\
 & FlexLoop & $0.993 \pm 0.001$ & $3.85 \pm 0.04$ & $\textbf{133.0} \pm 7.7$ \\
 \cmidrule{2-5}
 & IRU-10 & $0.997 \pm 0.000$ & $10$ & $333.2 \pm 36.5$ \\
 & FlexLoop & $0.996 \pm 0.001$ & $9.35 \pm 0.04$ & $\textbf{319.9} \pm 22.2$ \\
\midrule
\multirow{4}{*}{Unseen} & IRU-5 & $0.511 \pm 0.085$ & $5$ & $170.1 \pm 17.2$ \\
 & FlexLoop & $0.543 \pm 0.093$ & $3.80 \pm 0.05$ & $\textbf{127.4} \pm 6.3$ \\
 \cmidrule{2-5}
 & IRU-10 & $0.664 \pm 0.080$ & $10$ & $333.2 \pm 36.5$ \\
 & FlexLoop & $0.675 \pm 0.066$ & $9.37 \pm 0.04$ & $\textbf{321.1} \pm 22.7$ \\
\bottomrule
\end{tabular}
}
\vspace{-24pt}
\end{wraptable}
the pretrained IRU policy performance with reduced average recurrent depth. In addition, this favorable trade-off is preserved on unseen tasks, demonstrating that FlexLoop's learned depth elasticity enables generalizable adaptive computation beyond the training distribution.

\paragraph{Wall-Clock Time Saving.} 

To evaluate the practical inference speedup enabled by adaptive recurrent depth, we additionally train an IRU-$10$ policy and apply FlexLoop post-training to obtain the corresponding elastic variant. We evaluate both IRU-$5$ and IRU-$10$ fixed-depth policies and their FlexLoop variants on BoxPick-Exact-4 with a batch of $32,768$ parallel environments. As shown in Table~\ref{tab:wc_results}, FlexLoop maintains competitive performance relative to the pretrained policies while using fewer recurrent steps, achieving up to $1.34\times$ wall-clock speedup compared with fixed-depth inference.
\subsection{Ablation Study}
\begin{figure}[!ht]
  \centering
  \begin{subfigure}[t]{0.325\textwidth}
    \centering
    \includegraphics[width=\linewidth]{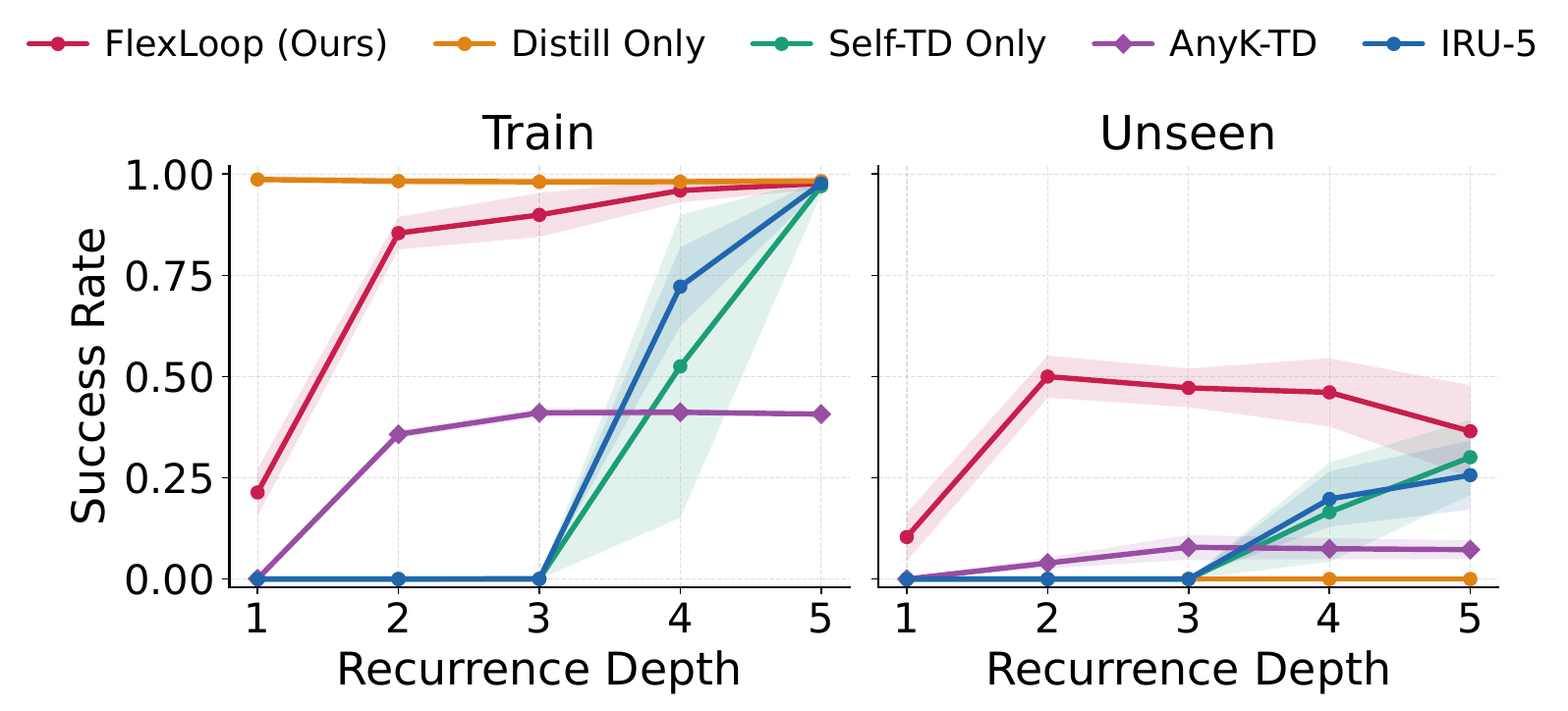}
    \caption{Component Ablation.}\label{fig:ablation_a}
  \end{subfigure}
  % \hfill
  \begin{subfigure}[t]{0.325\textwidth}
    \centering
    \includegraphics[width=\linewidth]{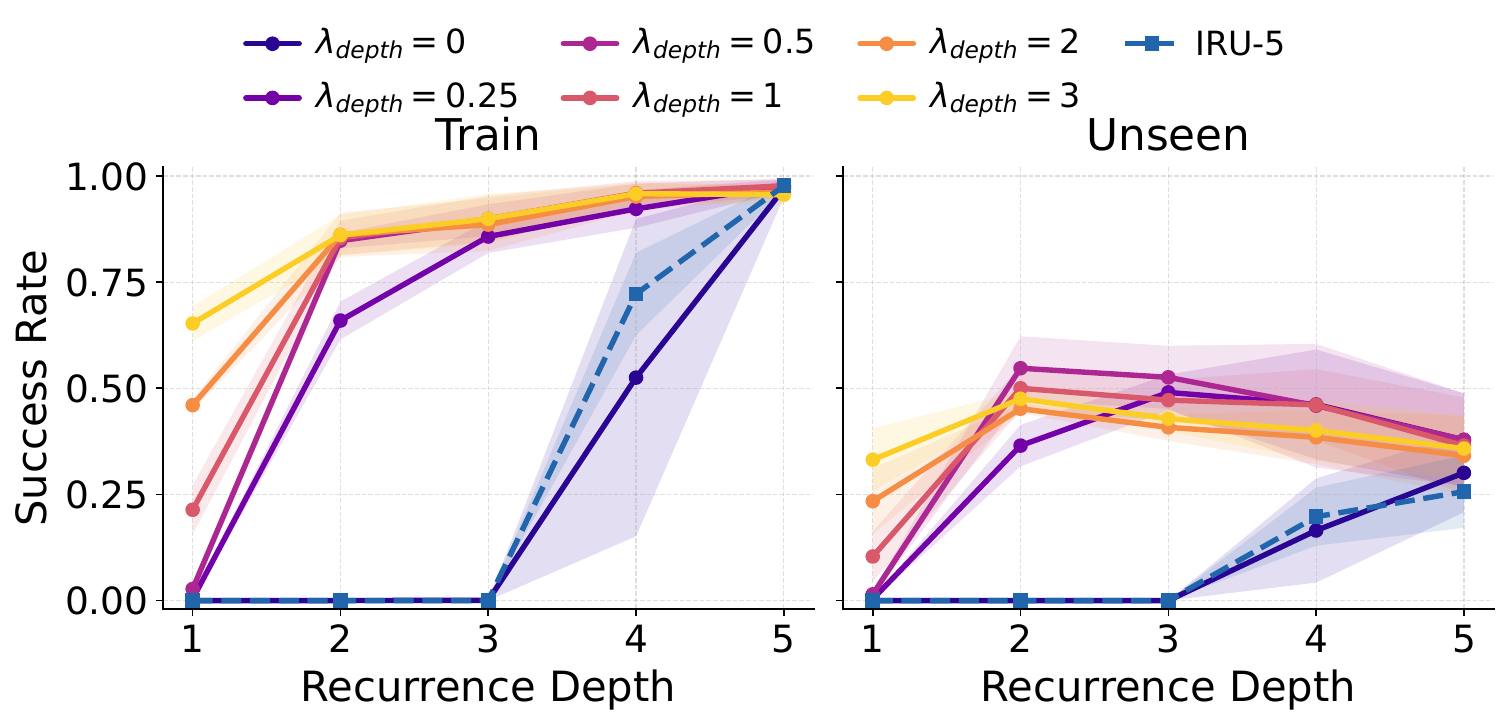}
    \caption{Sensitivity to $\lambda_{\rm depth}$.}\label{fig:ablation_b}
  \end{subfigure}
  % \hfill
  \begin{subfigure}[t]{0.325\textwidth}
    \centering
    \includegraphics[width=\linewidth]{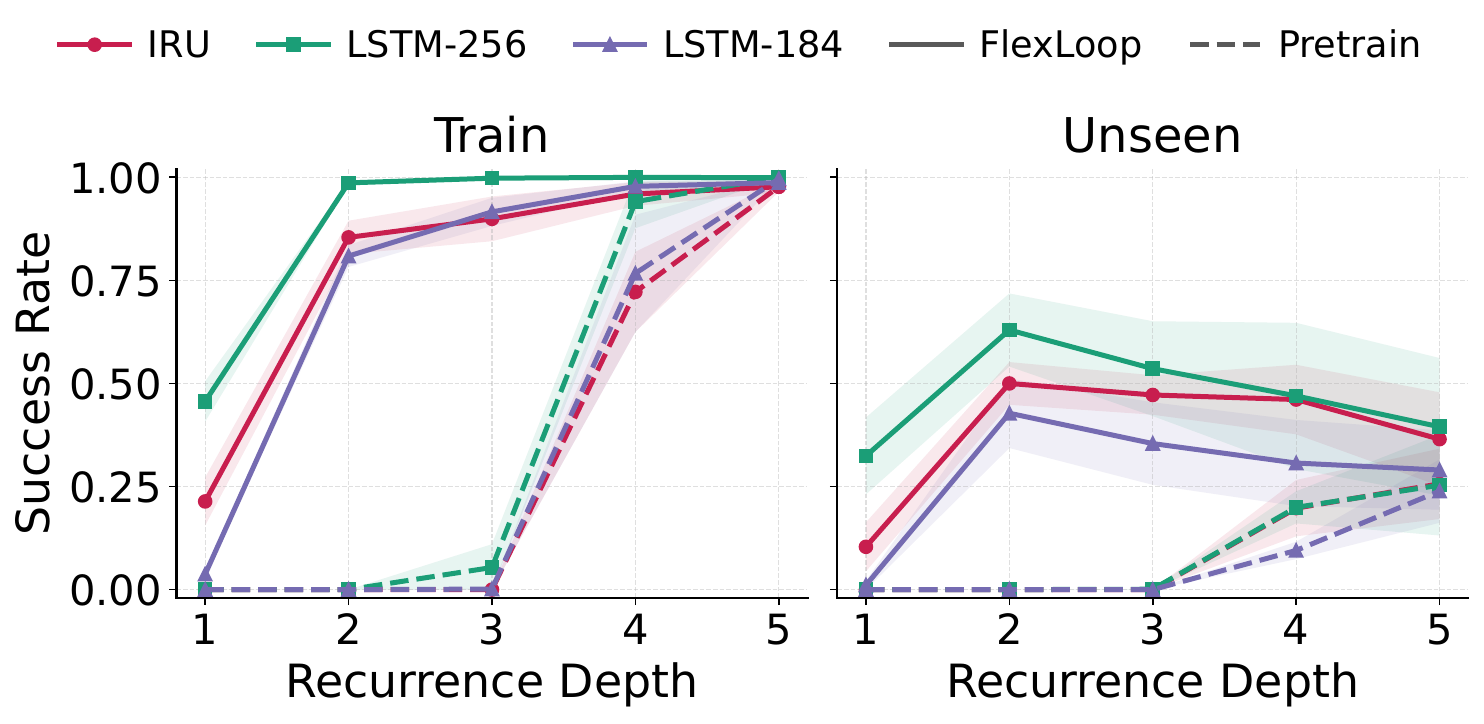}
    \caption{Architectural Generality.}\label{fig:ablation_c}
  \end{subfigure}
  \caption{Ablation studies of FlexLoop. Solid lines denote mean performance, while shaded regions indicate one standard deviation.}
  \label{fig:ablation}
\end{figure}
Based on BoxPick-Exact-4, we study FlexLoop's key components, $\lambda_{\rm depth}$ sensitivity, and architectural generality. Additional results on from-scratch training and pretraining recurrent depth are provided in
Appendix~\ref{app:ex_ablation}.
\paragraph{Contribution of Key Components.}
We ablate the key components of FlexLoop, including \textit{Distill Only}, which performs depth-wise distillation from the pretrained full-depth policy, \textit{Self-TD Only}, which solely keeps the original RL training, and \textit{AnyK-TD}, which jointly optimizes all recurrent depths with TD learning. As shown in Figure~\ref{fig:ablation_a}, \textit{Self-TD Only} preserves the pretrained policy's generalization but provides limited shallow-depth improvement, showing that depth elasticity does not emerge solely from the original RL objective. Meanwhile, \textit{AnyK-TD} and \textit{Distill Only} each suffer from limitations in full-depth capability or unseen-task generalization. These highlight the distinct roles of the two objectives. Continued RL optimization preserves full-depth capability and generalization, while depth-wise policy distillation transfers decision quality to shallower depths.
% FlexLoop overcomes these by continuing the original RL objective to preserve full-depth capability and generalization, while performing adjacent-depth policy distillation to transfer decision quality to shallower depths.

\paragraph{Sensitivity to $\lambda_{\rm depth}$.}
We further study the effect of the depth distillation weight $\lambda_{\rm depth}$. As shown in Figure~\ref{fig:ablation_b}, FlexLoop remains effective across a broad range of $\lambda_{\rm depth}$ values. Increasing $\lambda_{\rm depth}$ strengthens depth-wise distillation and consistently improves shallow-depth performance on training tasks. However, excessive distillation can hurt generalization to unseen tasks, where a moderate $\lambda_{\rm depth}$ achieves a better balance between shallow-depth adaptation and generalization.

\paragraph{Generality across Architectures.}
We further apply FlexLoop to LSTM-based policies \citep{lstm} to evaluate its architectural generality. LSTM-256 and LSTM-184 denote backbones with hidden sizes of 256 and 184, respectively, where LSTM-184 has a comparable parameter count to IRU. As shown in Figure~\ref{fig:ablation_c}, both LSTM variants exhibit depth specialization during pretraining, while FlexLoop consistently improves shallow-depth performance and preserves
full-depth capability across architectures. Moreover, LSTM-256 benefits more from FlexLoop than LSTM-184, suggesting that larger capacity allows a policy
to acquire stronger depth elasticity.

\subsection{Visualization of Adaptive Inference}
\begin{figure}[!t]
  \centering
  \begin{subfigure}[t]{0.45\textwidth}
    \centering
    \includegraphics[width=\linewidth]{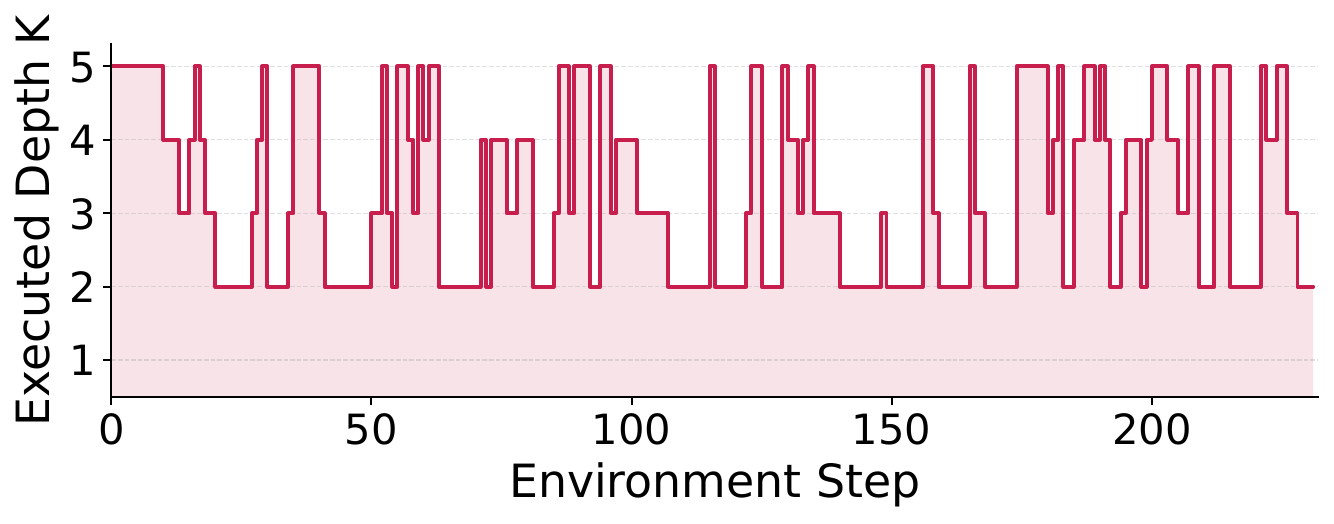}
    \caption{Scene.}
  \end{subfigure}
  \hfill
  \begin{subfigure}[t]{0.24\textwidth}
    \centering
    \includegraphics[width=.9\linewidth]{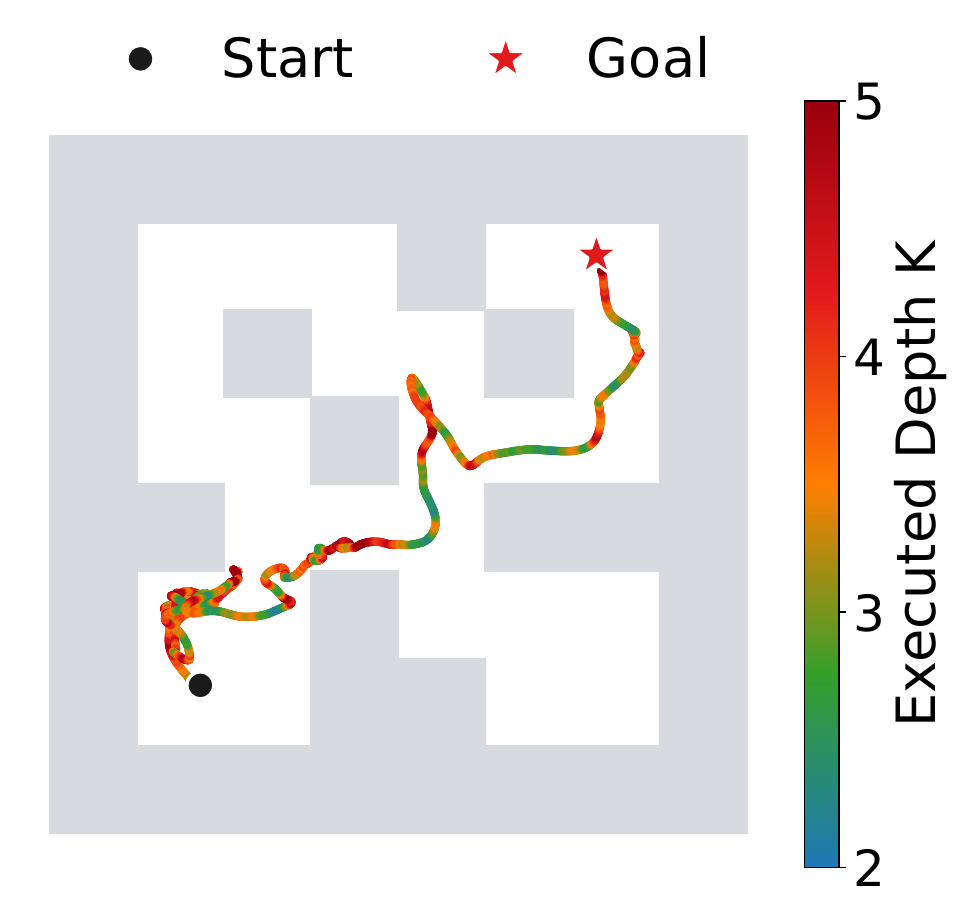}
    \caption{AntMaze.}
  \end{subfigure}
  \hfill
  \begin{subfigure}[t]{0.28\textwidth}
    \centering
    \includegraphics[width=\linewidth]{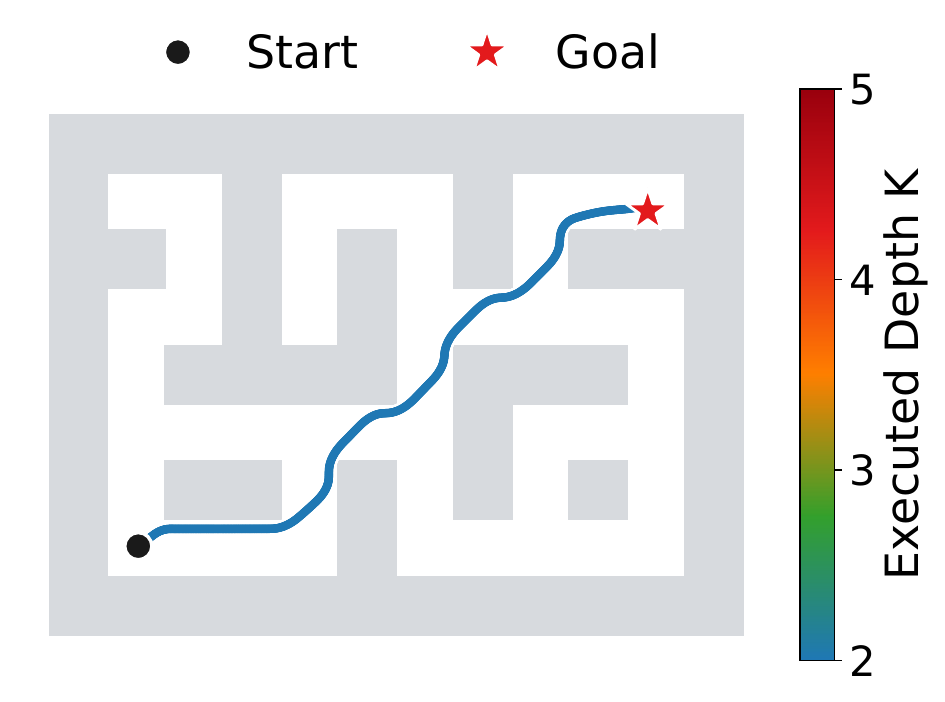}
    \caption{PointMaze.}
  \end{subfigure}
  \caption{
    State-wise adaptive recurrent-depth allocation across OGBench tasks. FlexLoop allocates more recurrent computation to states requiring more complex decisions.
    }
  \label{fig:adaptive_vis}
\end{figure}
To qualitatively understand how FlexLoop utilizes its learned depth elasticity, we visualize the recurrent computation allocated during successful trajectories on three OGBench tasks. Figure~\ref{fig:adaptive_vis} shows the recurrent depth selected by FlexLoop at each environment step, with representative trajectory segments visualized frame by frame in
Appendix~\ref{app:adaptive_vis}. The allocated computation appears to track state-level decision difficulty. In Scene, FlexLoop increases recurrent depth at skill boundaries, such as goal switching, while using only three or two recurrent steps during simpler execution phases such as button pressing. Similarly, deeper computation is concentrated near turns and goals in AntMaze, whereas straight navigation uses shallower inference. On the simpler PointMaze task, FlexLoop consistently terminates after two recurrent steps. These trajectories suggest that FlexLoop reallocates recurrent computation toward challenging decisions and terminates earlier once adjacent-depth policies have stabilized, rather than uniformly reducing computation.
% performance vs avg K
% fixed-K points on same plot
% wall-clock stress test

% \clearpage
\section{Conclusion and Future Directions}
\label{sec:conclusion}

In this work, we identify recurrent-depth specialization in pretrained looped RL policies and propose FlexLoop, a post-training framework that induces depth elasticity through continued RL optimization and adjacent-depth policy distillation. Across 30 online and offline long-horizon goal-conditioned environments, FlexLoop preserves full-depth capability while enabling reliable shallower-depth inference, reducing average recurrent computation by up to 43\% and achieving up to $1.34\times$ wall-clock speedup, while maintaining competitive or better task performance. Future work includes extending depth elasticity beyond the trained depth, developing learning-based adaptive inference mechanisms beyond the current depth-consistency criterion, and studying depth elasticity in larger-scale looped architectures such as transformer-based models.

\FloatBarrier

\begingroup
\small
\urlstyle{same}
\bibliographystyle{dilab_ref}
\bibliography{references}
\endgroup

\makeappendixtoc
\appendix
\section{Proofs of the Propositions}\label{app:proof}
\subsection{Proof of Proposition~\ref{prop:no_intermediate}}
\label{app:proof_no_intermediate}
\begin{proof}
Since the return is bounded, $\sup_{\pi} J(\pi)$ and $\inf_{\pi} J(\pi)$ are both finite, where $\sup_\pi$ and $\inf_\pi$ are taken over all stationary goal-conditioned Markov policies. Hence, for any $\epsilon>0$ there exist policies $\pi^{+}$ and $\pi^{-}$ such that
\begin{equation}
    J(\pi^{+})
    \geq
    \sup_{\pi}J(\pi)-\epsilon,
    \qquad
    J(\pi^{-})
    \leq
    \inf_{\pi}J(\pi)+\epsilon.
    \label{eq:extreme_policies}
\end{equation}

We construct a shared-weight recurrent policy $\Pi=(F,G)$, where $F$
denotes the shared recurrent block and $G$ the policy output head.
Let the recurrent state contain a scalar $h_k$, initialized as $h_0=0$,
and define
\begin{equation}
    h_{k+1}=F(h_k,s,g)=h_k+1,
\end{equation}
where $F$ simply ignores $(s,g)$, yielding $h_k=k$ at recurrent depth $k$.

Define the shared policy output head as
\begin{equation}
    G(\cdot\mid s,g,h)
    =
    \begin{cases}
        \pi^{-}(\cdot\mid s,g), & h<K,\\
        \pi^{+}(\cdot\mid s,g), & h\geq K.
    \end{cases}
\end{equation}
The depth-$k$ policy induced by $\Pi$ is
\begin{equation}
    \Pi_k(\cdot\mid s,g)
    :=G(\cdot\mid s,g,h_k).
\end{equation}
Therefore, the policy obtained after $k$ recurrent steps is
\[
    \Pi_k=\pi^{-},
    \qquad \forall k<K,
\]
while at the full recurrent depth,
\[
    \Pi_K=\pi^{+}.
\]

Using \cref{eq:extreme_policies},
\[
    J(\Pi_K)
    =
    J(\pi^{+})
    \geq
    \sup_{\pi}J(\pi)-\epsilon,
\]
and, for every $k<K$,
\[
    J(\Pi_k)
    =
    J(\pi^{-})
    \leq
    \inf_{\pi}J(\pi)+\epsilon.
\]
Thus, a shared-weight recurrent policy can be arbitrarily close to optimal
at its full depth while all of its shallower policies are arbitrarily close
to the minimum achievable return.
\end{proof}
\subsection{Proof of Proposition~\ref{prop:depth_bound}}
\label{app:proof_shallow}
\begin{proof}
%凝练第一段，然后把k范围加到里面去让说法更完整
For analysis, we use the equivalent discounted infinite-horizon formulation
with a zero-reward absorbing terminal state, augmenting the state with time
when needed, and write $x=(s,g)$ for compactness. Define the worst-case total-variation discrepancy between adjacent depths as
\[\alpha_j=\sup_x
    D_{\mathrm{TV}}
    \left(
        \pi_j(\cdot\mid x),
        \pi_{j+1}(\cdot\mid x)
    \right),\quad 1\le j<K.
\]
By Pinsker's inequality,
\[
    \alpha_j
    \leq
    \sup_x
    \sqrt{
        \frac{1}{2}
        D_{\mathrm{KL}}
        \left(
            \pi_{j+1}(\cdot\mid x)
            \,\Vert\,
            \pi_j(\cdot\mid x)
        \right)
    }
    \leq
    \sqrt{\frac{\delta_{\max}}{2}}.
\]
For any $1\leq k<K$, the triangle inequality across recurrent depths gives
\begin{equation}
    D_{\mathrm{TV}}
    \left(
        \pi_k(\cdot\mid x),
        \pi_K(\cdot\mid x)
    \right)
    \leq
    \sum_{j=k}^{K-1}
    \alpha_j \leq
    (K-k)\sqrt{\frac{\delta_{\max}}{2}}.
    \label{eq:depth_tv_bound}
\end{equation}
We next relate this policy discrepancy to return. By the performance
difference lemma \citep{kakade2002approximately}, for any two policies
$\pi$ and $\pi'$,
\begin{equation}
    J(\pi)-J(\pi')
    =
    \frac{1}{1-\gamma}
    \mathbb{E}_{x\sim d_{\pi}}
    \left[
        \mathbb{E}_{a\sim\pi(\cdot\mid x)}
        A^{\pi'}(x,a)
    \right],
    \label{eq:pdl}
\end{equation}
where $d_{\pi}$ denotes the normalized discounted state--goal occupancy
measure under $\pi$, and $A^{\pi'}$ is the advantage of $\pi'$.

Taking $\pi=\pi_k$ and $\pi'=\pi_K$, and using
$\mathbb{E}_{a\sim\pi_K(\cdot\mid x)}
A^{\pi_K}(x,a)=0$, we have
\begin{align}
    \left|
    \mathbb{E}_{a\sim\pi_k(\cdot\mid x)}
    A^{\pi_K}(x,a)
    \right|&=
    \left|
    \mathbb{E}_{a\sim\pi_k(\cdot\mid x)}
    A^{\pi_K}(x,a)-\mathbb{E}_{a\sim\pi_K(\cdot\mid x)}
A^{\pi_K}(x,a)
    \right|\nonumber\\
    &=
    \left|
    \mathbb{E}_{a\sim\pi_k(\cdot\mid x)}
    Q^{\pi_K}(x,a)
    -
    \mathbb{E}_{a\sim\pi_K(\cdot\mid x)}
    Q^{\pi_K}(x,a)
    \right|\nonumber.
\end{align}

Since $|r|\leq R_{\max}$, $\gamma<1$,
\[
    |Q^{\pi_K}(x,a)|
    \leq
    \frac{R_{\max}}{1-\gamma},
\]
and hence
\[
    \left|
    \mathbb{E}_{a\sim\pi_k(\cdot\mid x)}
    Q^{\pi_K}(x,a)
    -
    \mathbb{E}_{a\sim\pi_K(\cdot\mid x)}
    Q^{\pi_K}(x,a)
    \right|\leq
    \frac{2R_{\max}}{1-\gamma}\TV\left(
        \pi_k(\cdot\mid x),
        \pi_K(\cdot\mid x)
    \right).
\]
Therefore,
\begin{equation}
    \left|
    \mathbb{E}_{a\sim\pi_k(\cdot\mid x)}
    A^{\pi_K}(x,a)
    \right|
    \leq
    \frac{2R_{\max}}{1-\gamma}
    D_{\mathrm{TV}}
    \left(
        \pi_k(\cdot\mid x),
        \pi_K(\cdot\mid x)
    \right).
\end{equation}
Substituting this into \cref{eq:pdl} yields
\begin{equation}
    \left|J(\pi_k)-J(\pi_K)\right|
    \leq
    \frac{2R_{\max}}{(1-\gamma)^2}
    \sup_x
    D_{\mathrm{TV}}
    \left(
        \pi_k(\cdot\mid x),
        \pi_K(\cdot\mid x)
    \right).
\end{equation}
Combining with \cref{eq:depth_tv_bound} gives
\[
    \left|J(\pi_k)-J(\pi_K)\right|
    \leq
    \frac{2R_{\max}}{(1-\gamma)^2}
    (K-k)
    \sqrt{\frac{\delta_{\max}}{2}},
\]
which proves the claim.
\end{proof}

% \newpage
\section{Experiment Details}\label{app:exp_details}
\subsection{Environments}\label{app:env_details}
Following \citet{ghugare2026on}, we evaluate FlexLoop on $30$ long-horizon goal-conditioned environments from BoxPick \citep{bortkiewicz2025temporaldifferencelearninggold}, LightsOut \citep{anderson1998turning}, and OGBench \citep{ICLR2025_ecd92623}. 

\paragraph{BoxPick.} BoxPick is a discrete-action benchmark, where an agent navigates in a $6\times6$ grid, picks up boxes, and places them at designated goal locations. An episode is successful only when all boxes are correctly placed. We consider two types of configurations. (i) \textit{BoxPick-Exact-$n$} requires the agent to transport $n$ boxes. The training tasks involve moving boxes between two adjacent quadrants of the map, while the unseen tasks require transportation between diagonal quadrants. We evaluate $n\in\{3,4\}$. (ii) \textit{BoxPick-Generalize-$n$-$m$} requires transporting $n$ boxes in total, where training tasks only require moving $m$ boxes and unseen tasks require moving all $n$ boxes to their target locations. We evaluate $(n,m)\in\{(4,1),(4,2)\}$. These four environments provide a challenging test of long-horizon planning and generalization. We use goal-conditioned DQN as the underlying RL algorithm for this environment.

\paragraph{LightsOut.}\textit{LightsOut-4x5} is a large discrete-action decision-making benchmark with a $4\times5$ grid of 20 binary switches. At each step, the agent selects one switch, which toggles its own state and the states of all adjacent switches. An episode is successful when all switches are turned on. The training tasks consist of initial configurations that can be solved with relatively fewer steps, while the unseen tasks contain configurations requiring longer solution sequences. We use goal-conditioned PPO as the underlying RL algorithm for this environment.

\begin{wraptable}{r}{0.38\textwidth}
    \centering
    \vspace{-37pt}
\caption{Adjacent-depth policy distillation weight used for each environment.}\label{tab:lambda_depth}
\small
\begin{tabular}{lc}
\toprule
\textbf{Environment} & $\bf{\lambda_{\rm depth}}$ \\
\midrule
BoxPick-Exact-3 & 1.0 \\
BoxPick-Exact-4 & 1.0 \\
BoxPick-Generalize-4-1 & 1.0 \\
BoxPick-Generalize-4-2 & 1.0 \\
LightsOut& 0.05 \\
Scene & 0.05 \\
Cube & 0.02 \\
Puzzle & 0.2 \\
AntMaze & 0.03 \\
PointMaze & 0.3 \\
\bottomrule
\end{tabular}
\vspace{-12pt}
\end{wraptable}

\paragraph{OGBench.} 
OGBench is a continuous-action offline RL benchmark, where agents learn to achieve diverse test goals solely from static offline datasets. We evaluate five scenarios, including two navigation scenarios, \textit{AntMaze-Medium-Stitch} and \textit{PointMaze-Large-Stitch}, which test the ability to compose experiences through trajectory stitching, and three manipulation scenarios, \textit{Scene-Play}, \textit{Cube-Double-Play}, and \textit{Puzzle-4x4-Play}, which evaluate fine-grained control and compositional generalization. Following the evaluation protocol of \citet{ghugare2026on}, we report the average success rate over five task environments for each scenario. We use goal-conditioned IQL as the underlying RL algorithm for these environments.

\subsection{Implementation Details}\label{app:implementation}
We build our implementation upon the official codebase by \citet{ghugare2026on} and reproduce the IRU baseline before implementing FlexLoop. Unless otherwise stated, we use the same training steps, environment configurations, and hyperparameters as reported in the IRU paper. For BoxPick, the temperature $\tau$ used in \cref{eq:softpolicy} is set to $0.25$. For OGBench, we use a hidden size of 256 instead of the hidden size of 64 reported in the original paper, as the former produces results closer to those reported for the benchmark. For on-policy GC-PPO, we additionally replace the GAE-based value targets used in the value loss with value estimates from the pretrained IRU model to stabilize FlexLoop training. The policy optimization objective remains unchanged and continues to use standard on-policy PPO updates.

The key hyperparameter of FlexLoop is the adjacent-depth policy distillation weight $\lambda_{\rm depth}$. The selected values for each environment are reported in Table~\ref{tab:lambda_depth}.

\section{Extended Results}\label{app:ex_res}
In this section, we provide detailed environment-wise evaluation results, complementing the aggregated results reported in the main paper.
% \newpage
\subsection{Full training curves for Figure \ref{fig:res_kepoch}}\label{app:res_kepoch}
\begin{figure}[!ht]
  \centering
    \begin{subfigure}[t]{0.49\textwidth}
    \centering
    \includegraphics[width=\linewidth]{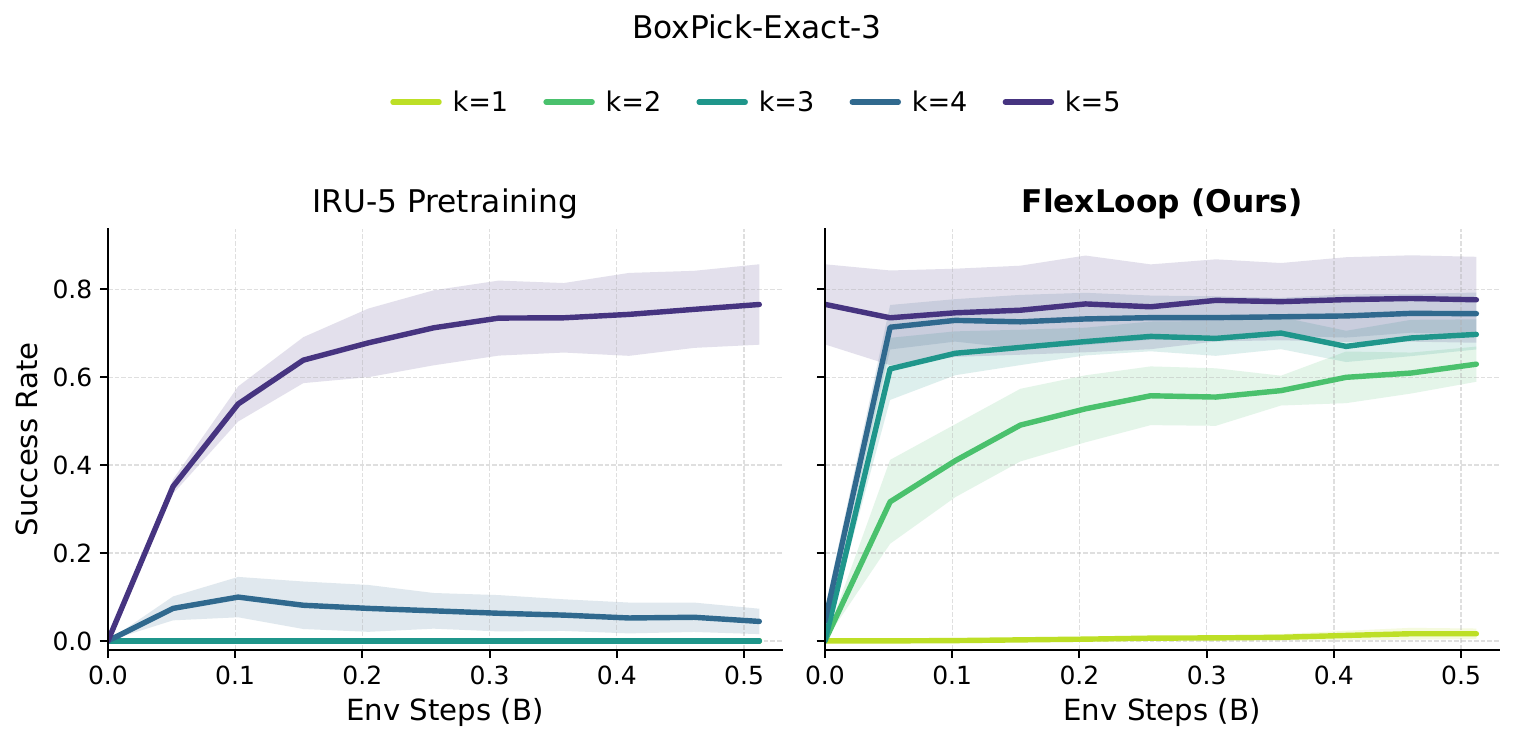}
  \end{subfigure}
  \hfill
  \begin{subfigure}[t]{0.49\textwidth}
    \centering
    \includegraphics[width=\linewidth]{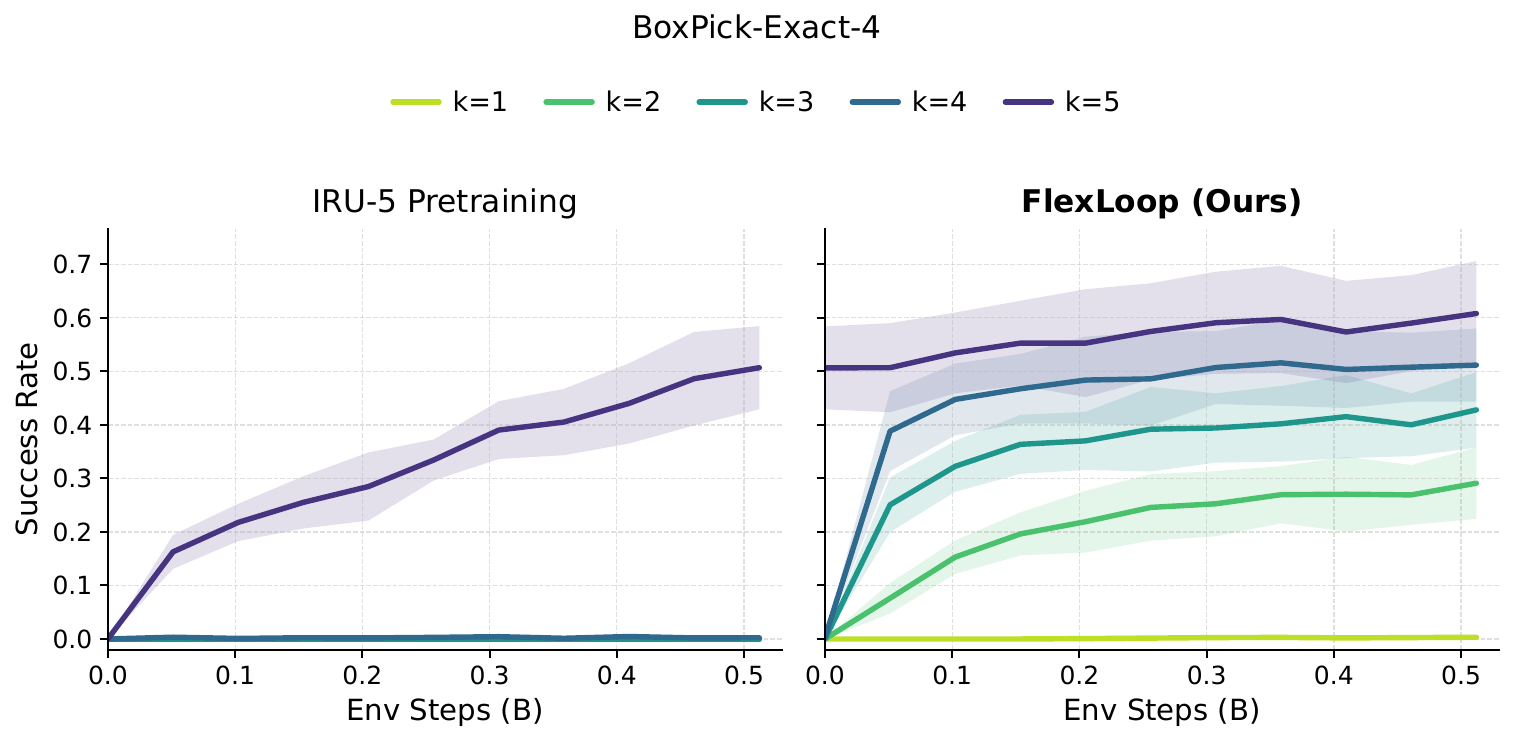}
  \end{subfigure}\\
  \begin{subfigure}[t]{0.49\textwidth}
    \centering
    \includegraphics[width=\linewidth]{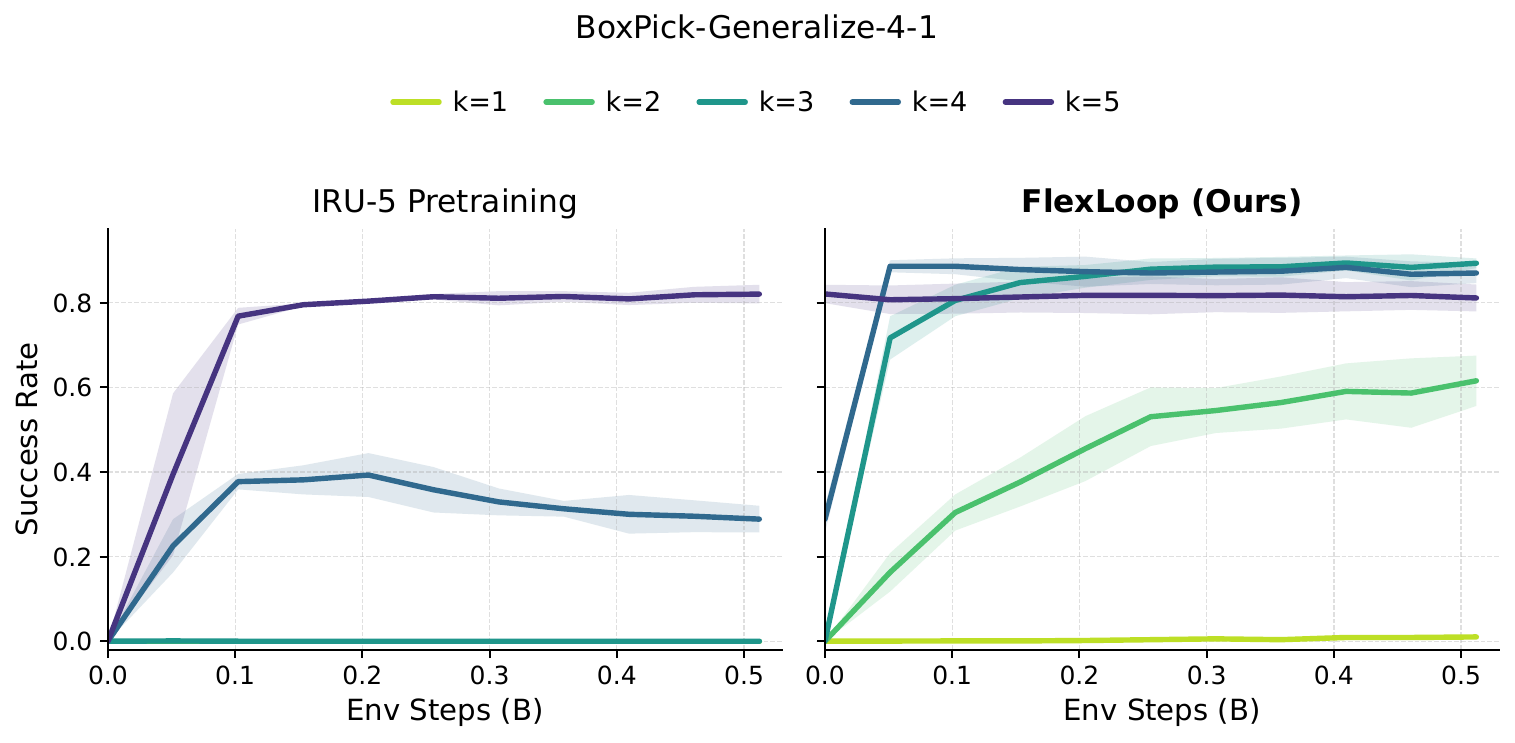}
  \end{subfigure}
  \hfill
  \begin{subfigure}[t]{0.49\textwidth}
    \centering
    \includegraphics[width=\linewidth]{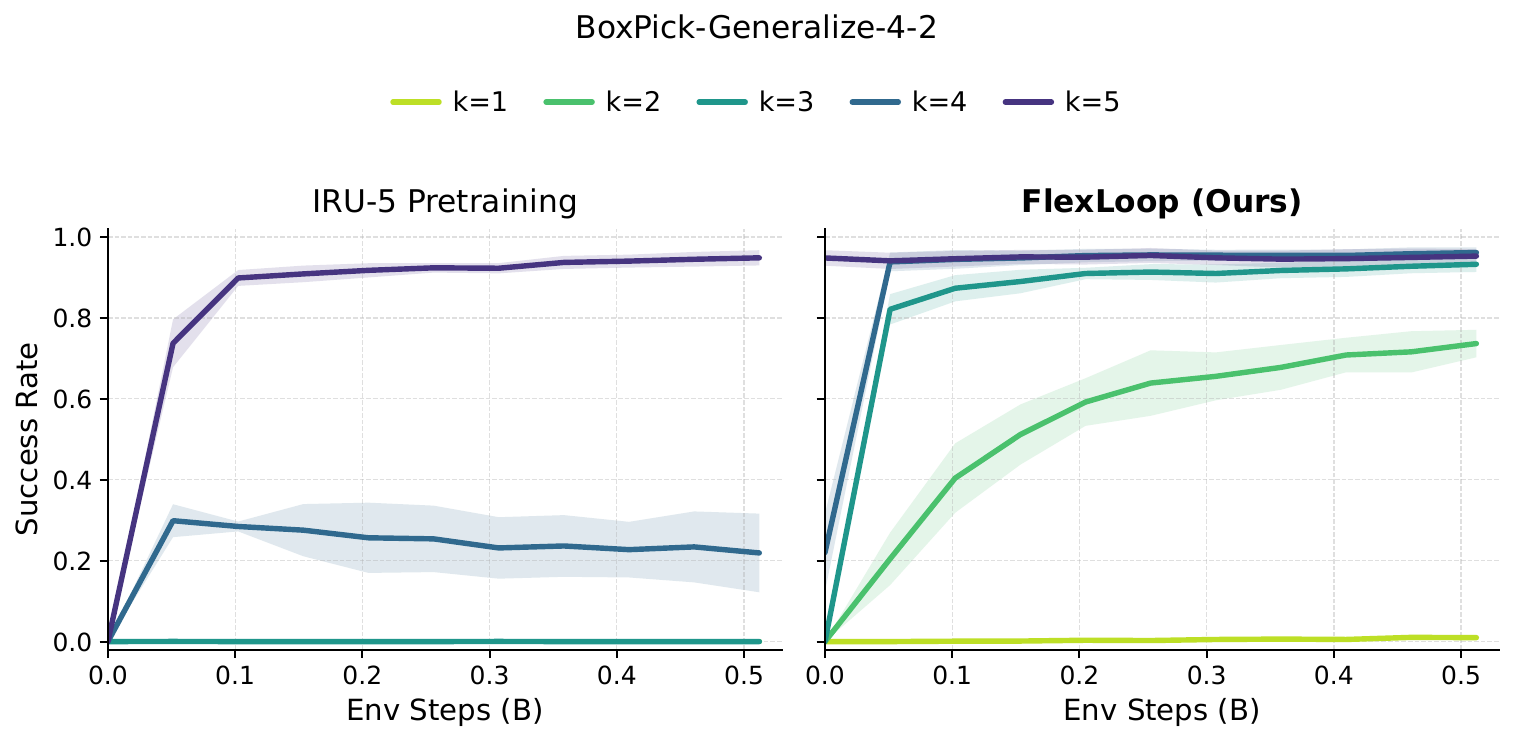}
  \end{subfigure}\\
  \begin{subfigure}[t]{0.49\textwidth}
    \centering
    \includegraphics[width=\linewidth]{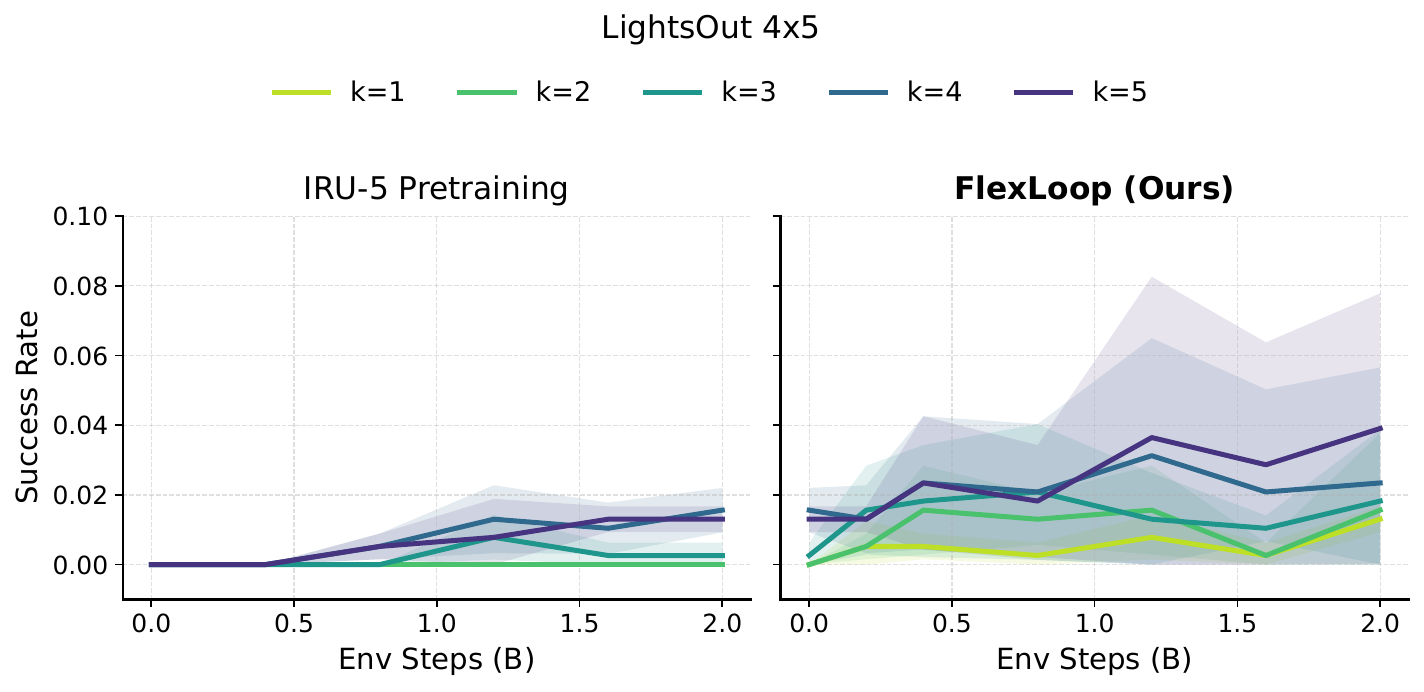}
  \end{subfigure}
  \hfill
  \caption{Depth-wise decision performance of pretrained IRU-$5$ and FlexLoop on unseen tasks. Solid lines denote mean success rates and shading indicates one standard deviation. FlexLoop improves shallow-depth decision elasticity while maintaining out-of-distribution generalization.}
  \label{fig:res_main_ood_full}
\end{figure}

\begin{figure}[!ht]
  \centering
  \begin{subfigure}[t]{0.49\textwidth}
    \centering
    \includegraphics[width=\linewidth]{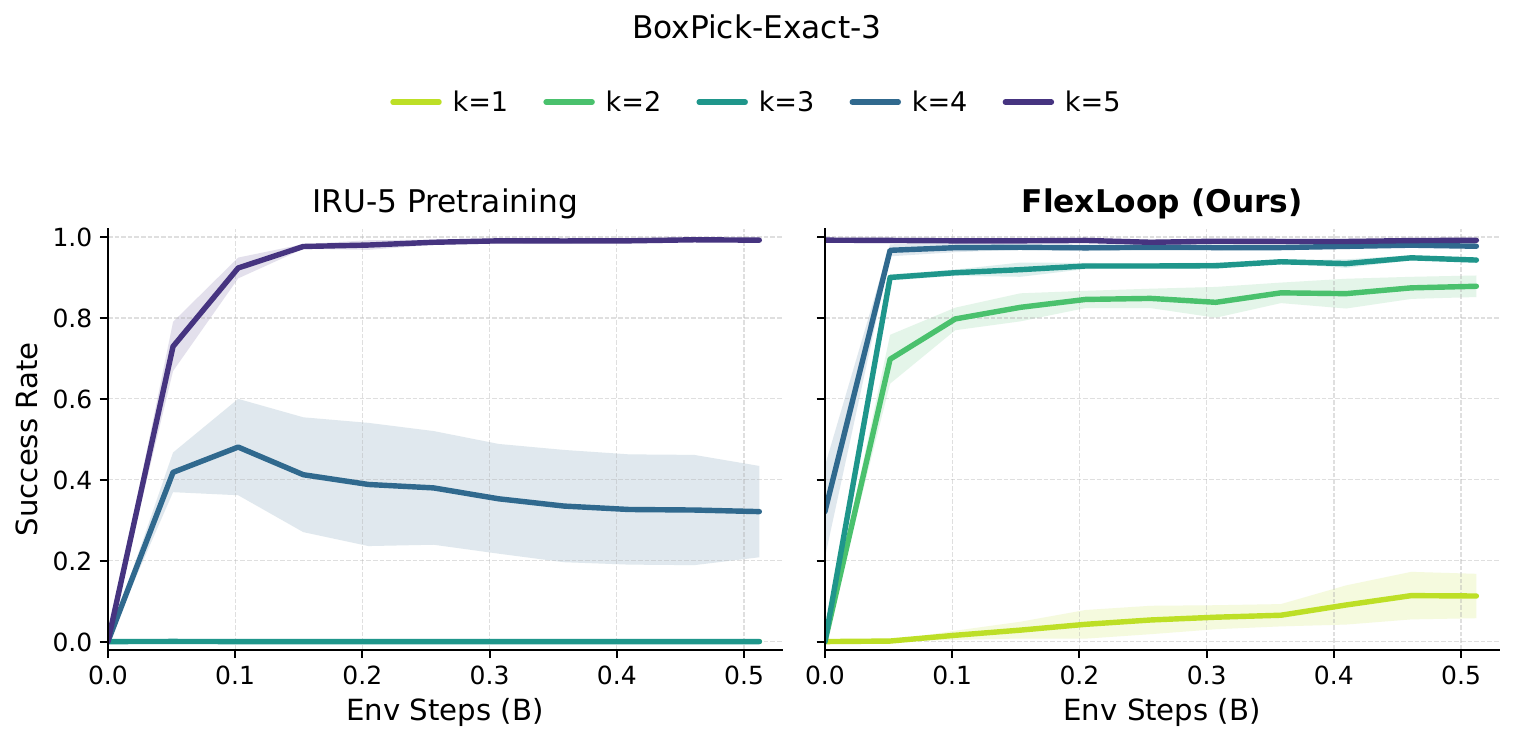}
  \end{subfigure}
  \hfill
  \begin{subfigure}[t]{0.49\textwidth}
    \centering
    \includegraphics[width=\linewidth]{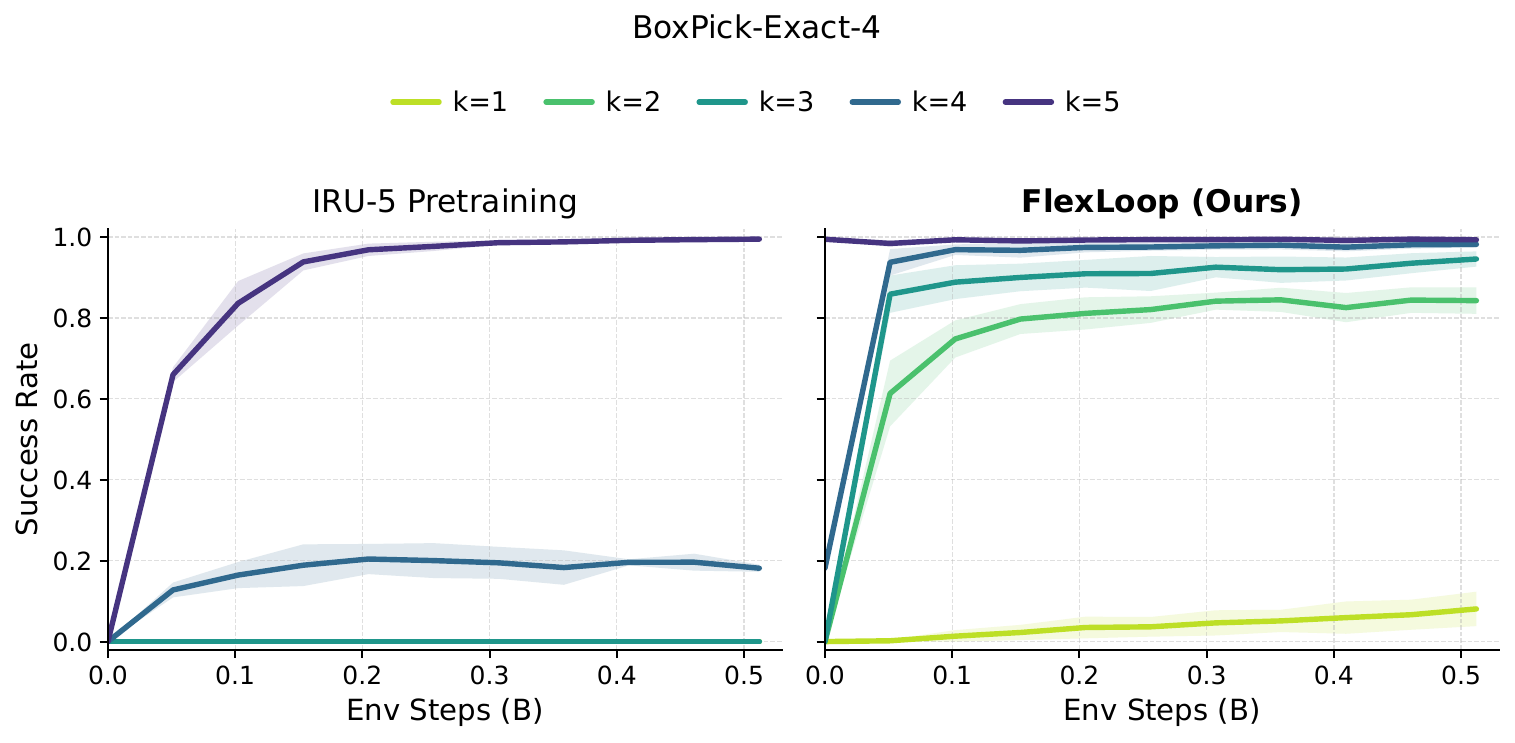}
  \end{subfigure}\\
  \begin{subfigure}[t]{0.49\textwidth}
    \centering
    \includegraphics[width=\linewidth]{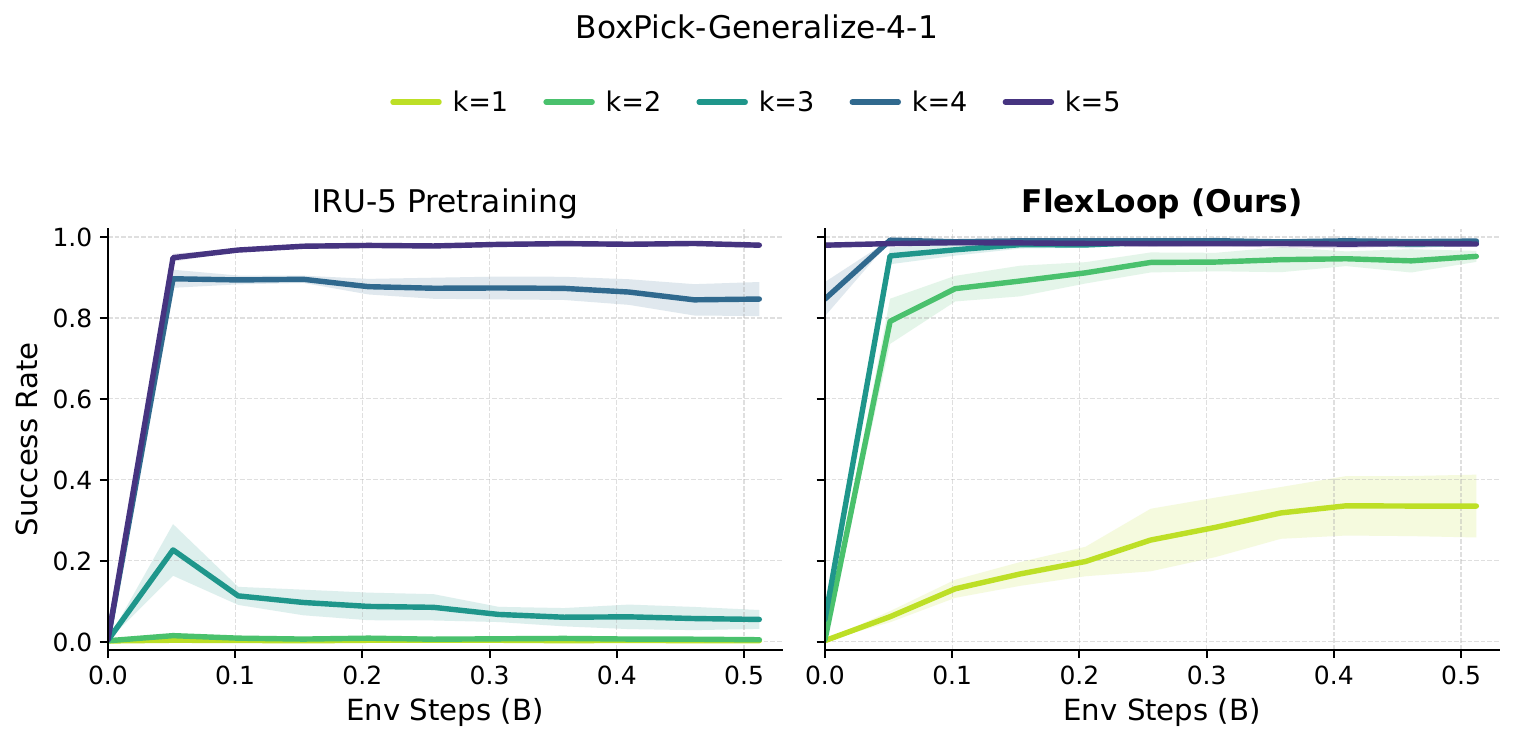}
  \end{subfigure}
  \hfill
  \begin{subfigure}[t]{0.49\textwidth}
    \centering
    \includegraphics[width=\linewidth]{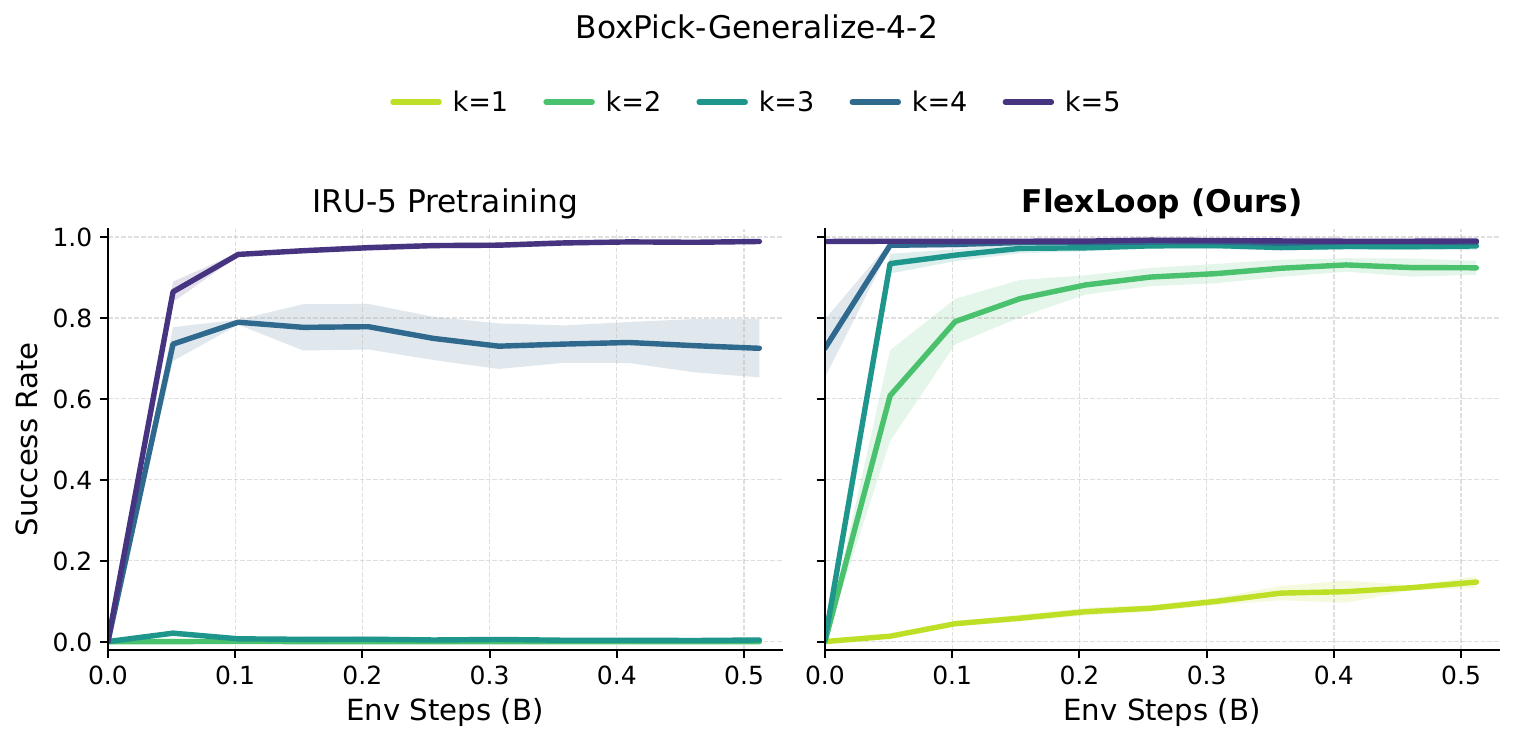}
  \end{subfigure}\\
  \begin{subfigure}[t]{0.49\textwidth}
    \centering
    \includegraphics[width=\linewidth]{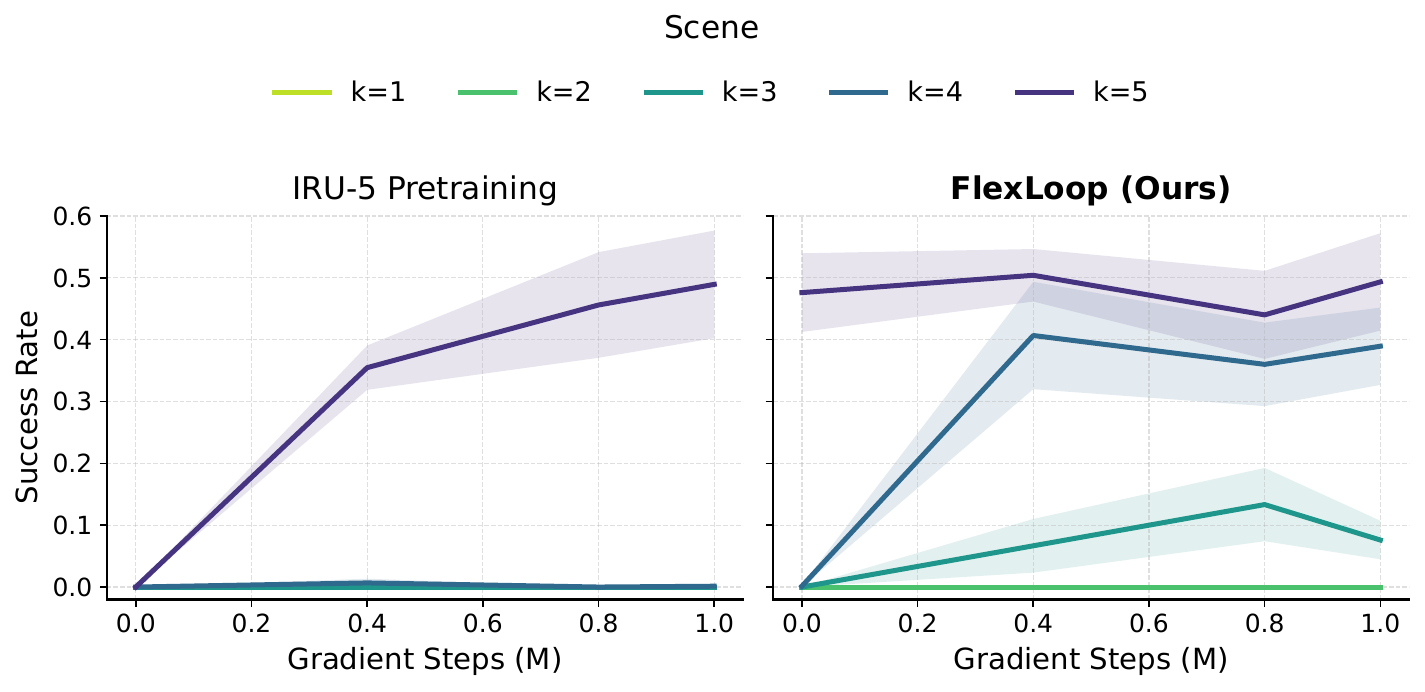}
  \end{subfigure}
  \hfill
  \begin{subfigure}[t]{0.49\textwidth}
    \centering
    \includegraphics[width=\linewidth]{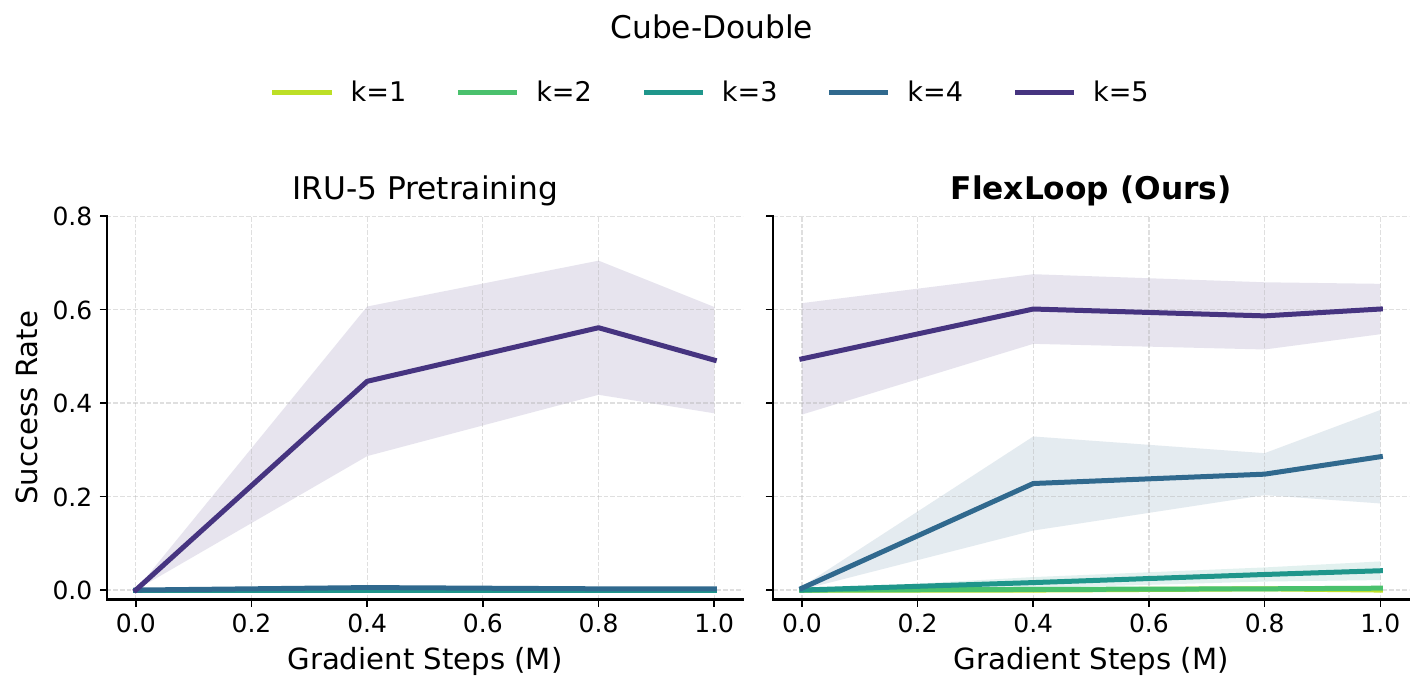}
  \end{subfigure}\\
  \begin{subfigure}[t]{0.49\textwidth}
    \centering
    \includegraphics[width=\linewidth]{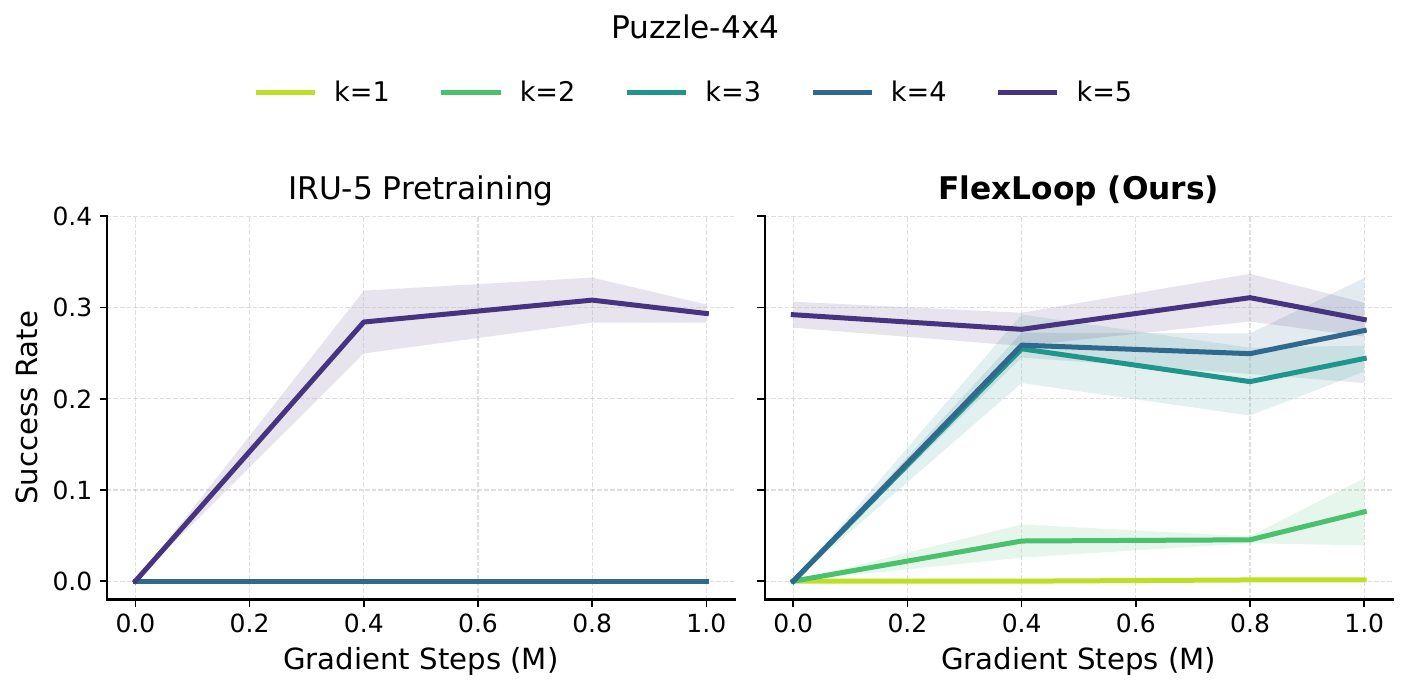}
  \end{subfigure}
  \hfill
  \begin{subfigure}[t]{0.49\textwidth}
    \centering
    \includegraphics[width=\linewidth]{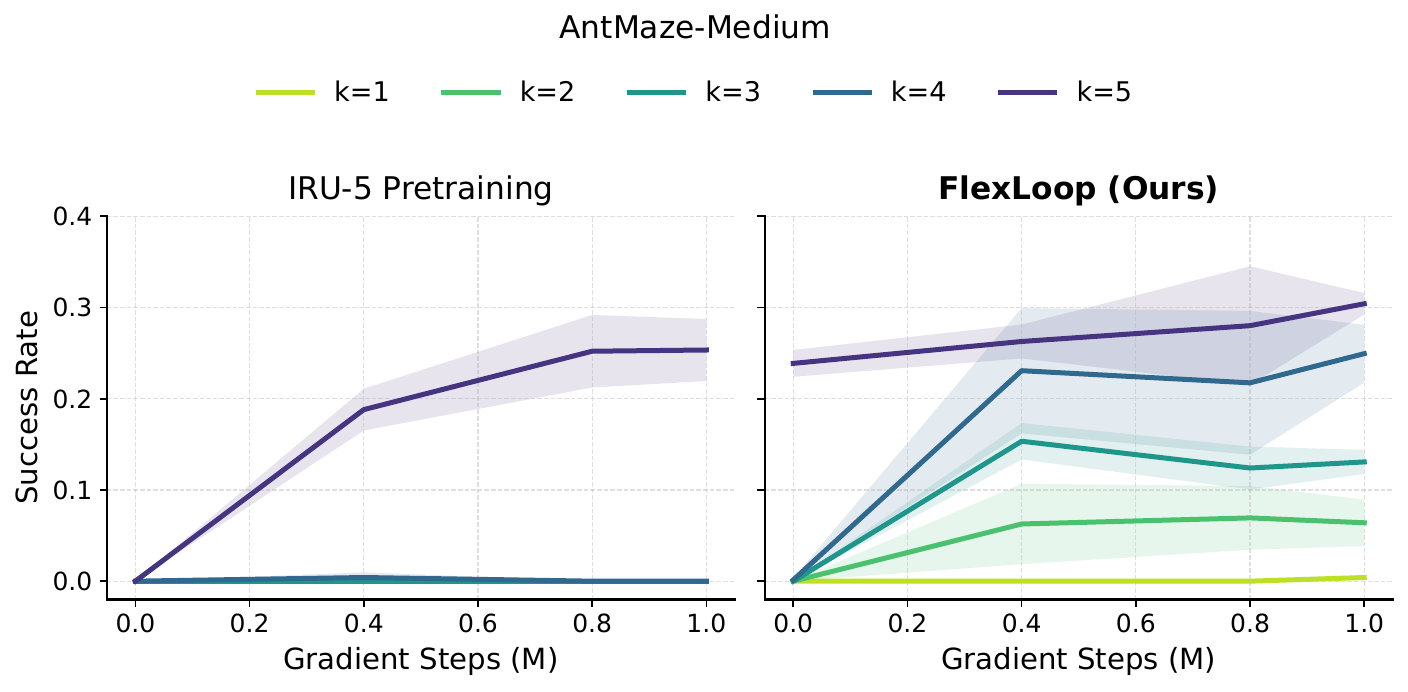}
  \end{subfigure}\\
  \begin{subfigure}[t]{0.49\textwidth}
    \centering
    \includegraphics[width=\linewidth]{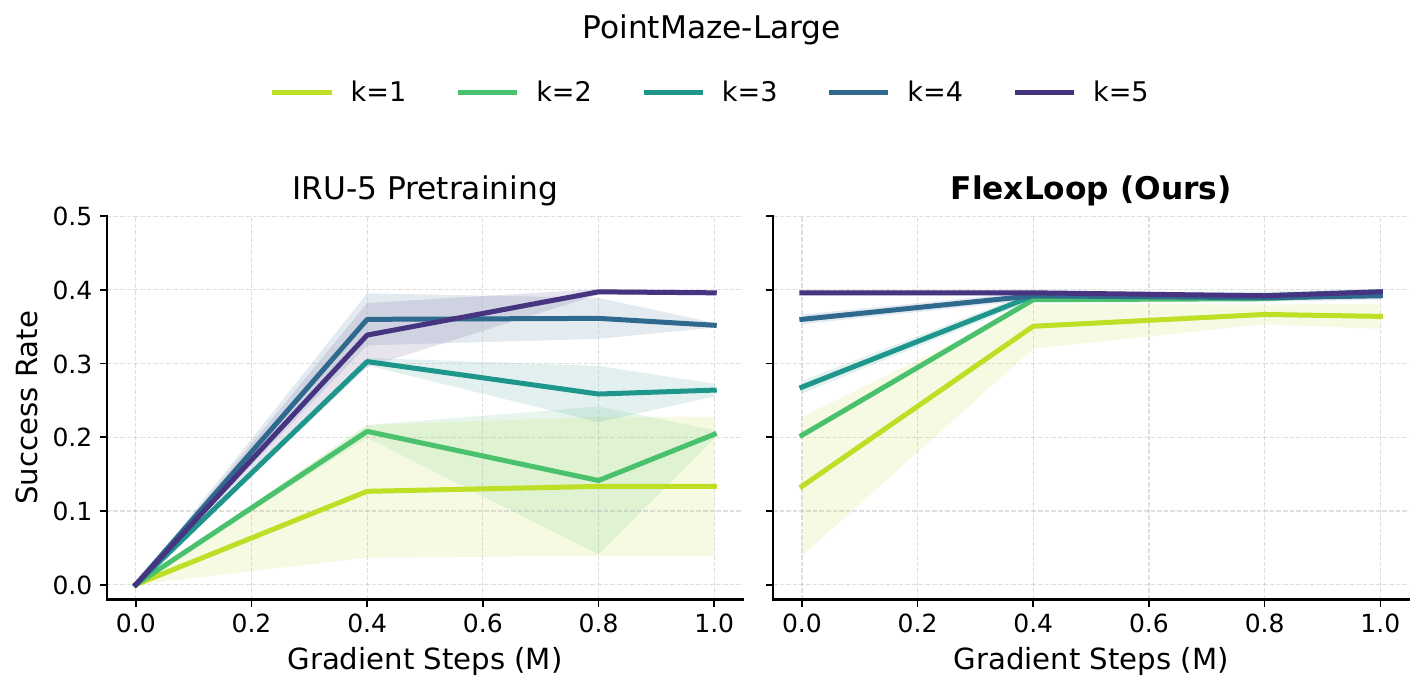}
  \end{subfigure}
  \hfill
  \begin{subfigure}[t]{0.49\textwidth}
    \centering
    \includegraphics[width=\linewidth]{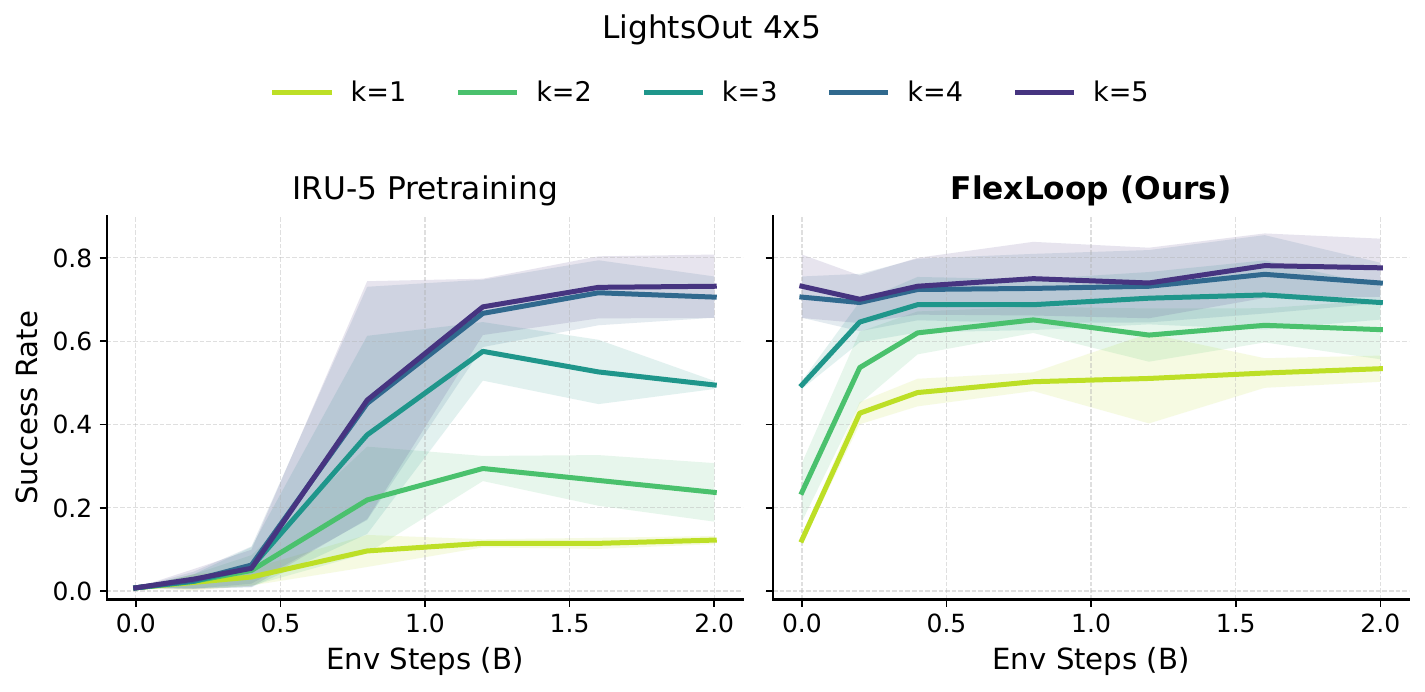}
  \end{subfigure}
  \caption{Depth-wise decision performance during IRU pretraining and FlexLoop post-training. Solid lines denote mean success rates and shading indicates one standard deviation. FlexLoop rapidly improves shallow-depth performance while preserving the full-depth capability of pretrained policies.}
  \label{fig:res_main_full}
\end{figure}

\newpage\clearpage
\subsection{Elasticity Curves}\label{app:elasticity}

We evaluate the depth-wise decision performance of the pretrained IRU-$5$ checkpoint and the corresponding FlexLoop post-trained checkpoint at each recurrent depth, and visualize the resulting depth-performance curves for comparison. Figure~\ref{fig:res_elastic} presents the curves on training tasks, while Figure~\ref{fig:res_elastic_test} presents the corresponding results on unseen tasks. These curves provide a direct characterization of the depth elasticity learned by FlexLoop.
\begin{figure}[!ht]
  \centering
  \begin{subfigure}[t]{0.24\textwidth}
    \centering
    \includegraphics[width=\linewidth]{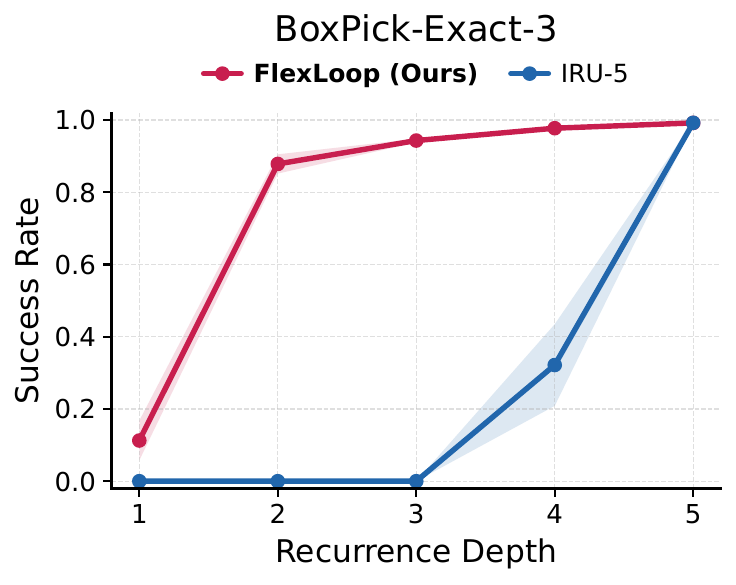}
  \end{subfigure}
  \hfill
  \begin{subfigure}[t]{0.24\textwidth}
    \centering
    \includegraphics[width=\linewidth]{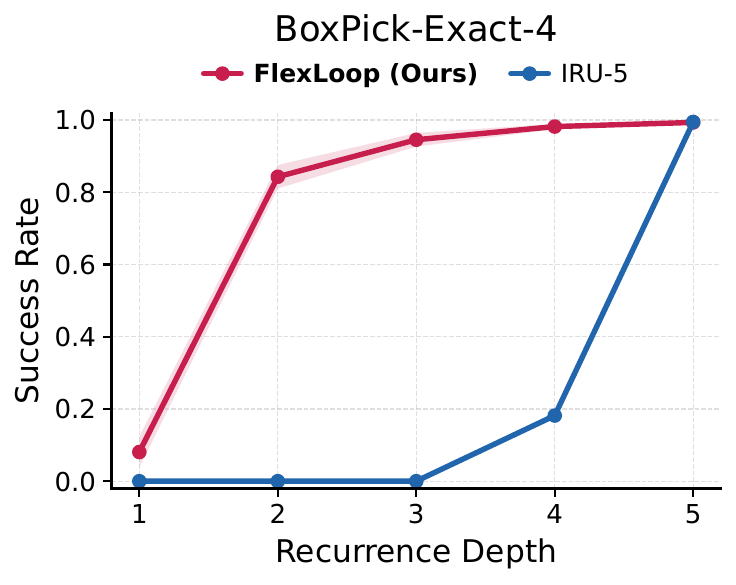}
  \end{subfigure}
  \hfill
  \begin{subfigure}[t]{0.24\textwidth}
    \centering
    \includegraphics[width=\linewidth]{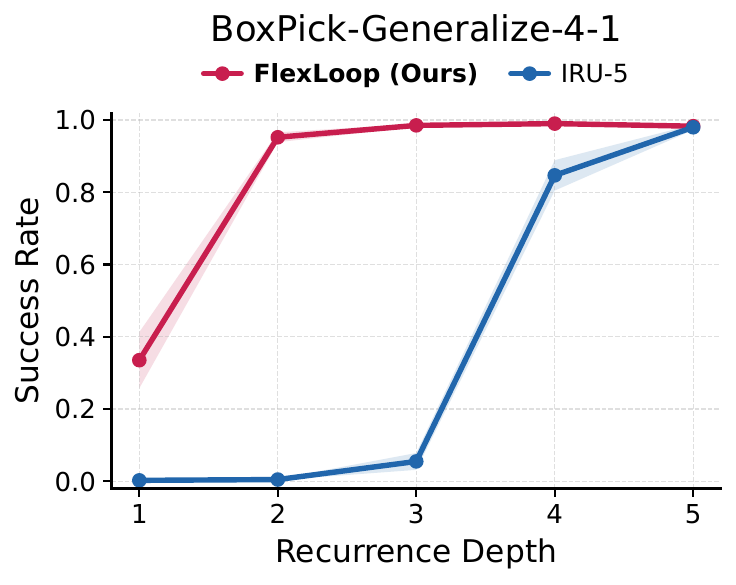}
  \end{subfigure}
  \hfill
  \begin{subfigure}[t]{0.24\textwidth}
    \centering
    \includegraphics[width=\linewidth]{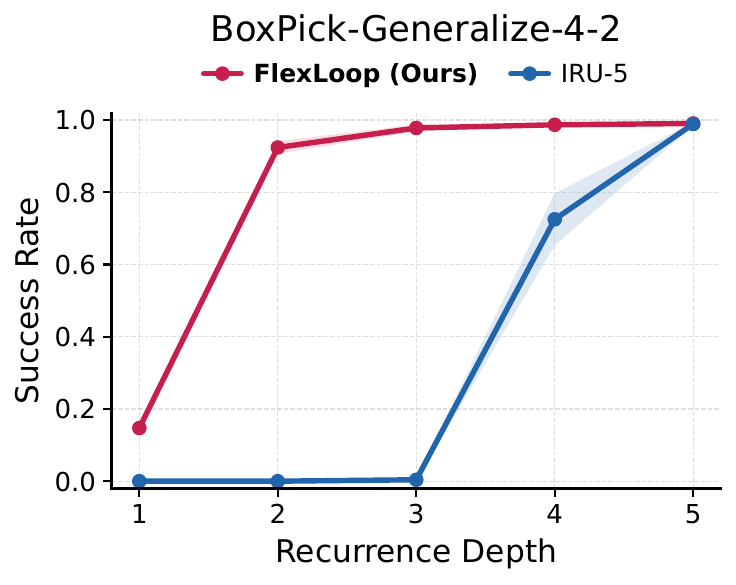}
  \end{subfigure}\\
  \begin{subfigure}[t]{0.24\textwidth}
    \centering
    \includegraphics[width=\linewidth]{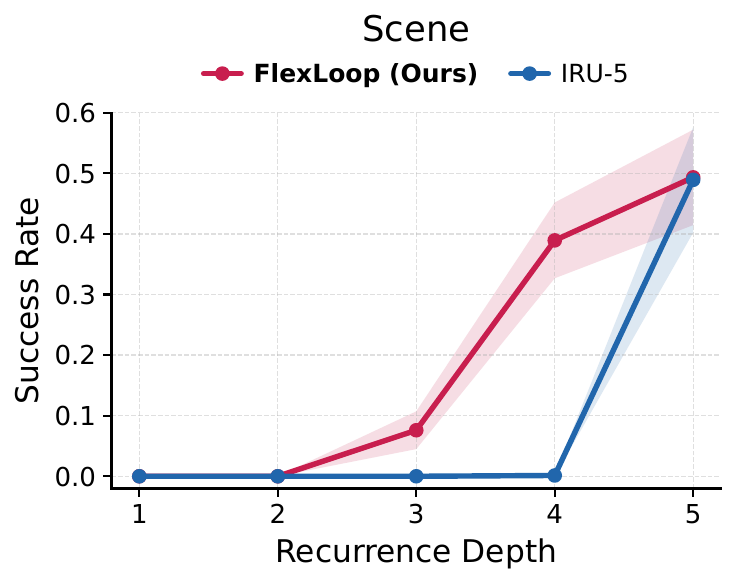}
  \end{subfigure}
  \hfill
  \begin{subfigure}[t]{0.24\textwidth}
    \centering
    \includegraphics[width=\linewidth]{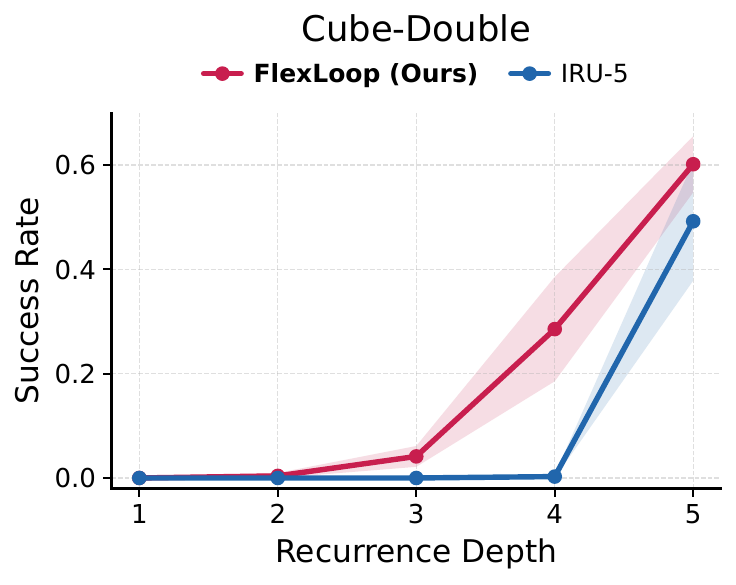}
  \end{subfigure}
  \hfill
  \begin{subfigure}[t]{0.24\textwidth}
    \centering
    \includegraphics[width=\linewidth]{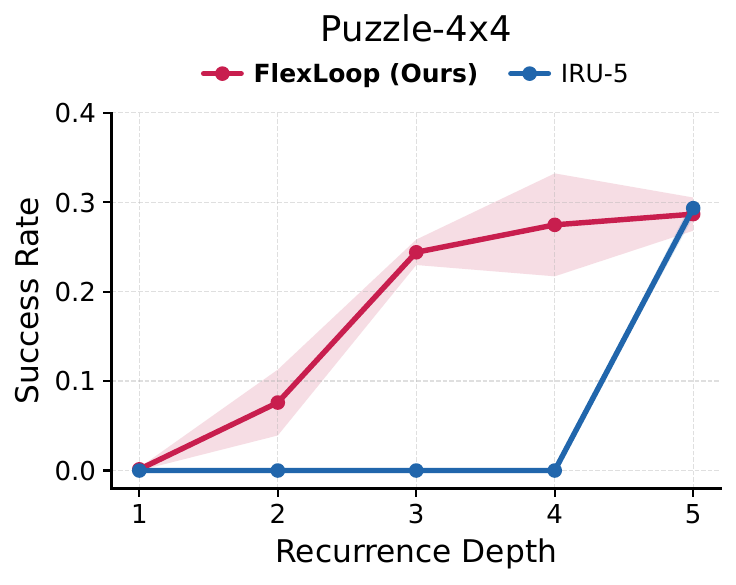}
  \end{subfigure}
  \hfill
  \begin{subfigure}[t]{0.24\textwidth}
    \centering
    \includegraphics[width=\linewidth]{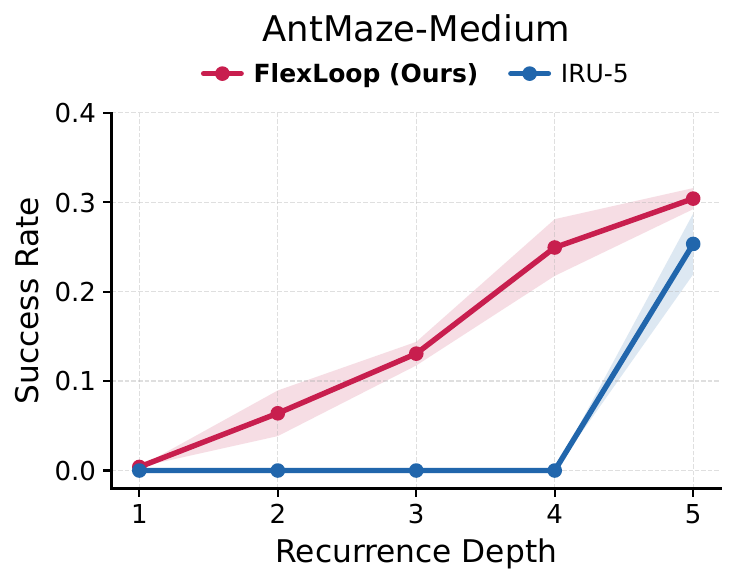}
  \end{subfigure}\\
  \begin{subfigure}[t]{0.24\textwidth}
    \centering
    \includegraphics[width=\linewidth]{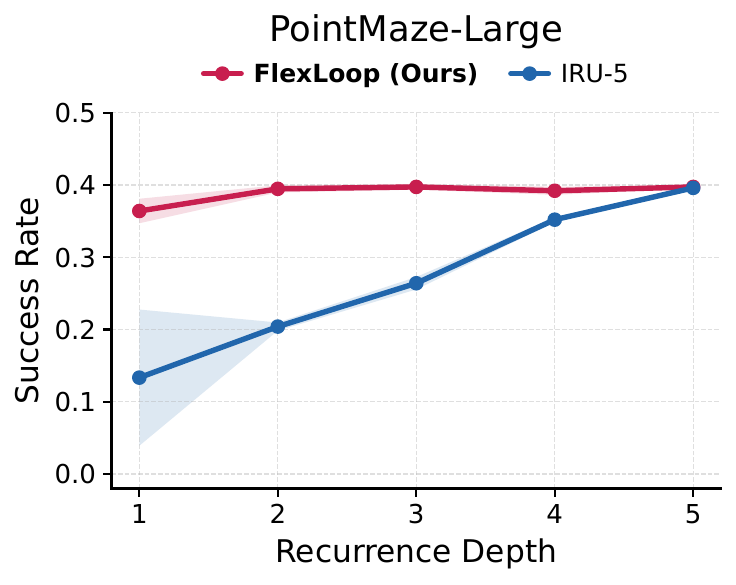}
  \end{subfigure}
  \hspace{0.2\linewidth}
  \begin{subfigure}[t]{0.24\textwidth}
    \centering
    \includegraphics[width=\linewidth]{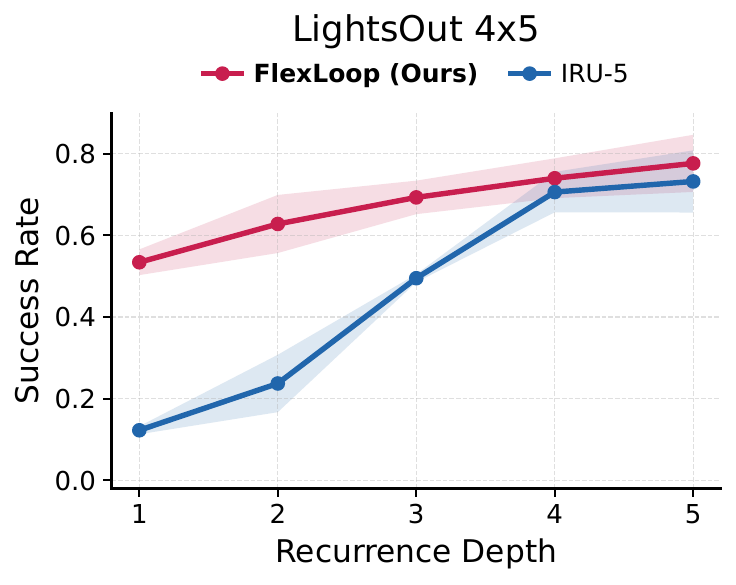}
  \end{subfigure}
  \caption{
    Depth-wise decision performance on training tasks before and after FlexLoop post-training. Markers denote mean success rates and shading indicates one standard deviation.}
  \label{fig:res_elastic}
\end{figure}

\begin{figure}[!ht]
  \centering
  \begin{subfigure}[t]{0.32\textwidth}
    \centering
    \includegraphics[width=\linewidth]{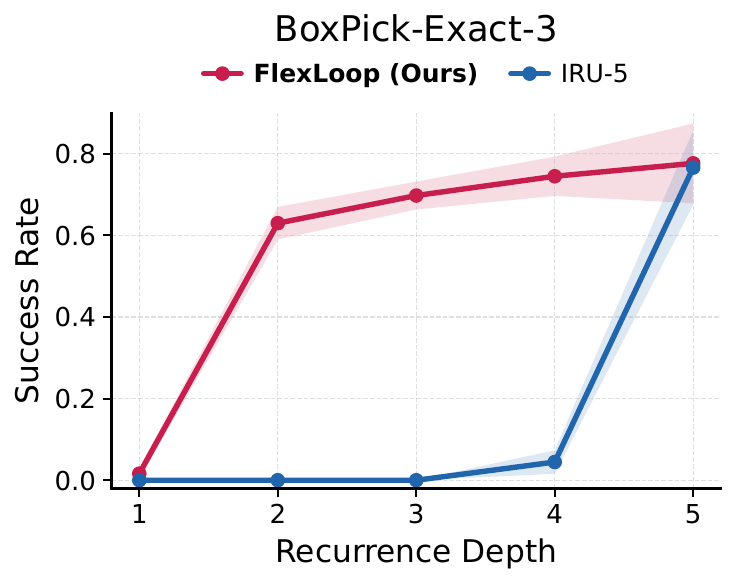}
  \end{subfigure}
  \hfill
  \begin{subfigure}[t]{0.32\textwidth}
    \centering
    \includegraphics[width=\linewidth]{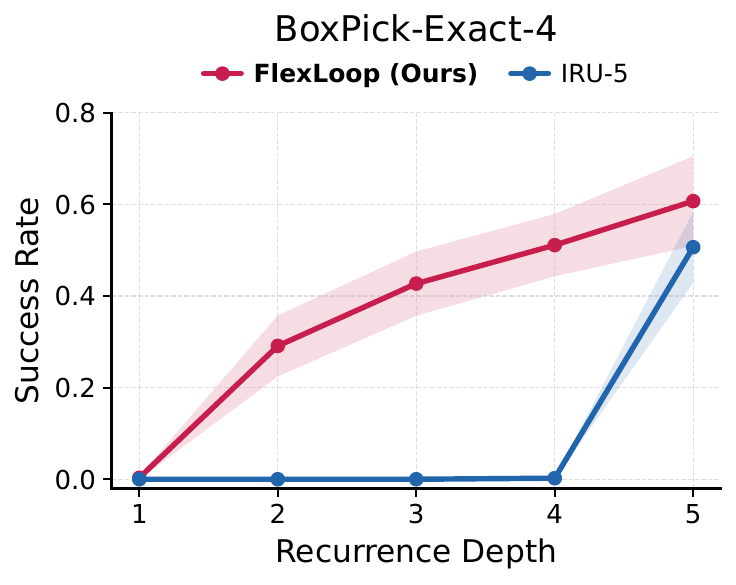}
  \end{subfigure}
  \hfill
  \begin{subfigure}[t]{0.32\textwidth}
    \centering
    \includegraphics[width=\linewidth]{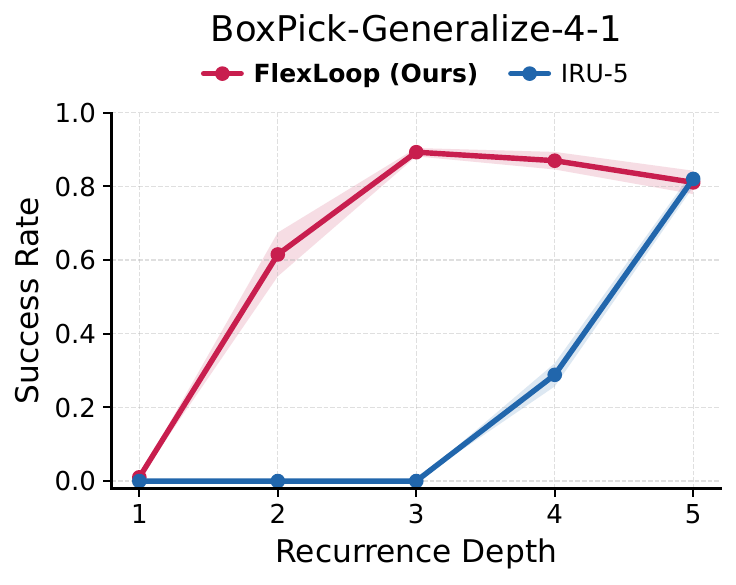}
  \end{subfigure}\\
  \begin{subfigure}[t]{0.32\textwidth}
    \centering
    \includegraphics[width=\linewidth]{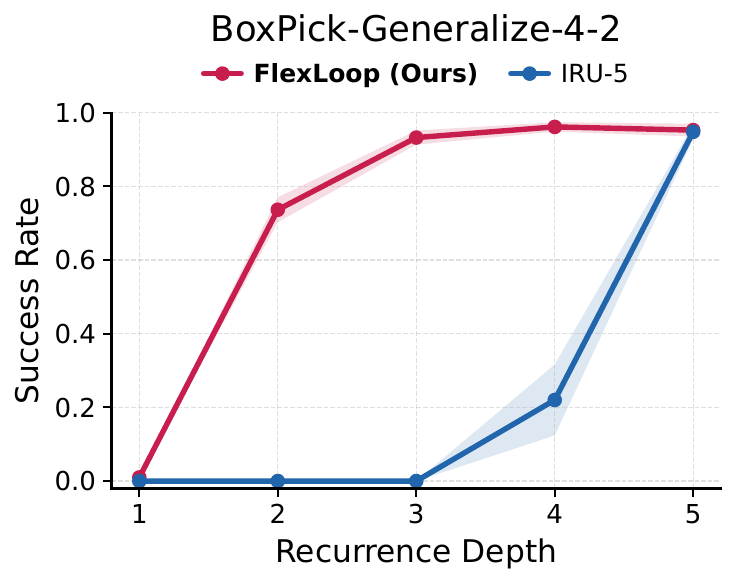}
  \end{subfigure}
  \hspace{.15\linewidth}
  \begin{subfigure}[t]{0.32\textwidth}
    \centering
    \includegraphics[width=\linewidth]{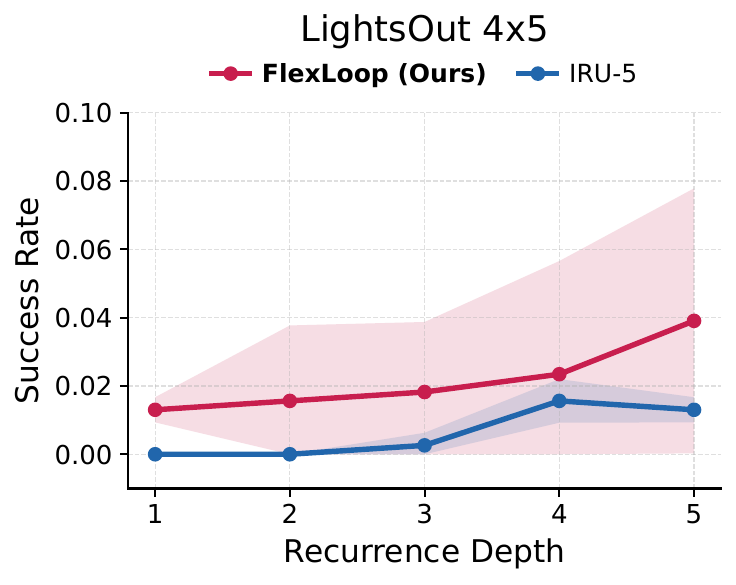}
  \end{subfigure}
  \caption{
    Depth-wise decision performance on unseen tasks before and after FlexLoop post-training. Markers denote mean success rates and shading indicates one standard deviation.}
  \label{fig:res_elastic_test}
\end{figure}

\newpage\clearpage
\subsection{Full results for Table \ref{tab:adaptive_results}}\label{app:adaptive}
\begin{table}[!ht]
\centering
\caption{Full state-wise adaptive inference results on all environments. $\pm$ indicates one standard deviation.}
\label{tab:adaptive_results_full}
\small
\begin{tabular}{lcccc}
\toprule
\multirow[c]{2}{*}{\textbf{Task}} & \multicolumn{2}{c}{\textbf{FlexLoop}} & \textbf{IRU-5}  \\
\cmidrule(lr){2-3}\cmidrule(lr){4-4}
 & Avg. Depth & Succ. & Succ. \\
\midrule
Scene      & 4.66 $\pm$ 0.06 & 0.496 $\pm$ 0.064 & 0.489$\pm$0.087 \\
Cube       & 4.75 $\pm$ 0.03 & 0.569 $\pm$ 0.056 & 0.492$\pm$0.114\\
Puzzle     & 3.45 $\pm$ 0.03 & 0.296 $\pm$ 0.029 & 0.293$\pm$0.01\\
AntMaze    & 4.80 $\pm$ 0.03 & 0.292 $\pm$ 0.059 & 0.253$\pm$0.034 \\
PointMaze  & 3.28 $\pm$ 0.30 & 0.395 $\pm$ 0.004 & 0.396$\pm$0.003 \\
LightsOut-4x5 (Train) & 2.84 $\pm$ 0.08 & 0.747 $\pm$ 0.081 & 0.732 $\pm$ 0.076 \\
LightsOut-4x5 (Unseen) & 2.93 $\pm$ 0.21 & 0.053 $\pm$ 0.031 & 0.013 $\pm$ 0.004 \\
BoxPick-Exact-3(Train)&	4.47 $\pm$ 0.02& 0.992 $\pm$ 0.004& 0.992 $\pm$ 0.003\\
BoxPick-Exact-3(Unseen) &4.45 $\pm$ 0.05& 0.768 $\pm$ 0.101&0.766 $\pm$ 0.091 \\
BoxPick-Exact-4(Train)&4.21 $\pm$ 0.04& 0.993 $\pm$ 0.002&0.994$\pm$0.003\\
BoxPick-Exact-4(Unseen)&4.25 $\pm$ 0.04& 0.567 $\pm$ 0.098&0.506$\pm$0.077\\
BoxPick-Gen-4-1(Train)&4.68 $\pm$ 0.03&0.982 $\pm$ 0.007&0.98 $\pm$ 0.007\\
BoxPick-Gen-4-1(Unseen)&4.57 $\pm$ 0.02& 0.82 $\pm$ 0.033&0.82$\pm$0.022\\
BoxPick-Gen-4-2(Train)&4.66 $\pm$ 0.01& 0.991 $\pm$ 0.004&0.989$\pm$0.004\\
BoxPick-Gen-4-2(Unseen)&4.57 $\pm$ 0.01& 0.956 $\pm$ 0.017&0.948$\pm$0.019\\
\bottomrule
\multicolumn{4}{l}{-Gen means -Generalize}
\end{tabular}
\end{table}
\newpage
\subsection{Pareto Curve}\label{app:pareto}
Each Pareto curve is obtained by sweeping the stopping threshold $\varepsilon$ in \cref{eq:adaptive}, where each threshold corresponds to one computation--performance point. The $\varepsilon$ grids are selected according to the characteristics of different environment families. For BoxPick, we use a 7-point grid $\{0,0.005,0.01,0.02,0.03,0.05,0.1\}$. For LightsOut, we use an 11-point uniformly spaced grid over $[0,0.05]$. For OGBench, we use an 11-point uniformly spaced grid for all environments: manipulation tasks (Scene, Cube, and Puzzle) use a finer grid over $[0, 0.005]$, while navigation tasks (AntMaze and PointMaze) use a grid over $[0, 0.05]$. 
\begin{figure}[!ht]
  \centering
  \begin{subfigure}[t]{\linewidth}
    \centering
    \includegraphics[width=0.19\linewidth]{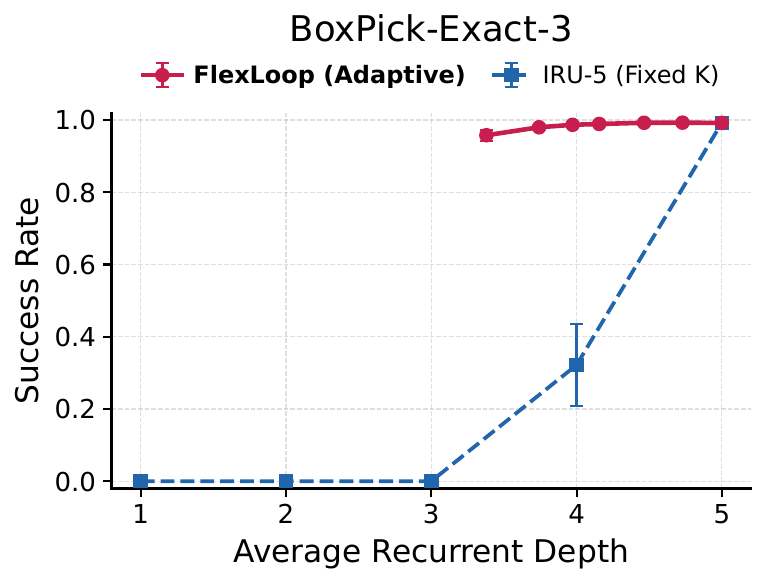}
  \hfill
    \includegraphics[width=0.19\linewidth]{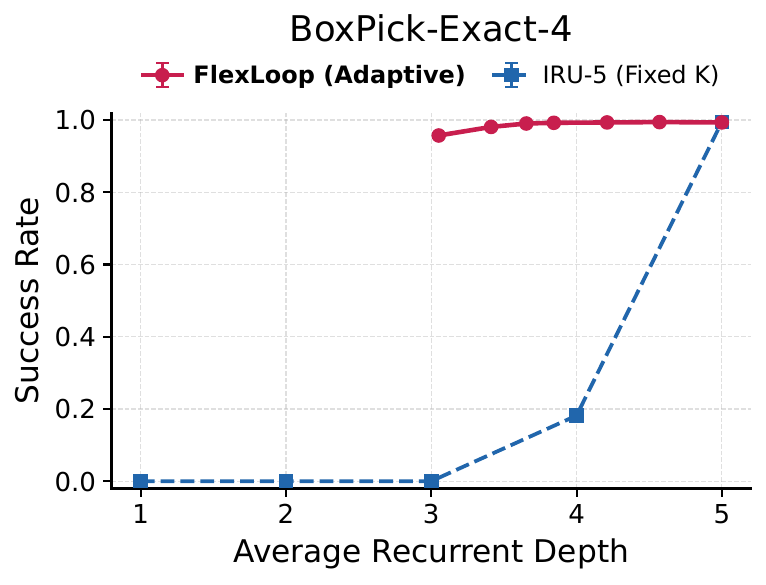}
  \hfill
    \includegraphics[width=0.19\linewidth]{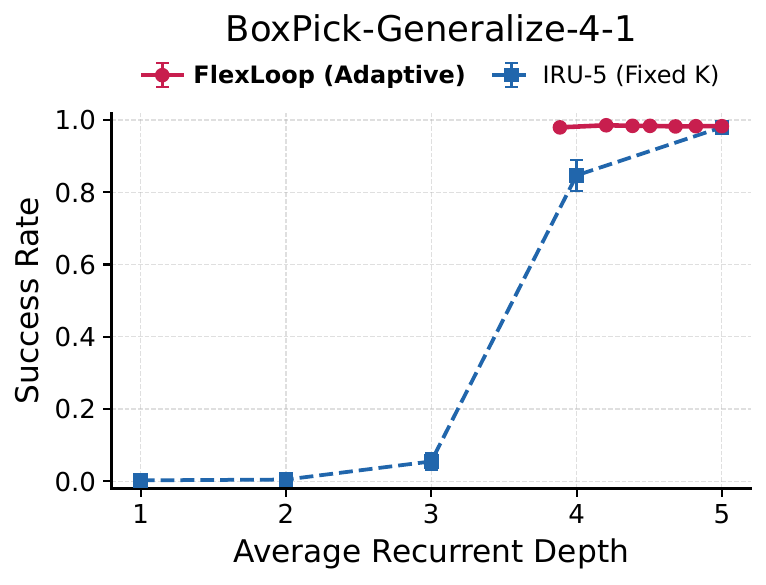}
  \hfill
    \includegraphics[width=0.19\linewidth]{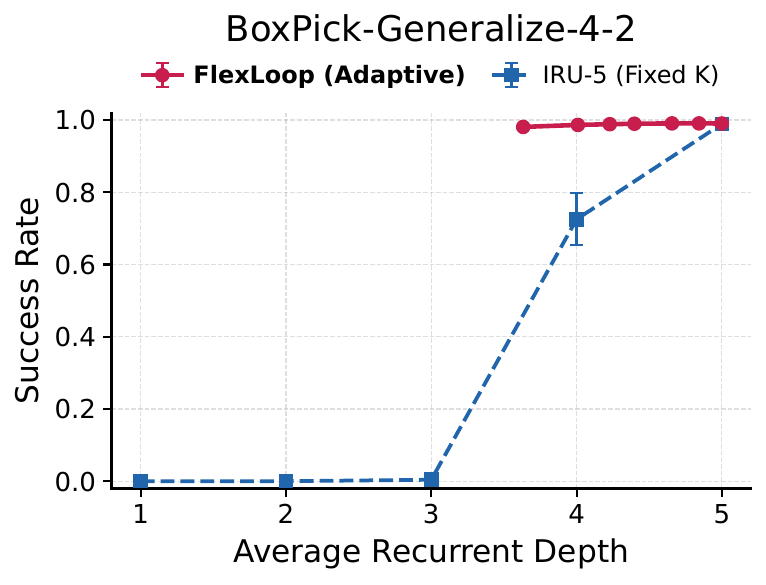}
  \hfill
    \includegraphics[width=0.19\linewidth]{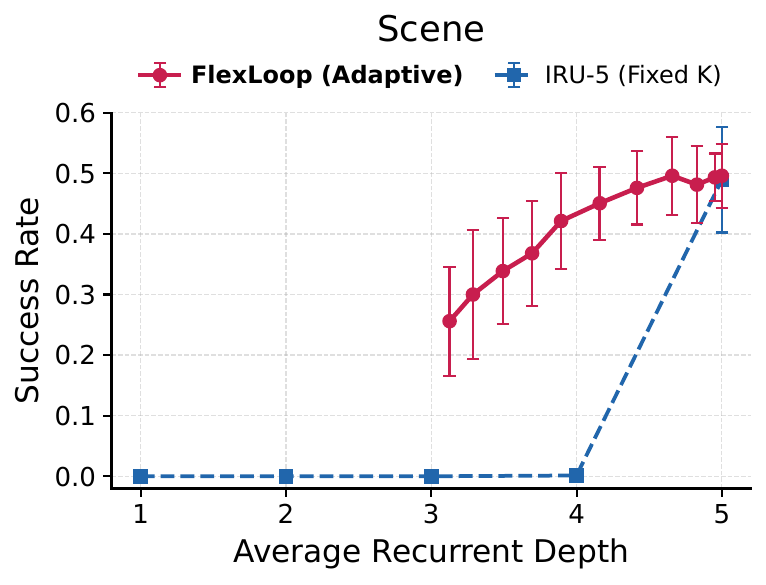}\\
\includegraphics[width=0.19\linewidth]{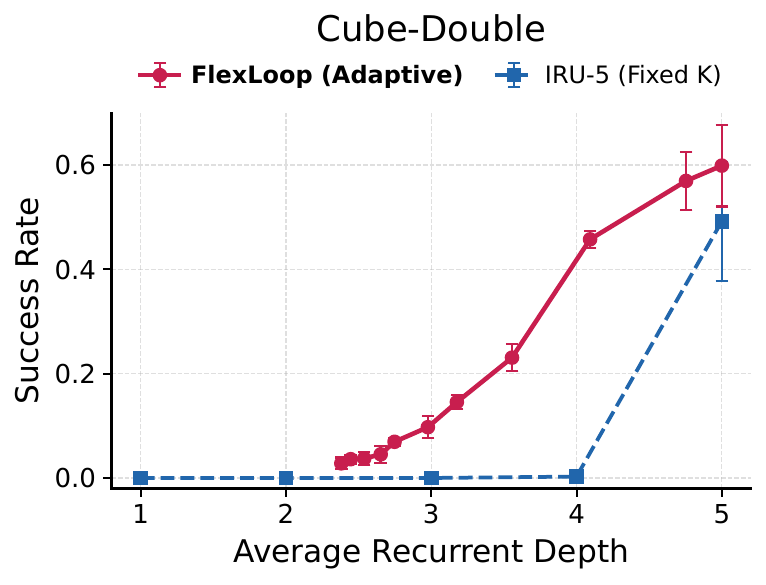}
  \hfill
    \includegraphics[width=0.19\linewidth]{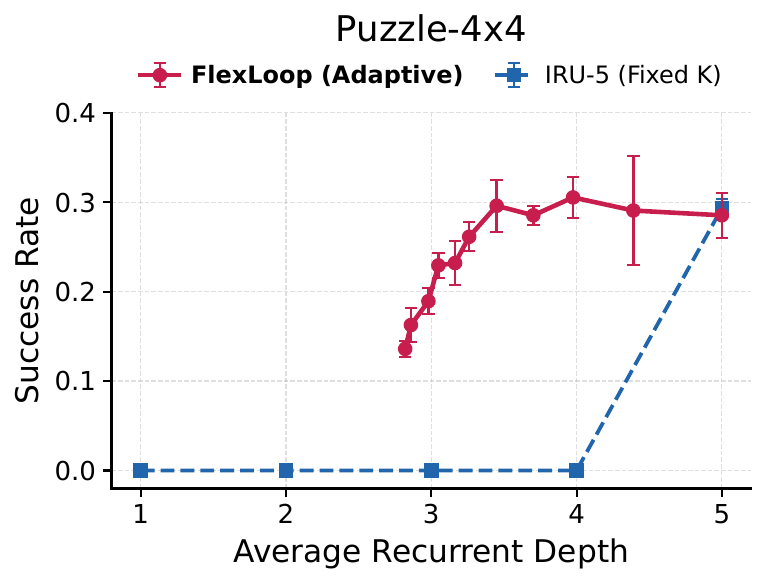}
  \hfill
    \includegraphics[width=0.19\linewidth]{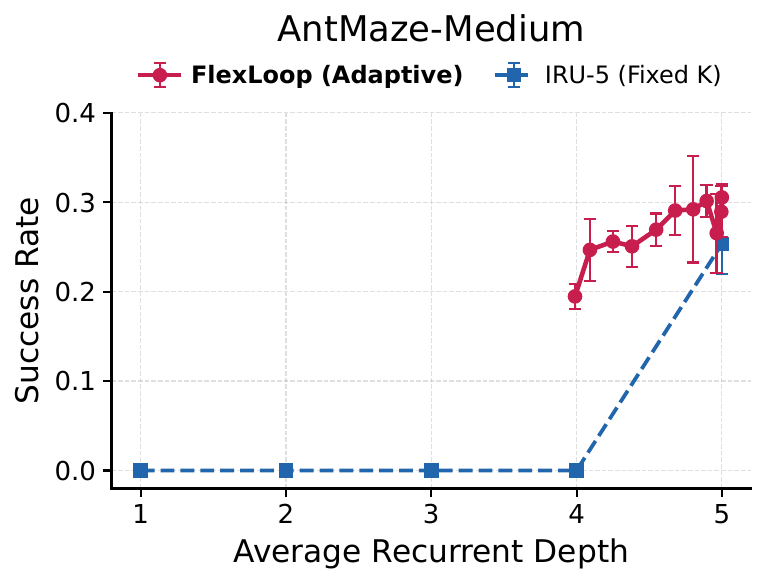}
  \hfill
    \includegraphics[width=0.19\linewidth]{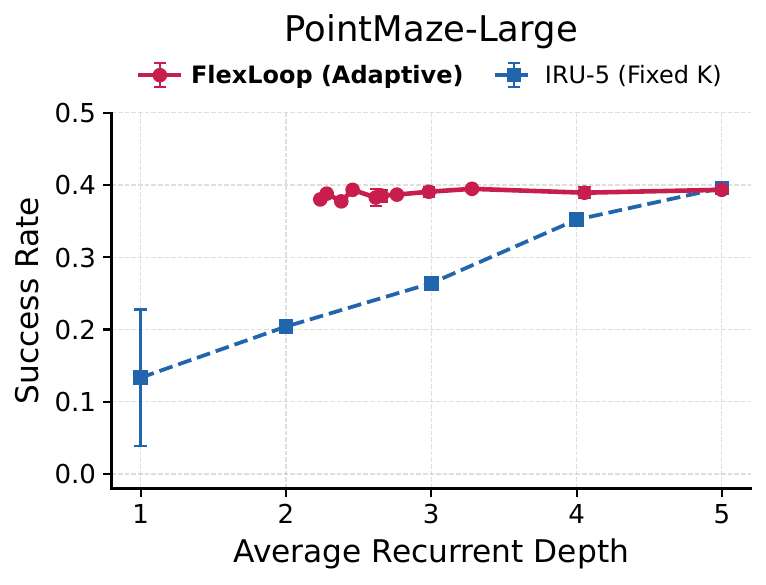}
  \hfill
    \includegraphics[width=0.19\linewidth]{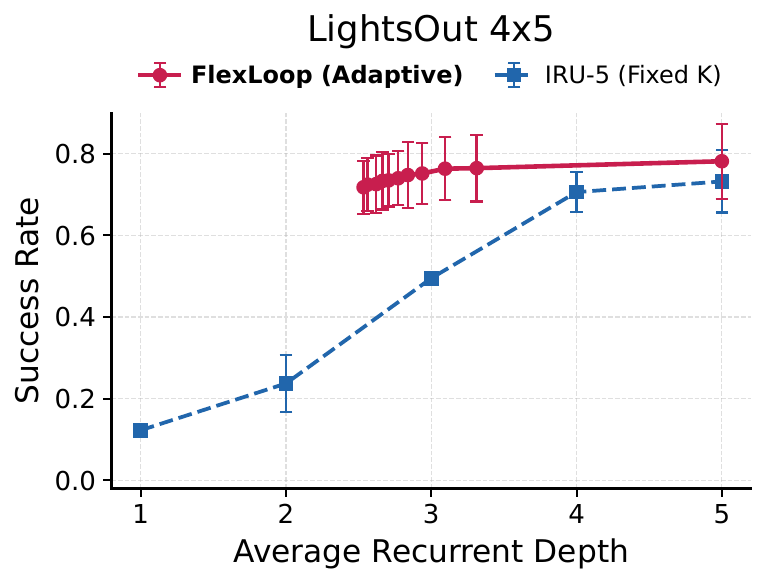}
  \caption{Pareto trade-offs on training tasks.}\label{fig:pareto_train}
  \end{subfigure}
  \\
  \vspace{6pt}
  \begin{subfigure}[t]{\linewidth}
    \includegraphics[width=0.19\linewidth]{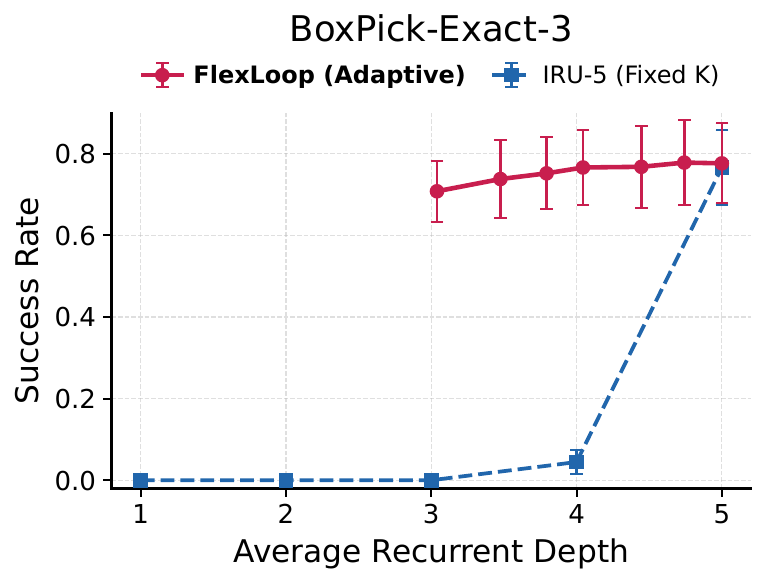}
  \hfill
    \includegraphics[width=0.19\linewidth]{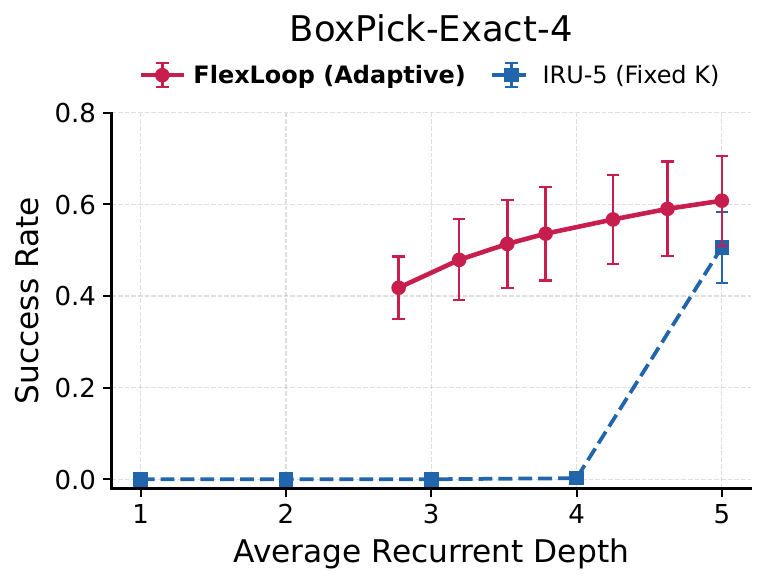}
  \hfill
    \includegraphics[width=0.19\linewidth]{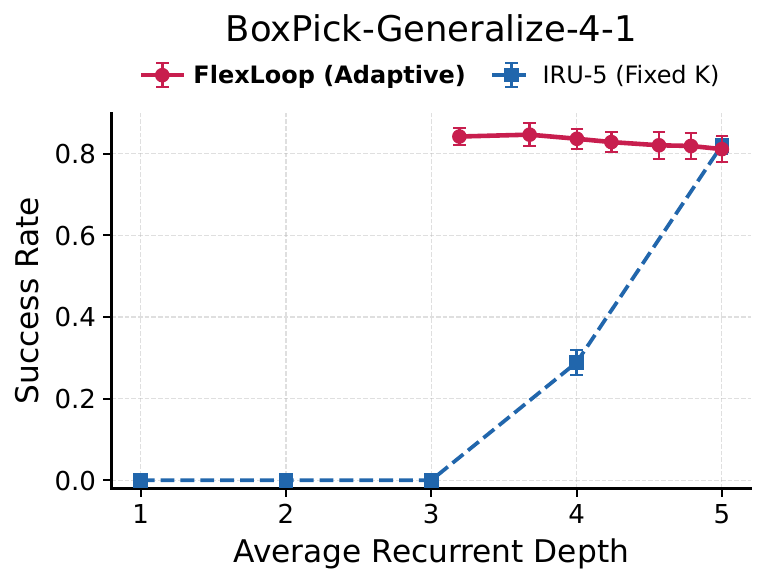}
  \hfill
    \includegraphics[width=0.19\linewidth]{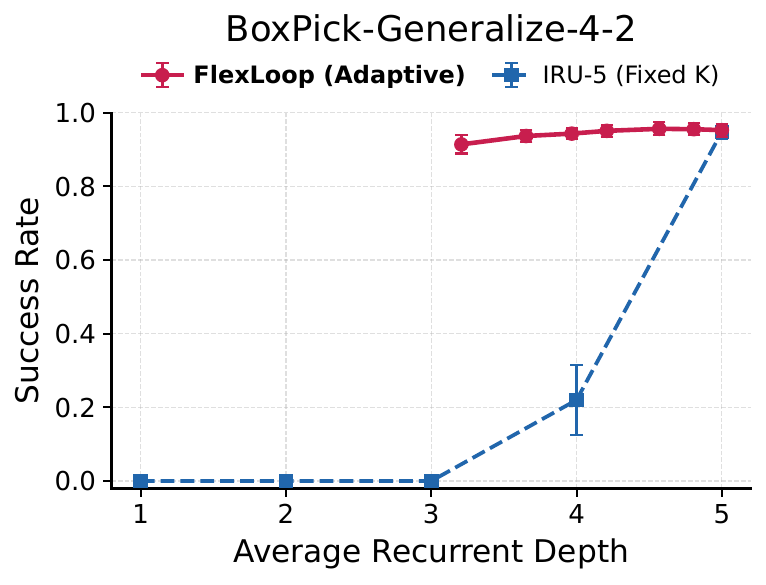}
  \hfill
    \includegraphics[width=0.19\linewidth]{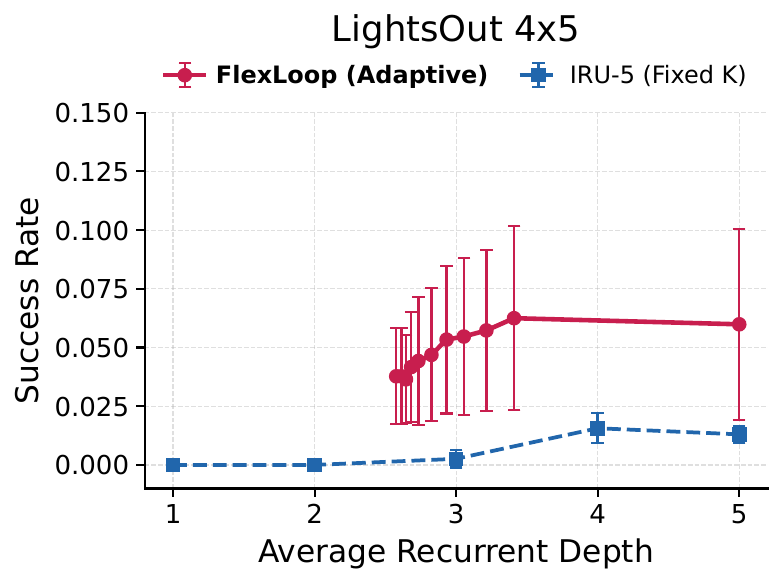}
\caption{Pareto trade-offs on unseen tasks.}\label{fig:pareto_test}
  \end{subfigure}
  
  \caption{
    Pareto trade-offs between task performance and average recurrent depth under different adaptive inference thresholds. Markers denote mean success rates and error bars indicate one standard deviation.}
  \label{fig:res_pareto}
\end{figure}
\newpage
\subsection{Visualized Trajectory Segments of Scene}\label{app:adaptive_vis}

Figure \ref{fig:vis_frame} visualizes selected trajectory segments of FlexLoop in the Scene environment to illustrate its state-wise adaptive inference behavior. FlexLoop uses shallower recurrent depths for relatively simple and monotonic actions, such as lowering the robotic arm to press a button or rotating the gripper orientation, while allocating deeper computation to more critical decision points, such as switching goals after pressing the button and moving toward the drawer.
\begin{figure}[!ht]
  \centering
  \begin{subfigure}[t]{\linewidth}
    \centering
    \includegraphics[width=\linewidth]{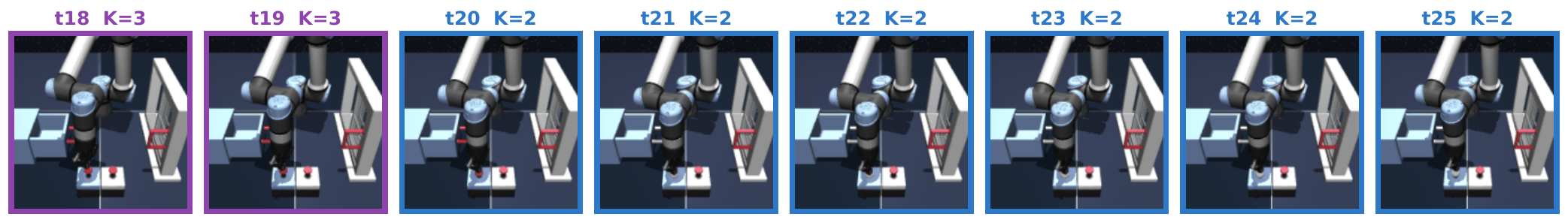}
  \caption{Lowering the arm to press the button.}\label{fig:press}
  \end{subfigure}
  \begin{subfigure}[t]{\linewidth}
    \centering
    \includegraphics[width=\linewidth]{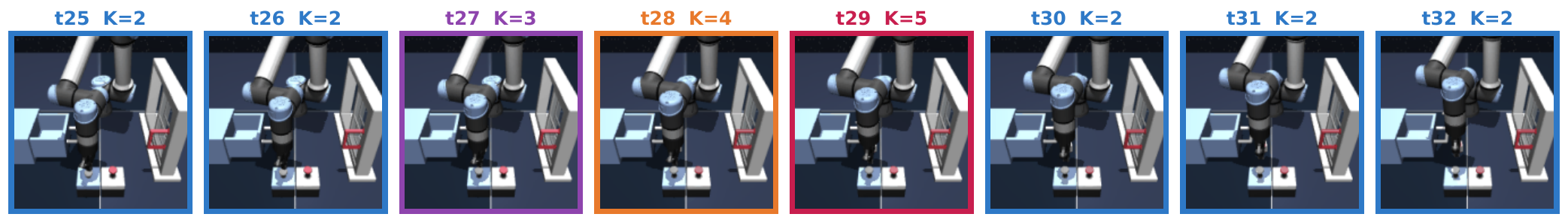}
  \caption{Transitioning from button pressing to moving toward the drawer.}\label{fig:boundary}
  \end{subfigure}
  \begin{subfigure}[t]{\linewidth}
    \centering
    \includegraphics[width=\linewidth]{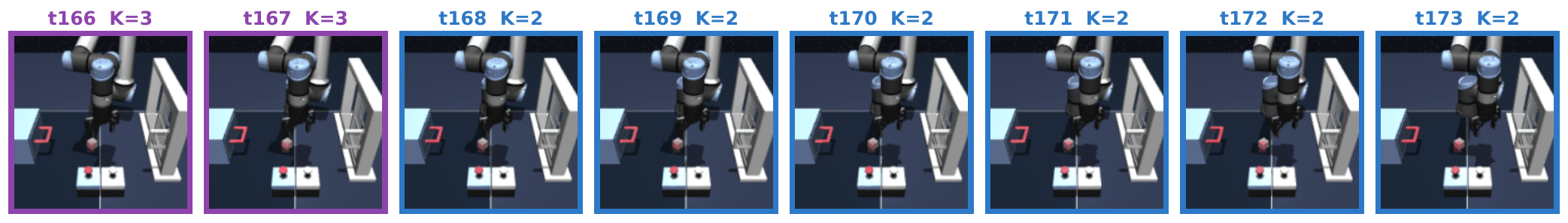}
  \caption{Rotating the gripper orientation.}\label{fig:turn}
  \end{subfigure}
  \caption{
    Selected trajectory segments from a single episode in the Scene environment. FlexLoop adaptively uses fewer recurrent steps for simple, monotonic motions and increases computation at challenging decision points, such as goal switching.}
  \label{fig:vis_frame}
\end{figure}

% \newpage
\section{Additional Ablations}\label{app:ex_ablation}

\subsection{FlexLoop from Scratch}\label{app:scratch}
\begin{figure}[!t]
  \centering
  \begin{subfigure}[t]{.49\linewidth}
    \centering
    \includegraphics[width=\linewidth]{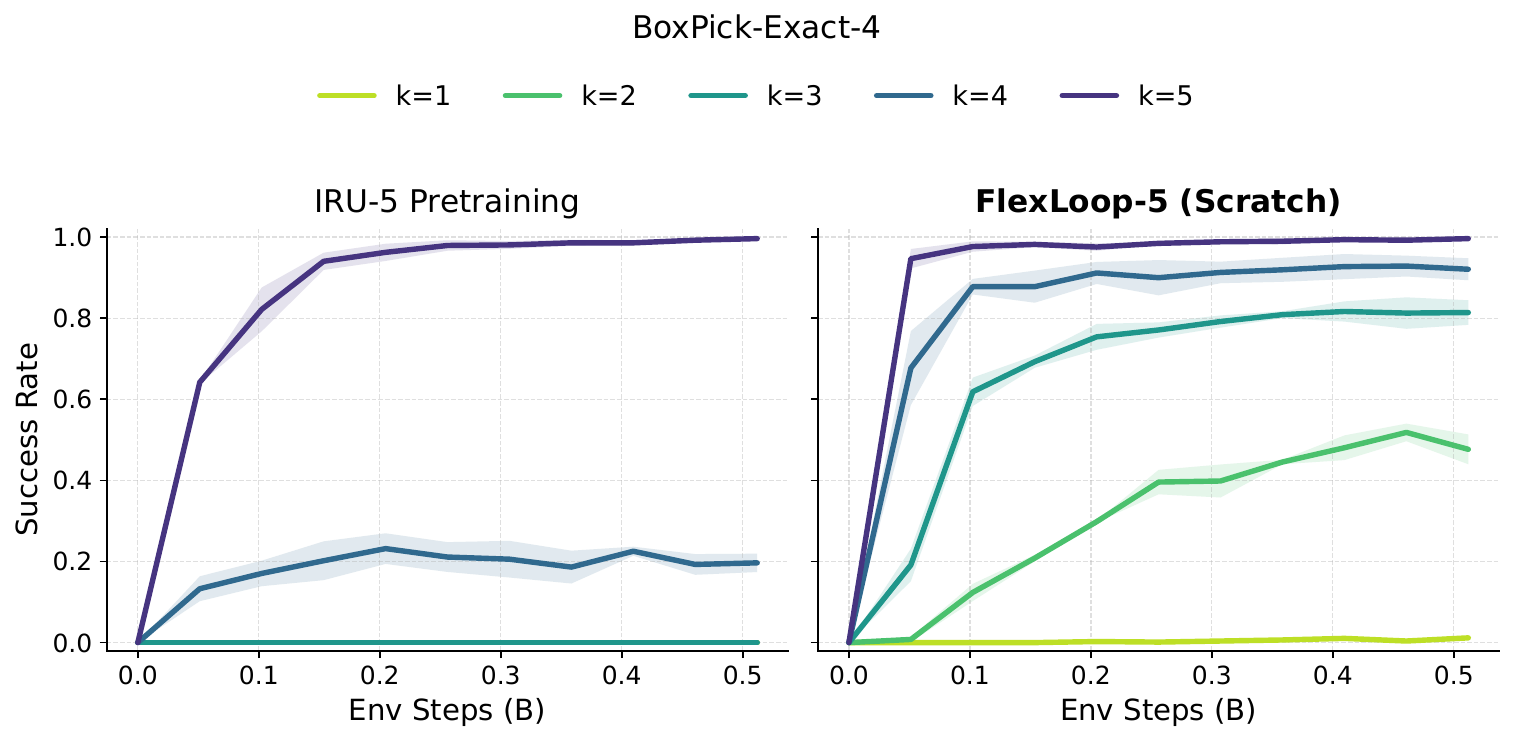}
  \caption{Training Tasks.}\label{fig:scratch_train}
  \end{subfigure}
  \begin{subfigure}[t]{.49\linewidth}
    \centering
    \includegraphics[width=\linewidth]{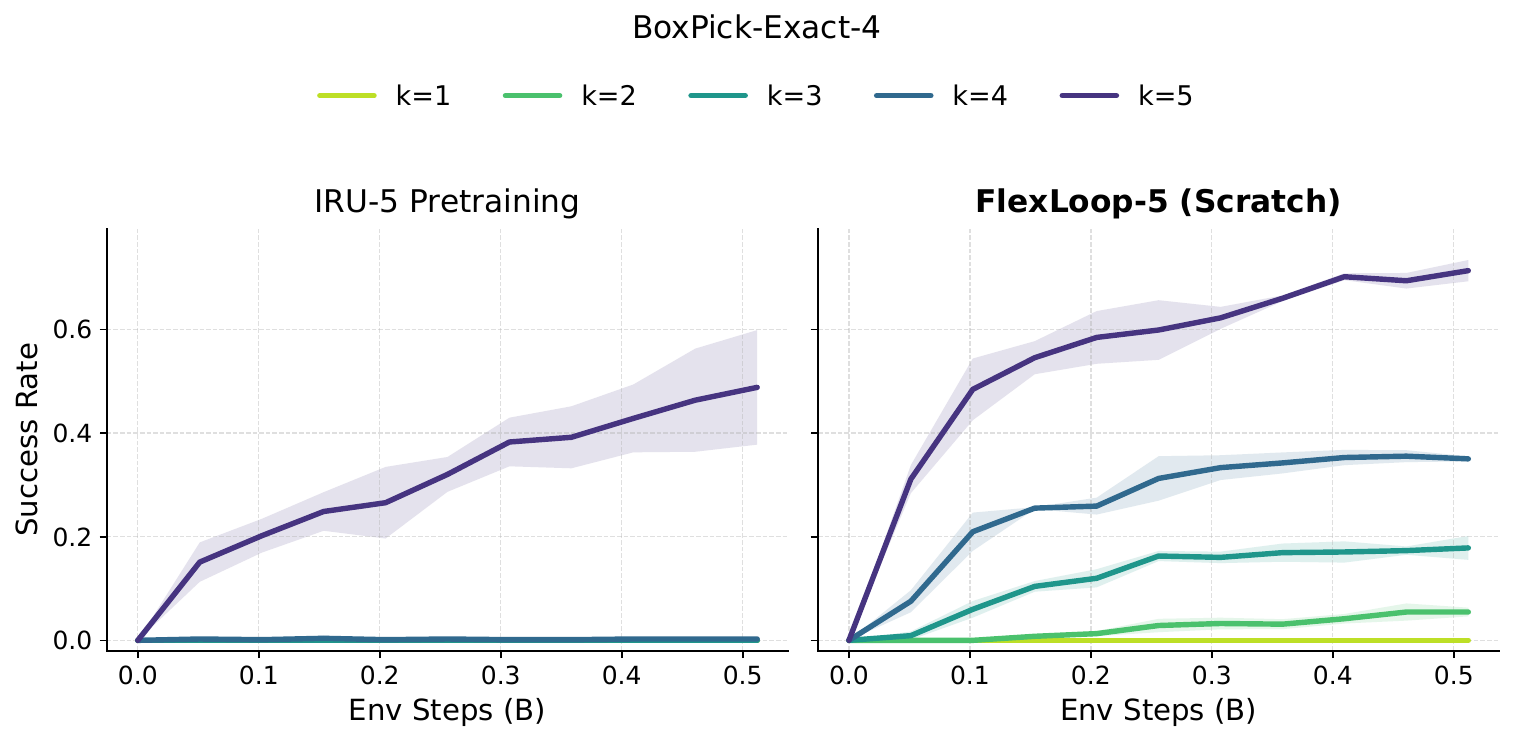}
  \caption{Unseen Tasks.}\label{fig:scratch_ood}
  \end{subfigure}
  \caption{
    Depth-wise performance evolution of pretrained IRU-5 and FlexLoop from scratch on both training and unseen tasks. Solid lines denote mean success rates and shading indicates one standard deviation.}
  \label{fig:scratch}
\end{figure}

In this section, we further consider a more restrictive setting where the pretrained looped policy weights are unavailable and only its outputs can be queried. We ask whether FlexLoop can still acquire depth elasticity from random initialization while recovering the teacher's full-depth capability. 

We instantiate this setting on BoxPick-Exact-4 with GC-DQN. Since the scratch policy does not inherit the pretrained weights, we make two modifications to FlexLoop.
First, at the full depth, we replace the standard TD target with the value estimate queried from the pretrained looped policy:
\begin{equation}
    \mathcal{L}_{\mathrm{TD}}^{T}
    =
    \mathbb{E}_{\mathcal D}
    \left[
        \left(
            Q_K^\theta(s,g,a)
            -
            \operatorname{sg}[y^{T}]
        \right)^2
    \right], \qquad 
    y^{T}
    =
    r+\gamma(1-d)
    \max_{a'}Q^{T}(s',g,a').
\end{equation}
Second, we introduce an inter-depth Bellman refinement objective for shallow recurrent steps. Using the FlexLoop's target network $\bar{\theta}$, we define
\begin{equation}
    y_k
    =
    r+\gamma(1-d)
    \max_{a'}Q_k^{\bar{\theta}}(s',g,a'),
\end{equation}
and optimize
\begin{equation}
    \mathcal{L}_{\mathrm{aux}}^{\mathrm{DQN}}
    =
    \frac{1}{K-2}
    \sum_{k=1}^{K-2}
    \mathbb{E}_{\mathcal D}
    \left[
        \rho\!\left(
            Q_{k+1}^{\theta}(s,g,a)
            -
            \operatorname{sg}[y_k]
        \right)
    \right].
\end{equation}
We set the weight of $\mathcal{L}_{\mathrm{aux}}^{\mathrm{DQN}}$ to $0.1$.

Notably, these modifications are introduced only to compensate for the absence of pretrained initialization. We observe no additional benefit in the teacher-initialized setting.

As shown in Figures~\ref{fig:scratch} and \ref{fig:scratch_more}, FlexLoop from scratch rapidly develops depth elasticity and recovers strong full-depth performance, while still supporting effective state-wise adaptive inference. 
Together, these results broaden the applicability of FlexLoop to settings where only output-level access to a pretrained policy is available.

% Finally, we introduce a shallow Bellman value-iteration objective that encourages each additional recurrent step to perform a locally Bellman-consistent value refinement. Empirically, this auxiliary objective further improves optimization speed and out-of-distribution generalization.
% Using the target network $\bar{\theta}$, we define
% \begin{equation}
%     y_k
%     =
%     r+\gamma(1-d)
%     \max_{a'}Q_k^{\bar{\theta}}(s',g,a'),
% \end{equation}
% and apply
% \begin{equation}
%     \mathcal{L}_{\mathrm{aux}}^{\mathrm{DQN}}
%     =
%     \frac{1}{K-2}
%     \sum_{k=1}^{K-2}
%     \mathbb{E}_{\mathcal D}
%     \left[
%         \rho\!\left(
%             Q_{k+1}^{\theta}(s,g,a)
%             -
%             \operatorname{sg}[y_k]
%         \right)
%     \right].
% \end{equation}
% This auxiliary objective is restricted to shallow recurrent steps, leaving
% the full-depth value $Q_K$ anchored exclusively by the frozen teacher.

\subsection{FlexLoop on Policies with Different Pretraining Depths}\label{app:kvary}
We further apply FlexLoop to IRU policies pretrained with different recurrent depths, $K\in\{4,5,8,10\}$.   Figure~\ref{fig:kvary} shows that, across all pretraining depths, FlexLoop substantially improves shallow-depth performance while preserving full-depth capability, with the same trend observed on unseen tasks. The corresponding
Pareto curves further show that state-wise adaptive inference remains effective across different pretraining depths. We note that deeper pretrained policies generally require smaller stopping thresholds $\varepsilon$, suggesting that the appropriate threshold scale depends on the pretrained recurrent depth.
\begin{figure}[!ht]
  \centering
  \begin{subfigure}[t]{.24\linewidth}
    \centering
    \includegraphics[width=\linewidth]{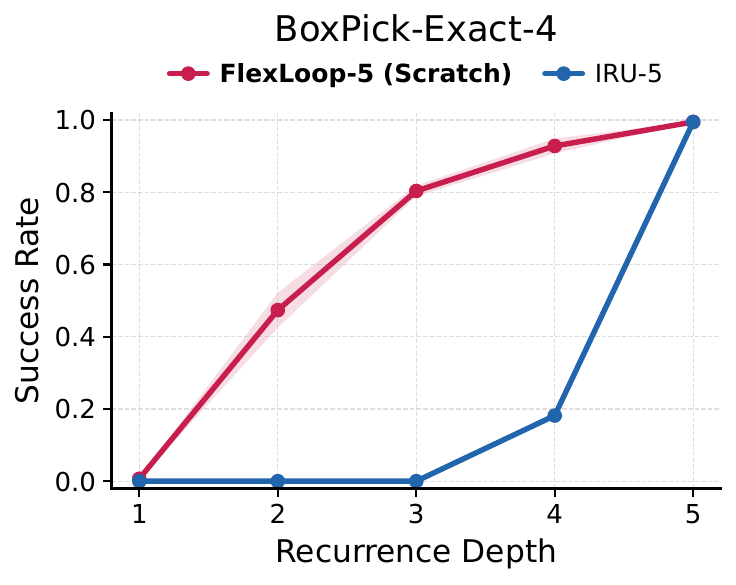}
  \caption{Elasticity (Training).}\label{fig:scratch_elas_train}
  \end{subfigure}
  \begin{subfigure}[t]{.24\linewidth}
    \centering
    \includegraphics[width=\linewidth]{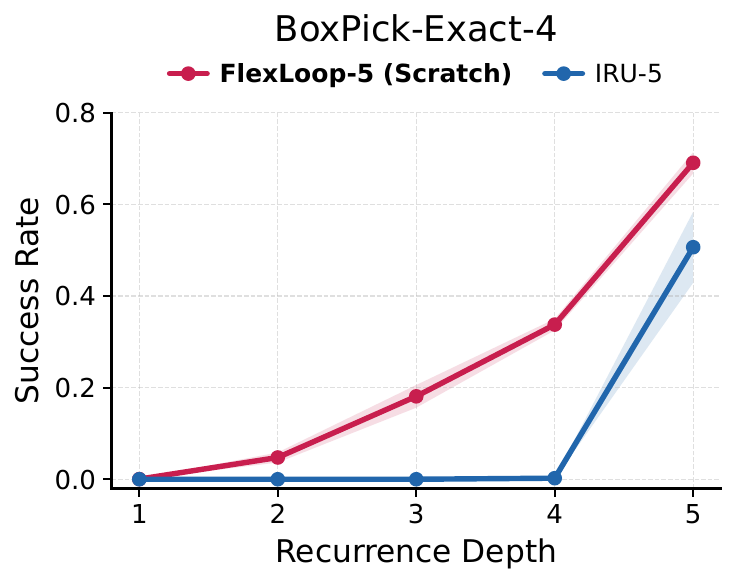}
  \caption{Elasticity (Unseen).}\label{fig:scratch_elas_ood}
  \end{subfigure}
  \begin{subfigure}[t]{.24\linewidth}
    \centering
    \includegraphics[width=\linewidth]{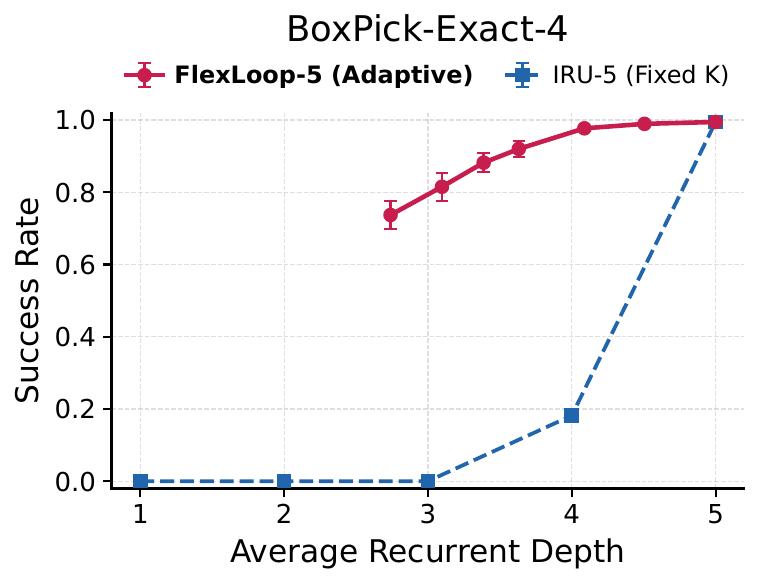}
  \caption{Pareto (Training).}\label{fig:scratch_pareto_train}
  \end{subfigure}
  \begin{subfigure}[t]{.24\linewidth}
    \centering
    \includegraphics[width=\linewidth]{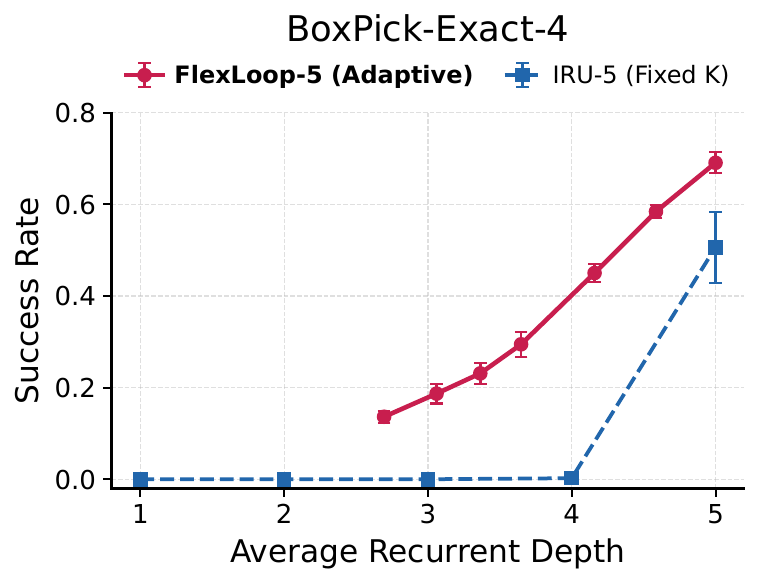}
  \caption{Pareto (Unseen).}\label{fig:scratch_pareto_ood}
  \end{subfigure}
  \caption{
    Elasticity and Pareto curves of pretrained IRU-5 and FlexLoop from scratch on both training and unseen tasks. Markers denote mean success rates and the shading and error bars indicate one standard deviation.}
  \label{fig:scratch_more}
\end{figure}

\begin{figure}[!ht]
  \centering
  \begin{subfigure}[t]{\linewidth}
    \centering
    \includegraphics[width=\linewidth]{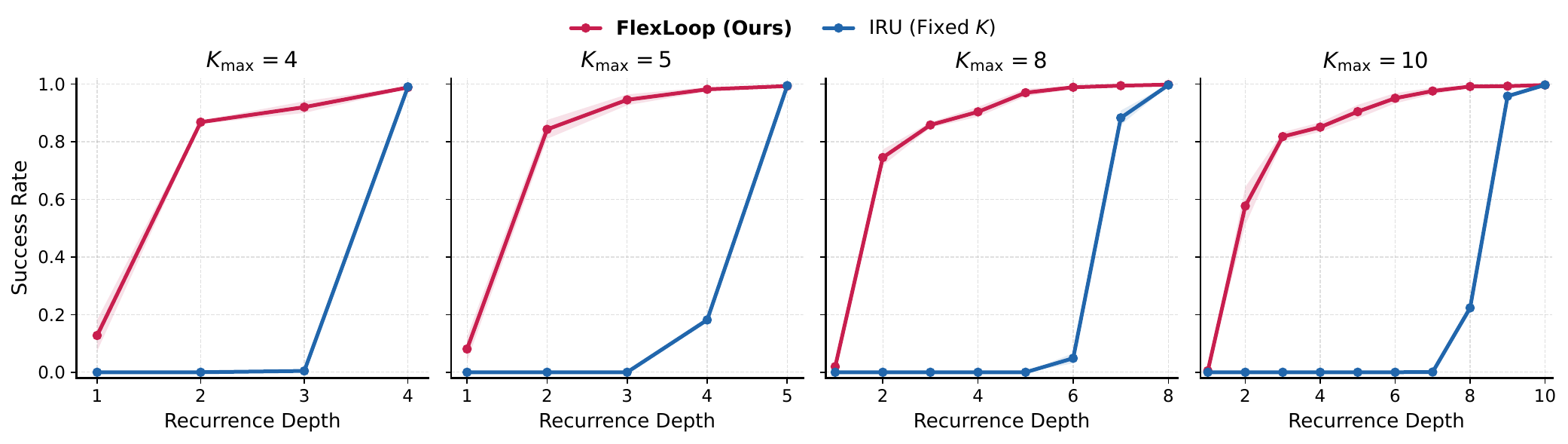}
  \caption{Elasticity on Training Tasks.}\label{fig:kvary_elas_train}
  \end{subfigure}
  \begin{subfigure}[t]{\linewidth}
    \centering
    \includegraphics[width=\linewidth]{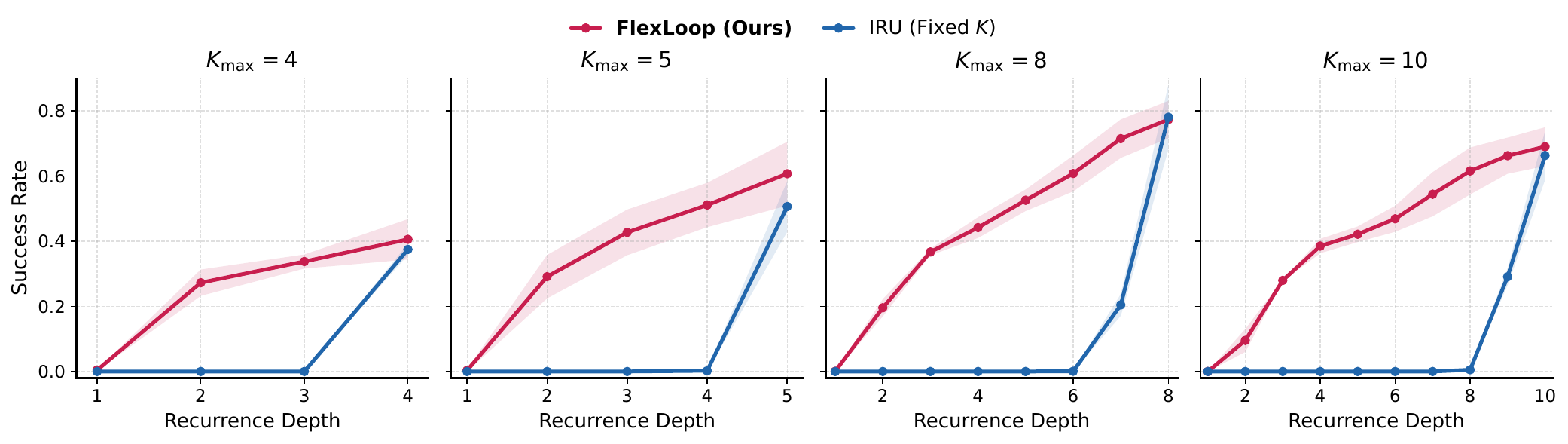}
  \caption{Elasticity on Unseen Tasks.}\label{fig:kvary_elas_ood}
  \end{subfigure}
  \begin{subfigure}[t]{\linewidth}
    \centering
    \includegraphics[width=\linewidth]{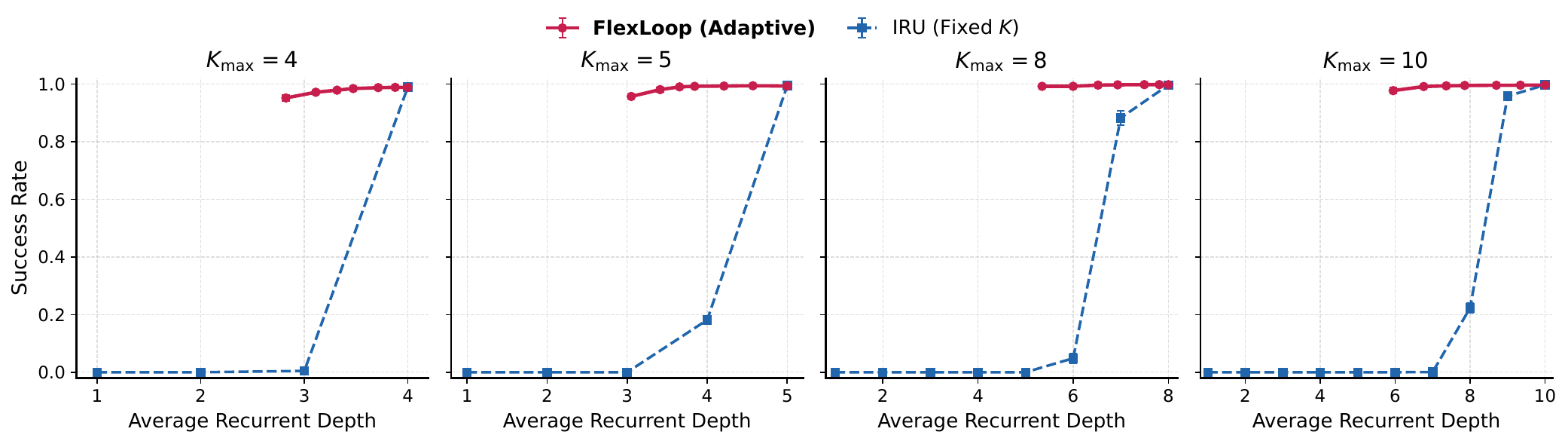}
  \caption{Pareto on Training Tasks.}\label{fig:kvary_pareto_train}
  \end{subfigure}
  \begin{subfigure}[t]{\linewidth}
    \centering
    \includegraphics[width=\linewidth]{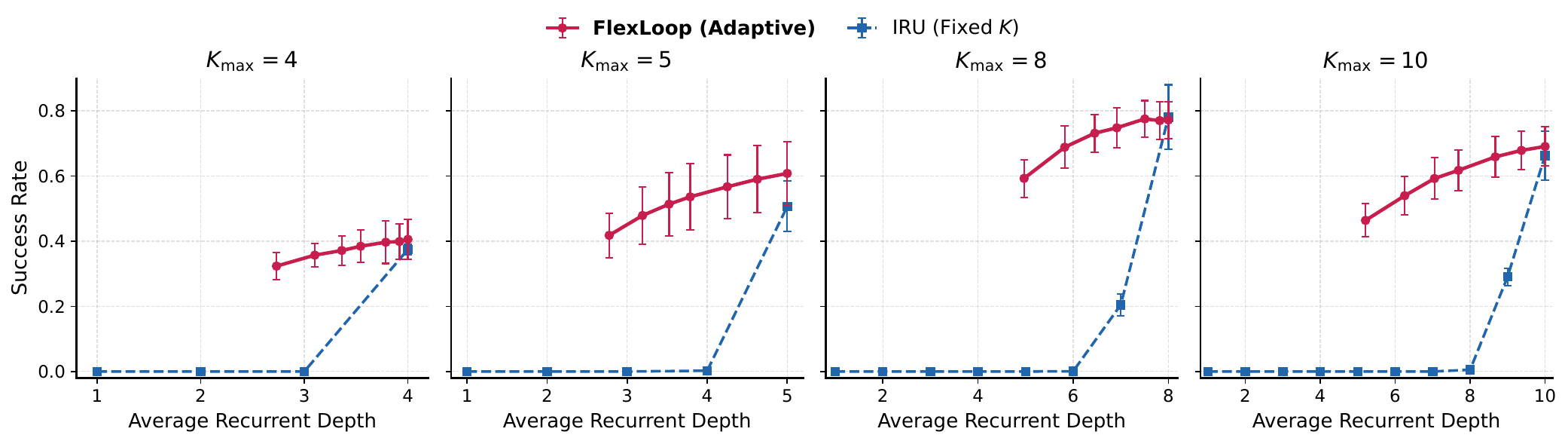}
  \caption{Pareto on Unseen Tasks.}\label{fig:kvary_pareto_ood}
  \end{subfigure}
  \caption{
    Elasticity and Pareto Curves of the policies with different pretraining recurrent depths before and after FlexLoop post-training. Markers denote mean success rates and shading indicates one standard deviation.}
  \label{fig:kvary}
\end{figure}

\end{document}